\documentclass{article}

\usepackage{microtype}
\usepackage{graphicx}
\usepackage{subfigure}
\usepackage{booktabs} 

\usepackage{hyperref}

\usepackage[accepted]{new}

\icmltitlerunning{Simultaneous Inference for Massive Data: Distributed Bootstrap}

\usepackage{amsmath,amsthm,amsfonts,amssymb}
\usepackage{enumitem}

\newtheorem{theo}{Theorem}[section]
\newtheorem{lemma}[theo]{Lemma}

\newtheorem{rem}[theo]{Remark}

\makeatletter\@addtoreset{equation}{section}\makeatother

\newcommand{\defn}{\ensuremath{: \, =}}

\newcommand{\matrixnorm}[1]{\left|\!\left|\!\left|{#1}\right|\!\right|\!\right|}

\newcommand{\lambdamin}{\lambda_{\tiny{\min}}}
\newcommand{\lambdamax}{\lambda_{\tiny{\max}}}

\newcommand{\R}{\mathbb{R}}

\newcommand{\Ee}{\mathbb{E}}

\newcommand{\bg}{\mathbf{g}}
\newcommand{\bh}{\mathbf{h}}

\newcommand{\cL}{\mathcal{L}}
\newcommand{\cLs}{\mathcal{L}^\ast}

\newcommand{\cN}{\mathcal{N}}
\newcommand{\cM}{\mathcal{M}}

\DeclareMathOperator*{\argmin}{\arg\min}
\DeclareMathOperator{\cov}{cov}

\newcommand{\btheta}{\bar \theta}
\newcommand{\thetas} {\theta^\ast}
\newcommand{\htheta}{\widehat \theta}
\newcommand{\ttheta}{\widetilde \theta}

\newcommand{\tT}{\widetilde\Theta}

\makeatletter 
\newcommand{\labitemc}[2]{%
\def\@itemlabel{\textbf{#1}}
\item
\def\@currentlabel{#1}\label{#2}}
\makeatother

\renewenvironment{proof}[1][\proofname]{{\noindent\bfseries Proof of #1. }}{\qed}

\begin{document}

\twocolumn[
\icmltitle{Simultaneous Inference for Massive Data: Distributed Bootstrap}



\icmlsetsymbol{equal}{*}

\begin{icmlauthorlist}
\icmlauthor{Yang Yu}{pu}
\icmlauthor{Shih-Kang Chao}{um}
\icmlauthor{Guang Cheng}{pu}
\end{icmlauthorlist}

\icmlaffiliation{pu}{Department of Statistics, Purdue University, USA}
\icmlaffiliation{um}{Department of Statistics, University of Missouri, USA}

\icmlcorrespondingauthor{Guang Cheng}{chengg@purdue.edu}

\icmlkeywords{Machine Learning, ICML}

\vskip 0.3in
]



\printAffiliationsAndNotice{}  

\begin{abstract}

In this paper, we propose a bootstrap method applied to massive data processed distributedly in a large number of machines. This new method is computationally efficient in that we bootstrap on the master machine without over-resampling, typically required by existing methods \cite{kleiner2014scalable,sengupta2016subsampled}, while provably achieving optimal statistical efficiency with minimal communication. Our method does not require repeatedly re-fitting the model but only applies multiplier bootstrap in the master machine on the gradients received from the worker machines. Simulations validate our theory.

 


\end{abstract}

\section{Introduction}

\subsection{Background}

Modern massive data, with enormous sample size, are usually too hard to fit on a single machine. A master-slave architecture is often adopted using a cluster of nodes for data storage and processing; for example, Hadoop, as one of the most popular distributed framework, has facilitates distributed data processing; see Figure \ref{fig:master_slave} for a diagram of the master-slave architecture \cite{singh2014hadoop}, where the master node has also a portion of the data. A shortcoming of this architecture is that inter-node communication (between master and worker nodes) is through the TCP/IP protocol, which can be over a thousand times slower than intra-node computation and always comes with significant overhead \cite{lan2018communication,fan2019communication}. For these reasons, statistical inference for modern distributed data is very challenging, and communication efficiency is a desirable feature when developing distributed learning algorithms.



\begin{figure}[ht]
\vskip 0.2in
\begin{center}
\centerline{\includegraphics[width=0.8\columnwidth]{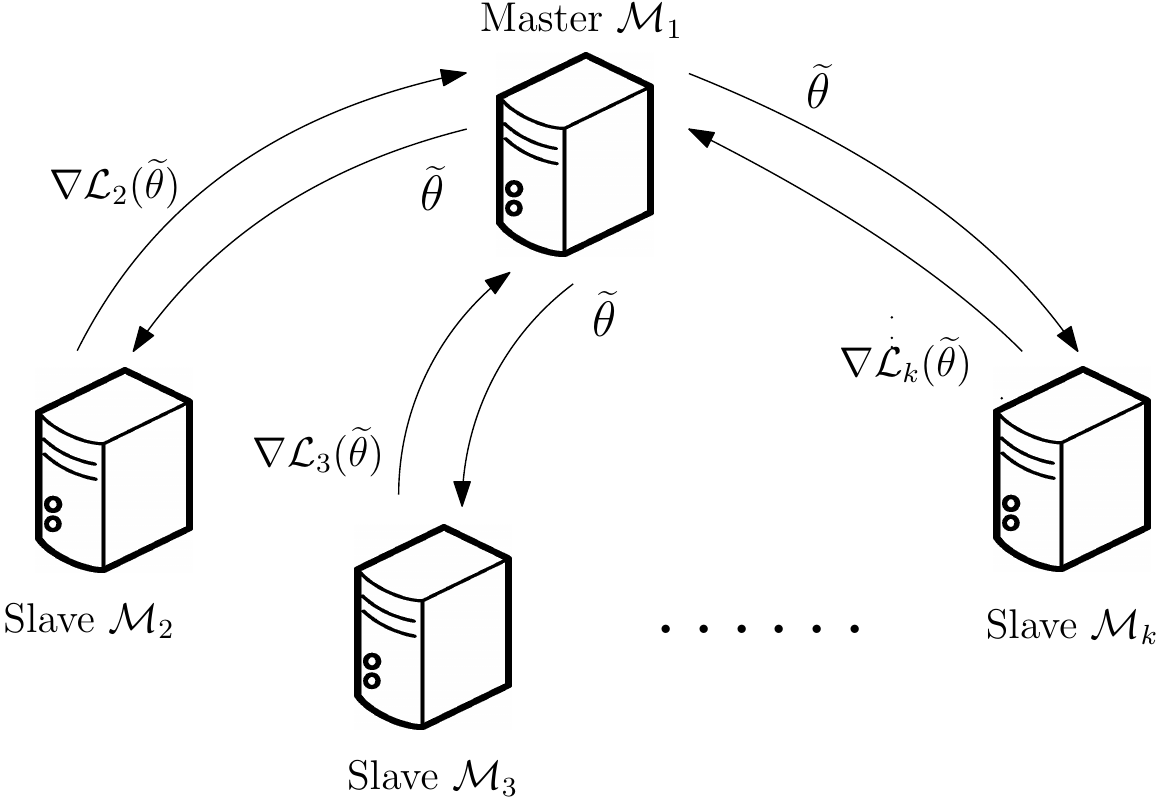}}
\caption{Master-slave architecture for storing and processing distributed data.}
\label{fig:master_slave}
\end{center}
\vskip -2em
\end{figure}

However, classical statistical procedures, which typically require many passes (in hundreds or even thousands) over the entire data set, are very communication-inefficient or even impossible to perform, including popular methods such as bootstrap, Bayesian inference and many maximum likelihood estimation procedures. Over the last few years, many papers proposed computational procedures for estimation from the maximum likelihood criteria \cite{zhang2012communication,li2013statistical,chen2014split,huang2015distributed,battey2015distributed,zhao2016partially,fan2017distributed,lee2017communication,wang2017improved,wang2017efficient,shi2018massive,jordan2019communication,volgushev2019distributed,banerjee2019divide,fan2019communication}. 

As a popular method for approximating the sample distribution of an estimator, Bootstrap, without modifications, is inapplicable in the environment of distributed processing. It typically requires hundreds or thousands of resamples that is of the same size as the original data, which is impossible for large-scale data stored in different locations. 

\subsection{Our Contributions}

In this paper, we first consider a na\"ive bootstrap method, named as \texttt{k-grad}, that uses local gradients from each machine, where $k$ is the number of machines. To provide higher accuracy, an improved version, named as 
\texttt{n+k-1-grad} bootstrap, is introduced. Both are communication (inter-node) and computation (intra-node) efficient 
for 
generalized linear models (GLM). Our methods can be easily extended to other statistical models. The statistical accuracy and efficiency are proved theoretically, and validated by simulations. 

Our \texttt{n+k-1-grad} method overcomes many constraints faced by the existing methods:
\vspace{-1em}
\begin{itemize}[noitemsep]
    \item It preserves bootstrap validity, while relaxing the constraints on the number of machines. 
    \item The computational cost of the bootstrap procedure is as small as it is conducted only on the master node;
    \item It performs statistical inference on 
    a group of parameters simultaneously, rather than on only individual parameters.
\end{itemize}

\subsection{Related Works}


The bag of little bootstraps (BLB) \cite{kleiner2014scalable} is one of the earliest methods that can be used in a distributed setting. However,
to achieve the bootstrap validity, they require that the number of machines has to be smaller than the sample size on local machine, while our methods relax such a requirement. In terms of intra-node computational cost, our methods are more efficient than BLB as expensive model re-fitting on each worker node is not required for obtaining each bootstrap sample (see Table \ref{tab:comp} for an empirical comparison on computational cost). The SDB approach \cite{sengupta2016subsampled} was proposed to improve upon BLB in terms of intra-node computational efficiency; however, it fails for both small and large number of machines, as witnessed in our simulation study.


\subsection{Notations}

We denote the $\ell_p$-norm ($p>0$) of any vector $v=(v_1,\dots,v_n)$ by $\|v\|_p=(\sum_{i=1}^n |v_i|^p)^{1/p}$ ($\|v\|_\infty=\max_{1\leq i\leq n}|v_i|$). We denote the induced $p$-norm and the max-norm of any matrix $M\in\R^{m\times n}$ (with element $M_{ij}$ at $i$-th row and $j$-th column) by $\matrixnorm{M}_p=\sup_{x\in\R^n;\|x\|_p=1} \|Mx\|_p$ and $\matrixnorm{M}_{\max}=\max_{1\leq i\leq m;1\leq j\leq n}|M_{i,j}|$. We write $a\lesssim b$ if $a=O(b)$, and $a\ll b$ if $a=o(b)$.

\section{Methodology}


Suppose i.i.d.\ data $\{Z_i\}_{i=1}^N$ with the same distribution as $Z$ are observed, and $\cL(\theta;Z)$ is a twice-differentiable convex loss function of $\theta=(\theta_1,\dots,\theta_d)\in\R^d$, which depends on a random variable $Z$. Suppose that the parameter of interest $\thetas$ is the minimizer of an expected loss: $$\thetas=\arg\min_{\theta\in\R^d} \cLs(\theta), \mbox{  where $\cLs(\theta)\defn\Ee_Z[\cL(\theta;Z)]$}.$$ 

 
\subsection{Distributed Data Processing} \label{sec:alg}

Assuming the data $\{Z_i\}_{i=1}^N$ is too large to be processed by a single machine, so an estimator for $\thetas$ cannot be straightforwardly obtained by minimizing the empirical loss. Instead, a distributed computation framework will be considered. Suppose the $N$ data are stored distributedly in $k$ machines, where each machine has $n$ data. Denote $\{Z_{ij}\}_{i=1,\dots,n; j=1,...,k}$ the entire data, where $Z_{ij}$ is $i$th datum on the $j$th machine $\cM_j$, and $N=nk$. Without loss of generality, assume that the first machine $\cM_1$ is the master node (see Figure \ref{fig:master_slave}).  
Define the local and global loss functions as 
\begin{align}
\begin{split}
    \mbox{global loss: }\cL_N(\theta)&=\frac1k\sum_{j=1}^k\cL_j(\theta),\quad\mbox{where}\\ \mbox{local loss: }\cL_j(\theta)&=\frac1n\sum_{i=1}^n\cL(\theta;Z_{ij}),\quad j=1,\dots,k. 
\end{split}
\label{eq:loss}
\end{align}
Recall that communication between the master and worker nodes are costly in the parallel processing framework, e.g.\ Hadoop. 

The goal in this paper is to obtain \emph{simultaneous} confidence region for $\thetas$ in low-dimensional regime. Simultaneous inference has become a common problem in many areas of application, such as financial economics, signal processing, marketing analytics, biological sciences, and social science \cite{cai2017large, zhang2017simultaneous}, where researchers want to investigate a group of variables at the same time, instead of a single variable at a time. Variable selection is usually done by simultaneous inference. 

The empirical loss minimizer is defined as:
\begin{align}
	\htheta=\arg\min_{\theta\in\R^d}\cL_N(\theta). \label{eq:ht}
\end{align}
Simultaneous confidence region can be found with confidence $1-\alpha$, for small $0<\alpha<1$, by finding the quantile
\begin{align}
c(\alpha)&\defn\inf\{t\in\R:P(\widehat T \leq t)\geq\alpha\} \quad\text{where} \label{eqn:c} \\
\widehat T&\defn \big\|\sqrt N\big(\htheta-\thetas\big)\big\|_\infty. \label{eqn:that}
\end{align}
The asymptotic distribution of $\htheta$ has been derived \cite{eicker1963asymptotic, gourieroux1981asymptotic}, and confidence regime can be constructed by finding the quantiles of $\widehat T$ in \eqref{eqn:that}.

While the procedure above has been well-developed if the data can be processed with a single machine, implementing $\htheta$ in a distributed framework faces two challenges:
\begin{itemize}
	\item $\htheta$ usually cannot be easily obtained due to significant communication requirement, so statistical inference for $\thetas$ has to be done via 
a surrogate estimator $\ttheta$, 
which imitates the distribution of $\htheta$ that is called the oracle estimator.
	\item Estimating $c(\alpha)$ is usually done via bootstrapping the distribution of \eqref{eqn:that} \cite{dasgupta2008asymptotic, efron1994introduction}. Unfortunately, implementing bootstrap is difficult in the distributed computational framework. The existing methods suffer from high computational cost due to resampling/model refitting in each worker nodes \cite{kleiner2014scalable,sengupta2016subsampled} or requiring a  large number of machines \cite{sengupta2016subsampled}.  
\end{itemize}

To perform statistical inference in distributed computational framework, a surrogate estimator $\ttheta$ satisfying $\|\ttheta-\htheta\|_\infty = o_p(N^{-1/2})$ (if $d$ is fixed) will be obtained (see Section \ref{sec:m}), and then we propose new distributed bootstrap algorithms to estimate the quantile $c(\alpha)$ of $\widehat T$ in \eqref{eqn:that}. 

\subsection{Distributed Bootstrap Algorithms}\label{sec:boot}

The new statistical inferential procedure in this paper is motivated by the fact that $\htheta$ in \eqref{eq:ht} can be expressed like a sample average \cite{he1996general}: 
\begin{align}
\begin{split}
    	&\sqrt N(\htheta-\thetas) \\
	&=\underbrace{-\nabla^2\cLs(\thetas)^{-1}\frac1{\sqrt N}\sum_{i=1}^n\sum_{j=1}^k\nabla\cL(\thetas;Z_{ij})}_{\defn A}+o_P(1). 
\end{split}
\label{eqn:bahadur}
\end{align}
It can be seen that the asymptotic distribution of $\sqrt N(\htheta-\thetas)$ is determined by that of $A$. Note that any surrogate estimator $\ttheta$ satisfying $\|\ttheta-\htheta\|_\infty = o_p(N^{-1/2})$ also has the same expansion. Among many ways to bootstrap the distribution of $A$, we focus on the multiplier bootstrap \cite{chernozhukov2013gaussian, vaart1996weak}. 



Multiplier bootstrap repeatedly generates $N$ i.i.d.\ $\cN(0,1)$ multipliers $\{\epsilon_{ij}^{(b)}\}_{i=1,\dots,n;j=1,\dots,k}$ for each $b=1,...,B$, and then approximate $c(\alpha)$ by the percentile of $\{{W^*}^{(b)}\}_{b=1,\dots,B}$, where
\begin{align}
	{W^*}^{(b)}=\bigg\|-\nabla^2\cLs(\hat\theta)^{-1}\frac1{\sqrt N}\sum_{j=1}^k\sum_{i=1}^n\epsilon_{ij}^{(b)}(\hat \bg_{ij}-\hat \bg)\bigg\|_\infty, \label{eqn:kgrad_2}
\end{align}
with $\hat \bg_{ij}=\nabla\cL(\hat\theta;Z_{ij})$, $\hat\bg=N^{-1}\sum_{j=1}^k\sum_{i=1}^n \hat\bg_{ij}$, and the Hessian $\nabla^2\cL_{N}(\hat\theta)^{-1}$. However, computing ${W^*}^{(b)}$ for one $b$ requires one communication in the distributed computational framework, so the computational cost is formidable when, e.g.\ $B=500$.

To adapt the multiplier bootstrap for distributed computational framework, we propose the $\texttt{k-grad}$ bootstrap, which replaces \eqref{eqn:kgrad_2} by
\begin{align}
	\overline W^{(b)} \defn \bigg\|\underbrace{-\widetilde\Theta\frac1{\sqrt{k}}\sum_{j=1}^k\epsilon_j^{(b)}\sqrt n(\bg_j-\bar\bg)}_{\defn\overline A}\bigg\|_\infty, \label{eqn:wb}
\end{align}
with $\epsilon_j^{(b)}\overset{iid}{\sim}\cN(0,1)$, $\bg_j=\nabla\cL_j(\tilde\theta)$, $\bar\bg = k^{-1}\sum_{j=1}^k \bg_j$, and a surrogate estimator $\tilde\theta$ (Section \ref{sec:m}) to replace $\hat\theta$ for communication efficiency, and a surrogate $\tilde\Theta$ for the Hessian $\nabla^2\cL_{N}(\hat\theta)^{-1}$. Particularly, the computation of $\tilde\Theta$, detailed in Algorithm \ref{alg:kgrad+csl}, will only use the data in the master and $\tilde\theta$. The key advantage of bootstrapping \eqref{eqn:wb} over \eqref{eqn:kgrad_2} is that, once the master has the gradients from the worker nodes, the percentile of $\{{\overline W}^{(b)}\}_{b=1,\dots,B}$ can be computed in the master node only, without the need to communicate with worker nodes. See Algorithm \ref{alg:kgrad} (\texttt{method}=`\texttt{k-grad}') for details.

\begin{algorithm}[tb]
\caption{\texttt{DistBoots$(\text{method},\ttheta,\{\bg_j\}_{j=1,\dots,k},\widetilde\Theta)$}: only need the master node $\cM_1$} \label{alg:kgrad}
\begin{algorithmic}
\STATE {\bfseries Input:} master node $\cM_1$ obtains local gradient $\bg_j$, estimate $\widetilde\Theta$ of inverse population Hessian
\STATE Compute $\bar\bg = k^{-1}\sum_{j=1}^k \bg_j$
\FOR{$b = 1,2,\ldots, B $}
\STATE Generate $k$ independent $\cN(0,1)$: $\{\epsilon_{1}^{(b)},\epsilon_{2}^{(b)},\ldots,\epsilon_{k}^{(b)}\}$\label{line:kmult}
\IF{\texttt{method}=`\texttt{k-grad}'}
\STATE Compute $W^{(b)}$ by \eqref{eqn:wb}
\ELSIF{\texttt{method}=`\texttt{n+k-1-grad}'}
\STATE Compute $W^{(b)}$ by \eqref{eqn:wt}
\ENDIF
\ENDFOR
\STATE Compute the percentile $c_{W}(\alpha)$ of $\{W_1,W_2,...,W_B\}$ for $\alpha\in(0,1)$ 
\STATE Return $\ttheta_l\pm N^{-1/2}c_W(\alpha)$, $l=1,\dots,d$
\end{algorithmic}
\end{algorithm}

%

A problem with the \texttt{k-grad} procedure is that it may perform poorly when $k$ is small, e.g.\ $k=2$ or $3$, as can be seen from the simulation analysis (Section \ref{sec:exp}). This is due to the failure of bootstrapping the variance with only 2 or 3 multipliers. 
This problem can be alleviated by using a unique multiplier to each datum in the master node $\cM_1$; that is, 
\begin{align}
\begin{split}
    	\widetilde W^{(b)}\defn \bigg\|&-\widetilde\Theta\frac1{\sqrt{n+k-1}}\bigg(\sum_{i=1}^n\epsilon_{i1}^{(b)}(\bg_{i1}-\bar\bg) \\
	&\underbrace{\hspace{60pt}+\sum_{j=2}^k\epsilon_j^{(b)}\sqrt n(\bg_j-\bar\bg) \bigg)}_{\defn\widetilde A}\bigg\|_\infty. 
\end{split}
\label{eqn:wt}
\end{align}
where $\epsilon_{i1}^{(b)}$ and $\epsilon_j^{(b)}$ are i.i.d. $\cN(0,1)$ multipliers in $i$, $j$ and $b$, and $\bg_{i1}=\nabla\cL(\tilde\theta;Z_{i1})$ is based on a single datum $Z_{i1}$ in the master. We call this method the \texttt{n+k-1-grad}. Note that the percentile of $\{{\widetilde W}^{(b)}\}_{b=1,\dots,B}$ can still be computed using only $\cM_1$, without needing to communicate with other machines. See Algorithm \ref{alg:kgrad} (\texttt{method}=`\texttt{n+k-1-grad}') for details. $\texttt{n+k-1-grad}$ can apply even when $k$ is small.

Besides simultaneous inference, our methods also apply to other problems, such as pointwise confidence intervals and confidence regions of other shapes, by replacing $\|\cdot\|_\infty$ with $|(\cdot)_l|$, $\|\cdot\|_2$, and so on, where we denote by $(\cdot)_l$ the $l$-th element of a vector.

\subsection{CSL Estimator}\label{sec:m}

To apply \texttt{k-grad} or \texttt{n+k-1-grad}, we need a surrogate estimator $\ttheta$ of $\htheta$.
We adopt the communication-efficient surrogate likelihood algorithm [CSL, \cite{jordan2019communication}], which achieves the same rate as $\htheta$ at the cost of one or more rounds of communication. The CSL estimator converges to $\htheta$ even if $n\leq k$ with sufficient rounds of communication/iteration, and when $n>k$, only one round of communication is required to achieve $\|\ttheta-\htheta\|_\infty=o_p(N^{-1/2})$ if $d$ is fixed. 
See Algorithm \ref{alg:kgrad+csl} for a detailed description.


\begin{algorithm}[tb]
\caption{\texttt{k-grad}/\texttt{n+k-1-grad} with CSL: $\tau$ rounds of communication, $\tau\geq 1$}\label{alg:kgrad+csl}
\begin{algorithmic}
\STATE Compute $\ttheta^{(0)}=\argmin_\theta \cL_1(\theta)$ at $\cM_1$
\FOR{$t = 1,\ldots, \tau $}
\STATE Transmit $\ttheta^{(t-1)}$ to $\{\cM_j\}_{j=2,\dots,k}$
\STATE Compute $\nabla\cL_1(\ttheta^{(t-1)})$ and $\nabla^2\cL_1(\ttheta^{(t-1)})^{-1}$ at $\cM_1$
\FOR{$j = 2,\ldots, k $}
\STATE Compute $\nabla\cL_j(\ttheta^{(t-1)})$ at $\cM_j$
\STATE Transmit $\nabla\cL_j(\ttheta^{(t-1)})$ to $\cM_1$
\ENDFOR
\STATE $\nabla\cL_N(\ttheta^{(t-1)})\gets k^{-1}\sum_{j=1}^k\nabla\cL_j(\ttheta^{(t-1)})$ at $\cM_1$
\STATE $\ttheta^{(t)}\gets\ttheta^{(t-1)}-\nabla^2\cL_1(\ttheta^{(t-1)})^{-1}\nabla\cL_N(\ttheta^{(t-1)})$ at $\cM_1$
\ENDFOR
\STATE Run \texttt{DistBoots}$(\text{`\texttt{k-grad}' or `\texttt{n+k-1-grad}'},$ \\
\hspace{78pt}$\ttheta=\ttheta^{(\tau)},\{\bg_j=\nabla\cL_j(\ttheta^{(\tau-1)})\}_{j=1}^k,$ \\
\hspace{78pt}$\tT=\nabla^2\cL_1(\ttheta^{(\tau-1)})^{-1})$ at $\cM_1$ \label{line:save}
\end{algorithmic}
\end{algorithm}
\vskip -2em

\section{Theoretical Results}

Section \ref{sec:over} provides an overview of the theoretical results.
 Section \ref{sec:ld_lm} presents the theory in a linear model framework for \texttt{k-grad} and \texttt{n+k-1-grad}. Section \ref{sec:ld_glm} shows the results for the generalized linear models (GLM).

\subsection{An Overview}\label{sec:over}

Figure \ref{fig:tau} shows the minimal number of iterations $\tau_{\min}$ (communication rounds) that is sufficient for the bootstrap validity. Panels in the top row of Figure \ref{fig:tau} illustrate the lower bound of $\tau$ for linear models given in Theorems \ref{theo:ld0_csl} and \ref{theo:ld_csl} of Section \ref{sec:ld_lm}, and those in the bottom row illustrating the results for the generalized linear models given in Theorem \ref{theo:ld0_glm_csl} and \ref{theo:ld_glm_csl} of Section \ref{sec:ld_glm}

As a general pattern of Figure \ref{fig:tau}, 
$\tau_{\min}$ is increasing in $k$ (decreasing in $n$) for both \texttt{k-grad} and \texttt{n+k-1-grad} and (generalized) linear model; 
in addition, $\tau_{\min}$ is (logarithmically) increasing in $d$. 

For the difference between \texttt{k-grad} and \texttt{n+k-1-grad}, we compare the left and right panel of Figure \ref{fig:tau}. With fixed $(n,k,d)$, the $\tau_{\min}$ for \texttt{n+k-1-grad} is always no larger than that for \texttt{k-grad}, which indicates a greater efficiency of \texttt{n+k-1-grad}. As $k$ is small, \texttt{k-grad} would not work, while \texttt{n+k-1-grad} can provably work. In addition, $\tau_{\min}=1$ can work for certain instances of \texttt{n+k-1-grad} but never for \texttt{k-grad}.

For the comparison between the linear model (top panels) and generalized linear model (bottom panels), GLMs require larger $n$ than linear models in order to ensure our bootstrap procedures work.

\begin{figure}[ht]
\vskip 0.2in
\begin{center}
\centerline{\includegraphics[width=\columnwidth]{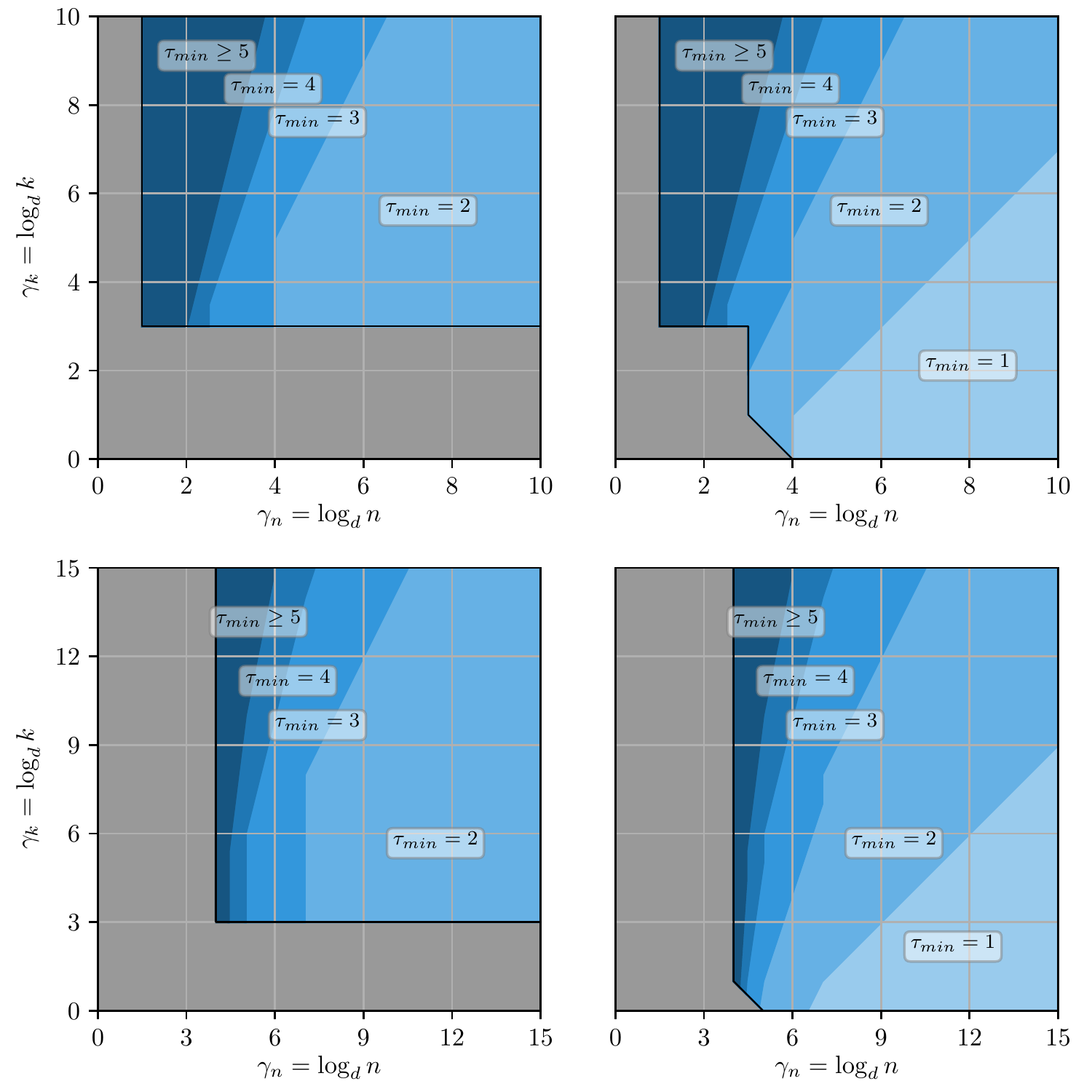}}
\caption{Illustration of Theorems \ref{theo:ld0_csl} (\textbf{top left}: linear model, \texttt{k-grad}), \ref{theo:ld_csl} (\textbf{top right}: linear model, \texttt{n+k-1-grad}), \ref{theo:ld0_glm_csl} (\textbf{bottom left}: GLM, \texttt{k-grad}), and \ref{theo:ld_glm_csl} (\textbf{bottom right}: GLM, \texttt{n+k-1-grad}). Gray area represents the region where the theorems do not validate the bootstrap procedures, and the other area is colored blue of varying lightness according to the lower bound of iteration $\tau$. 
}
\label{fig:tau}
\end{center}
\vskip -2em
\end{figure}

\subsection{Linear Model} \label{sec:ld_lm}

For simplicity, we start with the linear model. Suppose that $N$ i.i.d.\ observations come from a linear model, $y=x^\top\thetas+e$,
with unknown coefficient vector $\thetas\in\R^d$, covariate random vector $x\in\R^d$, and noise $e\in\R$ independent of $x$ with zero mean and variance of $\sigma^2$.  We define $\Sigma=\Ee[xx^\top]$ with its inverse $\Theta=\Sigma^{-1}$.  We consider the least-squares loss $\cL(\theta;z)=\cL(\theta;x,y)=(y-x^\top\theta)^2/2$.  We impose the following assumptions on the linear model.
\begin{itemize}
	\labitemc{(A1)}{as:design} $x$ is sub-Gaussian, that is,
	$$\sup_{\|w\|_2\leq1}\Ee\big[\exp[(w^\top x)^2/L^2]\big]=O(1),$$
	for some absolute constant $L>0$.  Moreover, $1/\lambdamin(\Sigma)\leq\mu$ for some absolute constant $\mu>0$.
	
	\labitemc{(A2)}{as:noise} $e$ is sub-Gaussian, that is,
	$$\Ee\big[\exp[e^2/L'^2]\big]=O(1),$$
	for some absolute constant $L'>0$.  Moreover, $\sigma>0$ is an absolute constant.
	
\end{itemize}

Under the assumptions, we first investigate the theoretical property of Algorithm \ref{alg:kgrad+csl}, where we apply \texttt{k-grad} along with the CSL estimator that takes advantage of multiple rounds of communication.  We define
\begin{align}
    T&\defn \big\|\sqrt N\big(\ttheta-\thetas\big)\big\|_\infty,\quad\text{and} \label{eqn:t} \\
    c_{\overline W}(\alpha)&\defn\inf\{t\in\R:P_\epsilon(\overline W\leq t)\geq\alpha\}, \notag
\end{align}
where $P_\epsilon$ denotes the probability with respect to the randomness from all the multipliers, $\overline W$ has the same distribution as $\overline W^{(b)}$ in \eqref{eqn:wb}, and $\ttheta$ and $\btheta$ are the $\tau$-step and $\tau-1$-step CSL estimators as specified in Algorithm \ref{alg:kgrad+csl}. Now, we state a result for \texttt{k-grad} bootstrap procedure with the CSL estimator.

\begin{theo}[\texttt{k-grad}, linear model]\label{theo:ld0_csl}
	Suppose \ref{as:design}-\ref{as:noise} hold, and that we run Algorithm \ref{alg:kgrad+csl} with \texttt{k-grad} method in linear model.  Assume $n=d^{\gamma_n}$ and $k=d^{\gamma_k}$ for some constants $\gamma_n,\gamma_k\geq0$. If $\gamma_n>1$, $\gamma_k>3$, $\tau\geq\tau_{\min}$, where
	$$\tau_{\min}=1+\bigg\lfloor\max\bigg\{\frac{\gamma_k+1}{\gamma_n-1},1+\frac{3}{\gamma_n-1}\bigg\}\bigg\rfloor,$$
	then we have
	\begin{align}
	\sup_{\alpha\in(0,1)}|P(T\leq c_{\overline W}(\alpha))-\alpha|=o(1). \label{eqn:kgradthm}
	\end{align}
	  In addition, \eqref{eqn:kgradthm} also holds if $T$ is replaced by $\widehat T$.
\end{theo}

Theorem \ref{theo:ld0_csl} states that under certain conditions, simultaneous confidence region given by Algorithm \ref{alg:kgrad+csl} with \texttt{k-grad} method provides sufficient coverage. It also suggests that the bootstrap quantile approximates the quantile of the centralized estimator $\htheta$, and therefore, the bootstrap procedure is also statistically efficient.

Next, we present a theorem that establishes the validity and the efficiency of \texttt{n+k-1-grad} bootstrap procedure in Algorithm \ref{alg:kgrad+csl}.  
We define 
$$c_{\widetilde W}(\alpha)\defn\inf\{t\in\R:P_\epsilon(\widetilde W\leq t)\geq\alpha\},$$
where $\widetilde W$ has the same distribution as $\widetilde W^{(b)}$ in \eqref{eqn:wt}.

\begin{theo}[\texttt{n+k-1-grad}, linear model]\label{theo:ld_csl}
	Suppose \ref{as:design}-\ref{as:noise} hold, and that we run Algorithm \ref{alg:kgrad+csl} with \texttt{n+k-1-grad} method in linear model.  Assume $n=d^{\gamma_n}$ and $k=d^{\gamma_k}$ for some constants $\gamma_n,\gamma_k\geq0$. If $\gamma_n>1$, $\gamma_n\vee\gamma_k>3$, $\gamma_n+\gamma_k>4$, $\tau\geq\tau_{\min}$, where
	$$\tau_{\min}=1+\bigg\lfloor\frac{(\gamma_k-1)\vee(\gamma_n\wedge\gamma_k)\vee1+2}{\gamma_n-1}\bigg\rfloor,$$
	then we have
	\begin{align}
	\sup_{\alpha\in(0,1)}|P(T\leq c_{\widetilde W}(\alpha))-\alpha|=o(1). \label{eqn:nk1gradthm}
	\end{align}
	  In addition, \eqref{eqn:nk1gradthm} also holds if $T$ is replaced by $\widehat T$.
\end{theo}

For a deeper look into the difference between \texttt{k-grad} and \texttt{n+k-1-grad}, we compare 
the difference between the covariance of the oracle score $A$ [defined in \eqref{eqn:bahadur}] and the conditional covariance of $\overline A$ (for \texttt{k-grad} [defined in \eqref{eqn:wb}], and $\widetilde A$ for \texttt{n+k-1-grad} [defined in \eqref{eqn:wt}]) conditioning on the data. These key quantities which determine how well the bootstrap procedure approximates the distribution of $\widehat T$. Conditioning on the data, we have the bounds 

\begin{align}
\begin{split}
&\matrixnorm{\cov_\epsilon(\overline A)-\cov(A)}_{\max}\leq d\|\ttheta^{(\tau-1)}-\thetas\|_1 \\
&\hspace{10pt}+nd\|\ttheta^{(\tau-1)}-\thetas\|_1^2 + O_P(\sqrt{d^2/k}+\sqrt{d/n}), 
\end{split}
\label{eqn:kgrad_err} \\
\begin{split}
&\matrixnorm{\cov_\epsilon(\widetilde A)-\cov(A)}_{\max}\leq d\|\ttheta^{(\tau-1)}-\thetas\|_1 \\
&+(n\wedge k)d\|\ttheta^{(\tau-1)}-\thetas\|_1^2 + O_P(\sqrt{d^2/(n+k)}+\sqrt{d/n}), 
\end{split}
\label{eqn:nk1grad_err}
\end{align}

up to logarithmic factors, provided $n\gtrsim d$. Comparing the two preceding equations, we first see that overall, \texttt{n+k-1-grad} \eqref{eqn:nk1grad_err} has a smaller error than \texttt{k-grad} \eqref{eqn:kgrad_err}. In particular, \texttt{k-grad} requires both $n$ and $k$ to be large, while \texttt{n+k-1-grad} requires a large $n$ but not a large $k$. 
In addition, a single round of communication could be enough for \texttt{n+k-1-grad}, but not for \texttt{k-grad}. To see it, if $\tau=1$, $\|\ttheta^{(0)}-\thetas\|_1$ is of order $O_P(d/\sqrt n)$, and the right-hand side of \eqref{eqn:kgrad_err} will grow with $d$; by contrast, the error in \eqref{eqn:nk1grad_err} still shrinks to zero as long as $k\ll n$.

\begin{rem}
Given that $d$ is fixed, $\tau=\lceil\log k/\log n\rceil$ is enough for CSL to achieve the optimal statistical rate \cite{jordan2019communication}. Under same circumstance, bootstrap consistency is warranted at the expense of at most one additional communication round $\tau_{\min}=1+\lfloor\log k/\log n\rfloor$ (Theorem \ref{theo:ld_csl}). 
\end{rem}

\begin{rem}
To apply BLB in the distributed setting, $k\lesssim n$ is required to achieve the higher order correctness of the bootstrap procedure \cite{kleiner2014scalable}. 
We conjecture that SDB requires $k\lesssim n$ as well, based on the  observations from simulation study in Section \ref{sec:exp2}. In contrast to BLB and SDB, \texttt{k-grad} (if $k\gg d^3$) and \texttt{n+k-1-grad} are both scalable to $k\gg n$, at the cost of a larger $\tau$. 
\end{rem}

\begin{rem} The non-asymptotic rate of $\sup_{\alpha\in(0,1)}\left|P(T\leq c_{\overline W}(\alpha))-\alpha\right|$ may be proven to be polynomial in $n$ and $k$, with a more delicate analysis. As an alternative, simultaneous inference can also be done with the the alternative extreme value distribution approach, but the convergence rate is at best logarithmic \cite{chernozhukov2013gaussian,zhang2017simultaneous}.
\end{rem}


\subsection{Generalized Linear Model} \label{sec:ld_glm}

In this section, we consider generalized linear models (GLMs), which generate i.i.d.\ observations $(x,y)\in\R^d\times\R$.
We assume that the loss function $\cL$ is of the form $\cL(\theta;z)=g(y,x^\top\theta)$ for $\theta,x\in\R^d$ and $y\in\R$ with $g:\R\times\R\to\R$, and $g(a,b)$ is three times differentiable with respect to $b$, and denote $\frac{\partial}{\partial b}g(a,b)$, $\left(\frac{\partial}{\partial b}\right)^2 g(a,b)$, $\left(\frac{\partial}{\partial b}\right)^3 g(a,b)$ by $g'(a,b)$, $g''(a,b)$, $g'''(a,b)$ respectively.  We let $\thetas$ be the unique minimizer of the expected loss $\cLs(\theta)$.  We impose the following assumptions on the GLM.

\begin{itemize}
	\labitemc{(B1)}{as:smth_glm} For some $\Delta>0$, and $\Delta'>0$ such that $|x^\top\thetas|\leq\Delta'$ almost surely,
	\begin{align*}
	\sup_{|b|\vee|b'|\leq\Delta+\Delta'}&\sup_a\frac{|g''(a,b)-g''(a,b')|}{|b-b'|}\leq1, \\
	\max_{|b_0|\leq\Delta} &\sup_a |g'(a,b_0)|=O(1),\quad\text{and} \\
	\max_{|b|\leq\Delta+\Delta'} &\sup_a |g''(a,b)|=O(1).
	\end{align*}
	
	\labitemc{(B2)}{as:design_glm} $\|x\|_\infty=O(1)$.
	
	\labitemc{(B3)}{as:hes_glm} The smallest and largest eigenvalues of $\nabla^2\cLs(\thetas)$ and $\Ee\left[\nabla\cL(\thetas;Z)\nabla\cL(\thetas;Z)^\top\right]$ are bounded away from zero and infinity respectively.
	
	
	\labitemc{(B4)}{as:subexp_glm} For some constant $L>0$,
	$$\max_l\max_{q=1,2}\Ee[|\bh_l^{2+q}|/L^q]+\Ee[\exp(|\bh_l|/L)]=O(1),\quad\text{or}$$
	$$\max_l\max_{q=1,2}\Ee[|\bh_l^{2+q}|/L^q]+\Ee[(\max_l|\bh_l|/L)^4]=O(1),$$
	where $\bh=\nabla^2\cLs(\thetas)^{-1}\nabla\cL(\thetas;Z)$ and $\bh_l$ is the $l$-th coordinate.


	
\end{itemize}

 Assumption \ref{as:smth_glm} imposes smoothness conditions on the loss function. For example, the logistic regression model has $g(a,b)=-ab+\log(1+\exp[b])$.  It is easy to see that $|g'(a,b)|\leq2$, $|g''(a,b)|\leq1$, $|g'''(a,b)|\leq1$. Therefore, Assumption \ref{as:smth_glm} is met for the loss function of the logistic regression model. Assumption \ref{as:design_glm} imposes boundedness condition on the input variables.  Assumption \ref{as:hes_glm} is a standard assumption in GLM literature. Assumption \ref{as:subexp_glm} is required for proving the validity of multiplier bootstrap \cite{chernozhukov2013gaussian}.

The following two theorems states the validity and the efficiency of \texttt{k-grad} and \texttt{n+k-1-grad} in GLM.  Recall the definitions of $T$, $\overline W$, and $\widetilde W$ in \eqref{eqn:t}, \eqref{eqn:wb}, and \eqref{eqn:wt}, respectively.

\begin{theo}[\texttt{k-grad}, GLM]\label{theo:ld0_glm_csl}
	Suppose \ref{as:smth_glm}-\ref{as:subexp_glm} hold, and that we run Algorithm \ref{alg:kgrad+csl} with \texttt{k-grad} method in GLM.  Assume $n=d^{\gamma_n}$ and $k=d^{\gamma_k}$ for some constants $\gamma_n,\gamma_k\geq0$. If $\gamma_n>4$, $\gamma_k>3$, $\tau\geq\tau_{\min}$, where
	\begin{align*}
	\tau_{\min}&=\tau_0+\max\bigg\{\bigg\lfloor\frac{\gamma_k-2}{\gamma_n-1}+\nu_0\bigg\rfloor, 1\bigg\},
	\end{align*}
	$$\tau_0=1+\left\lfloor\log_2\frac{\gamma_n-1}{\gamma_n-4}\right\rfloor,\quad\nu_0=2-\frac{2^{\tau_0}(\gamma_n-4)}{\gamma_n-1}\in(0,1],$$ then we have \eqref{eqn:kgradthm}.   In addition, \eqref{eqn:kgradthm} also holds if $T$ is replaced by $\widehat T$.
\end{theo}

\begin{theo}[\texttt{n+k-1-grad}, GLM]\label{theo:ld_glm_csl}
	Suppose \ref{as:smth_glm}-\ref{as:subexp_glm} hold, and that we run Algorithm \ref{alg:kgrad+csl} with \texttt{n+k-1-grad} method in GLM.  Assume $n=d^{\gamma_n}$ and $k=d^{\gamma_k}$ for some constants $\gamma_n,\gamma_k\geq0$. If $\gamma_n>4$, $\gamma_n+\gamma_k>5$, $\tau\geq\tau_{\min}$, where
	\begin{align*}
	\tau_{\min}=\tau_0+\bigg\lfloor\frac{(\gamma_k-1)\vee(\gamma_n\wedge\gamma_k)-1}{\gamma_n-1}+\nu_0\bigg\rfloor,
	\end{align*}
	$$\tau_0=1+\bigg\lfloor\log_2\frac{\gamma_n-1}{\gamma_n-4}\bigg\rfloor,\quad\nu_0=2-\frac{2^{\tau_0}(\gamma_n-4)}{\gamma_n-1}\in(0,1],$$then we have \eqref{eqn:nk1gradthm}.   In addition, \eqref{eqn:nk1gradthm} also holds if $T$ is replaced by $\widehat T$.
\end{theo}

See Figure \ref{fig:tau} for a comparison between the results of linear models and GLMs.

\begin{rem}
In both Theorems \ref{theo:ld0_glm_csl} and \ref{theo:ld_glm_csl}, $\tau_0$ is the communication rounds needed for the CSL estimator to go through the regions which are far from $\thetas$. As $d$ grows, the time spent in these regions can increase. However, when $n$ is large, e.g. $n\gg d^7$, the loss function is more well-behaved, and the time required reduces to $\tau_0=1$. 
\end{rem}



\section{Experiments} \label{sec:exp}

\subsection{Accuracy and Efficiency} \label{sec:exp1}

Fix total sample size $N=2^{16}$. Choose $d$ from $\{2^1,2^3,2^5,2^7\}$ and $k$ from $\{2^0,2^1,2^2,\dots,2^{11}\}$. $\thetas$ is determined by drawing uniformly from $[-0.5,0.5]^d$ and keep it fixed for all replications.  We generate each covariate vector $x$ independently from $\cN(0,\Sigma)$ and specify two different covariance matrices: Toeplitz ($\Sigma_{l,l'}=0.9^{|l-l'|}$) and equi-correlation ($\Sigma_{l,l'}=0.8$ for all $l\neq l'$, $\Sigma_{l,l}=1$ for all $l$), and the results for the latter are deferred to the appendix as they are similar to that under the Toeplitz design. For linear model, we generate $e$ independently from $\cN(0,1)$, simulate the response from $y=x^\top\thetas+e$; for GLM, we consider logistic regression and obtain each response from $y\sim\text{Ber}(1/(1+\exp[-x^\top\thetas]))$.  Under each choice of $d$ and $k$, we run \texttt{k-grad} and \texttt{n+k-1-grad} with CSL on $1000$ independent data sets, and compute the empirical coverage probability and the average width based on the results from these $1000$ replications.  At each replication, we draw $B=500$ bootstrap samples, from which we calculate the $95\%$ empirical quantile to further obtain the $95\%$ simultaneous confidence interval (the level $95\%$ is represented by a black solid line in all figures).

The average width is compared with the oracle width. We compute the oracle width (represented by a black dashed line in all figures) for each model as follows.  For a fixed $N$ and $d$, we generate $500$ independent data sets, and for each data set, we compute the centralized $\htheta$.  The oracle width is defined as two times the $95\%$ empirical quantile of $\|\htheta-\thetas\|_\infty$.

The empirical coverage probabilities and the average widths of \texttt{k-grad} and \texttt{n+k-1-grad} are displayed in Figures \ref{fig:lm_tp} (linear regression with Toeplitz design) and \ref{fig:glm_tp} (logistic regression with Toeplitz design). 
Note that the sub-sample size $n$ is determined by $k$ as $N$ is fixed, and therefore, a larger $k$ indicates a smaller $n$.

When $k$ is small, \texttt{k-grad} fails because $k$ multipliers cannot provide enough perturbation to approximate the sampling distribution whereas \texttt{n+k-1-grad} has a good coverage (Theorems \ref{theo:ld_csl} and \ref{theo:ld_glm_csl}). When $k$ gets too large (or $n$ gets too small), the coverage of both algorithms starts to fall, due to both the deviation of the center (the estimator $\ttheta^{(\tau)}$) from the centralized estimator $\htheta$ and the deviation of the width from the oracle width [\eqref{eqn:kgrad_err} and \eqref{eqn:nk1grad_err}].  We also see that the larger the dimension, the harder for both algorithms to achieve $95\%$ coverage, and the earlier both algorithm fail as $k$ grows (or $n$ decreases) [\eqref{eqn:kgrad_err} and \eqref{eqn:nk1grad_err}].  However, increasing the number of communication rounds improves the coverage, and thus, the coverage of both algorithms, even when $k\geq n$. When $k$ is too large (or $n$ is too small; see, for example, Figure \ref{fig:lm_tp}, \texttt{n+k-1-grad}, $d=2^7$), the width could go further away from the oracle width as the number of communication rounds increases, as predicted by the increase of the right-hand sides of both \eqref{eqn:kgrad_err} and \eqref{eqn:nk1grad_err} as $n$ decreases.

The cases of $d=2^3$ and $2^5$ and the equi-correlation case are deferred to the appendix, as the patterns are similar to Figure \ref{fig:lm_tp} and Figure \ref{fig:glm_tp}. Results on pointwise confidence intervals are also included in the appendix.

\begin{figure}[ht]
\vskip 0.2in
\begin{center}
\centerline{\includegraphics[width=\columnwidth]{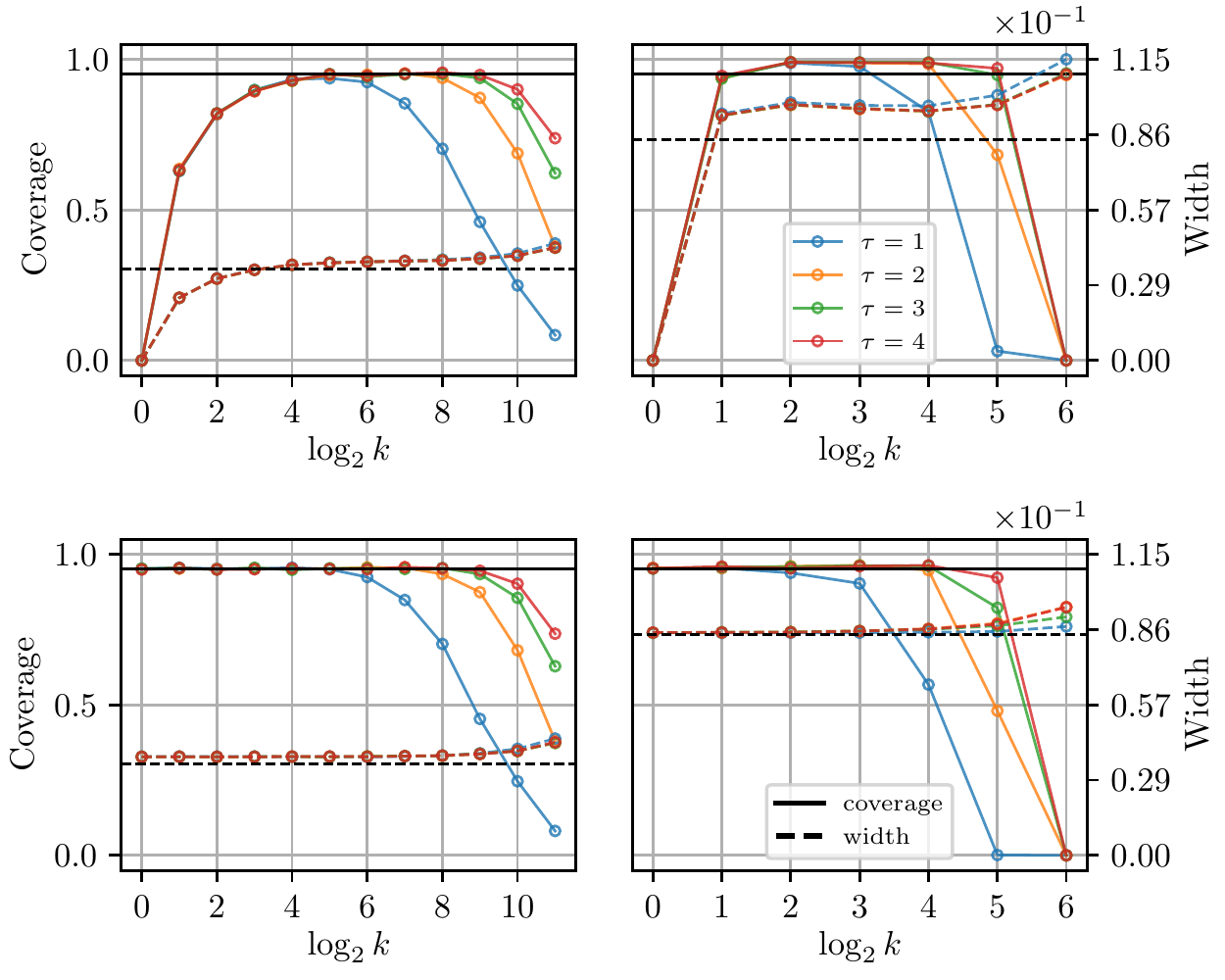}}
\caption{Empirical coverage probability (\textbf{left axis}) and average width (\textbf{right axis}) of simultaneous confidence intervals by \texttt{k-grad} (\textbf{top}) and \texttt{n+k-1-grad} (\textbf{bottom}) in a linear regression model with varying dimension (\textbf{left}: $d=2^1$, \textbf{right}: $d=2^7$).  Black solid line represents nominal confidence level ($95\%$) and black dashed line represents oracle width.}
\label{fig:lm_tp}
\end{center}
\vskip -2em
\end{figure}

\begin{figure}[ht]
\vskip 0.2in
\begin{center}
\centerline{\includegraphics[width=\columnwidth]{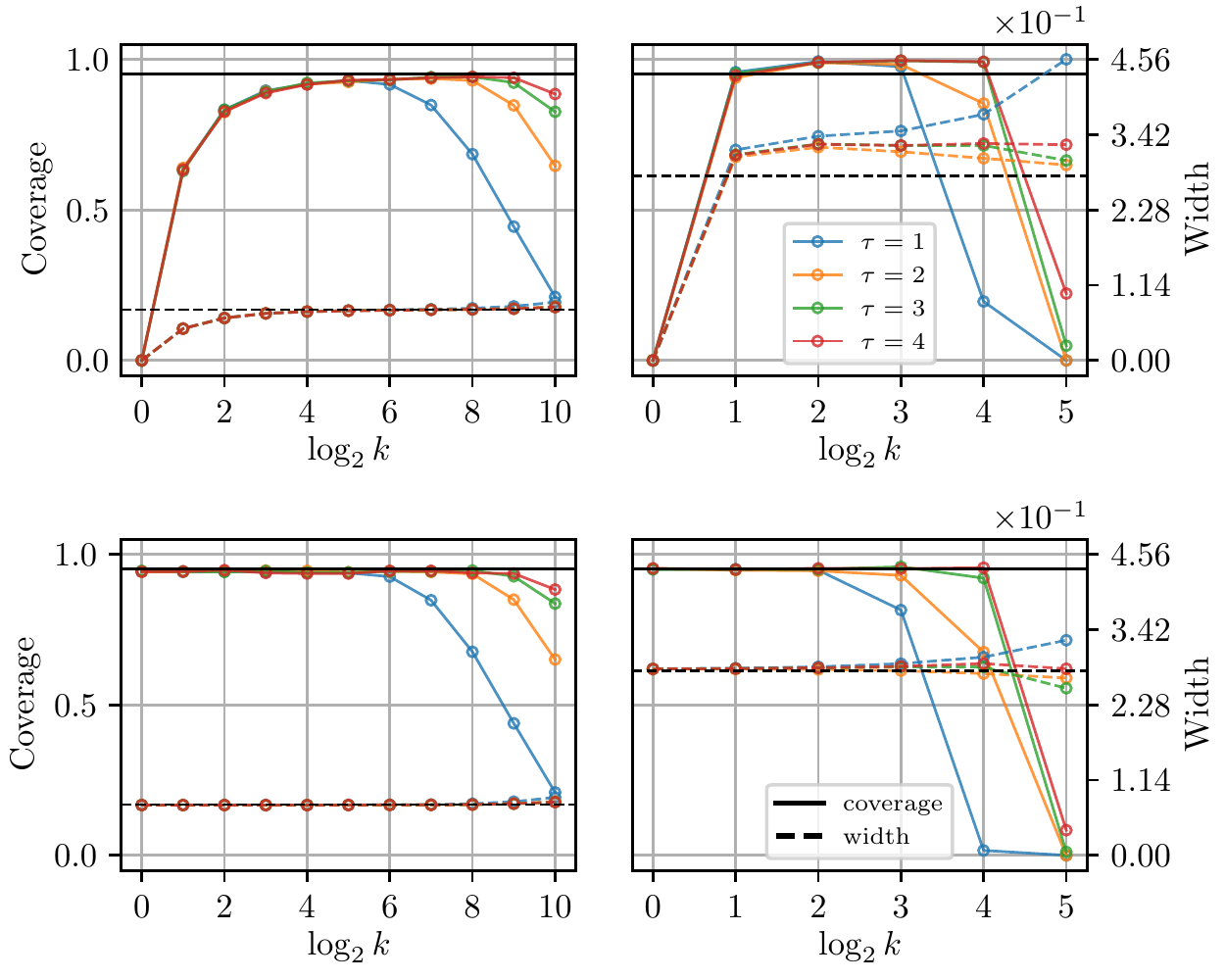}}
\caption{Empirical coverage probability (\textbf{left axis}) and average width (\textbf{right axis}) of simultaneous confidence intervals by \texttt{k-grad} (\textbf{top}) and \texttt{n+k-1-grad} (\textbf{bottom}) in a logistic regression model with varying dimension (\textbf{left}: $d=2^1$, \textbf{right}: $d=2^7$).  Black solid line represents nominal confidence level ($95\%$) and black dashed line represents oracle width.}
\label{fig:glm_tp}
\end{center}
\vskip -2em
\end{figure}

\subsection{Comparisons to existing methods: BLB and SDB} \label{sec:exp2}

We compare the width of \texttt{k-grad} and \texttt{n+k-1-grad} against two bootstrap procedures, BLB \cite{kleiner2014scalable} and SDB \cite{sengupta2016subsampled}, using Toeplitz design and similar experiment setting in Section \ref{sec:exp1}. 
We use BLB and SDB to compute the width of a confidence interval and compare it against the oracle width, instead of constructing the entire confidence interval. The results are displayed in Figures \ref{fig:cmpr_tp}.

SDB always has a significant deviation from the oracle width for small $k$ and has the same behavior as BLB when $k$ is large. 
The width of \texttt{n+k-1-grad} is closer to the oracle width than \texttt{k-grad}, as discussed in Section \ref{sec:exp1}. 

\begin{figure}[ht]
\vskip 0.2in
\begin{center}
\centerline{\includegraphics[width=\columnwidth]{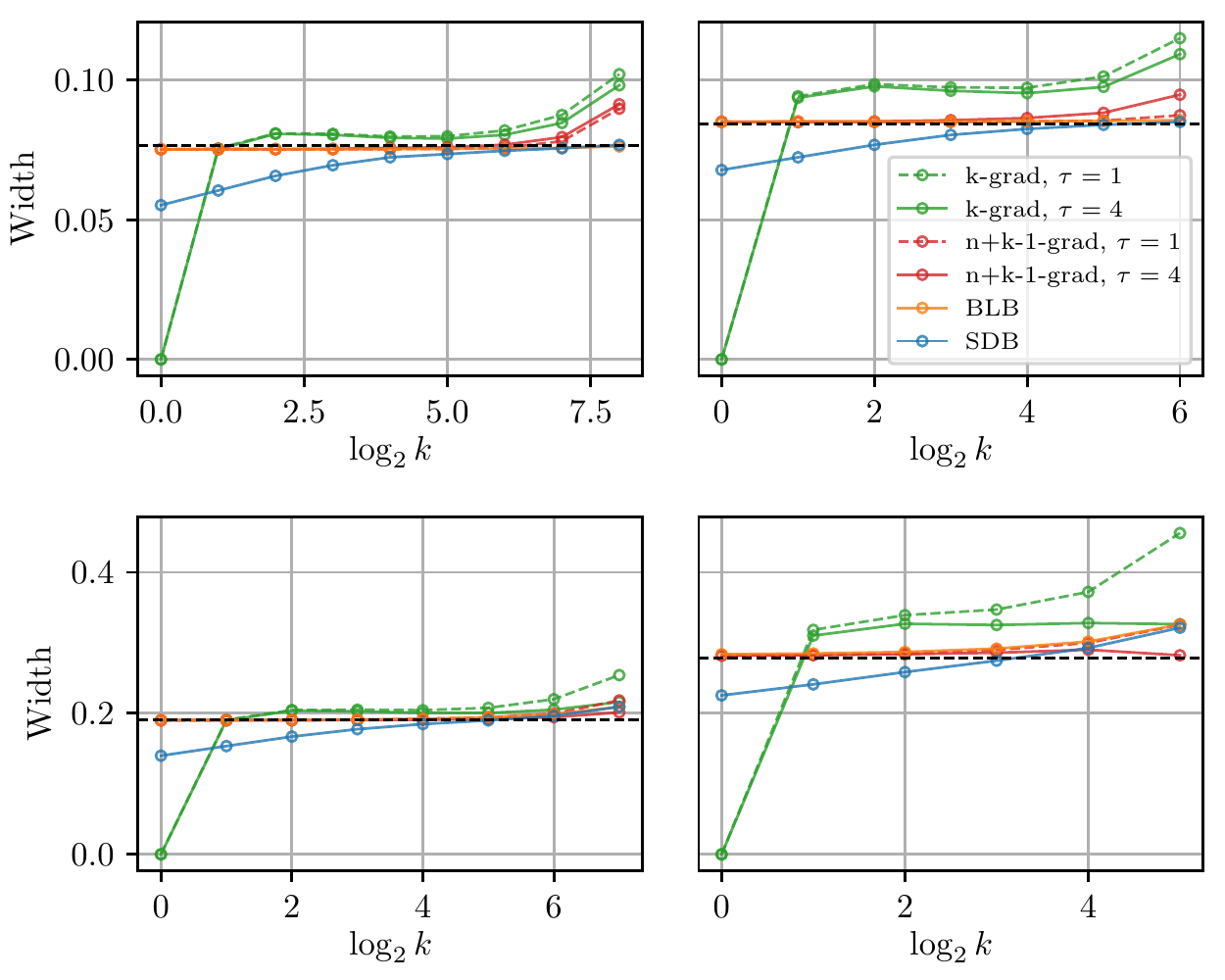}}
\caption{Comparison of \texttt{k-grad}, \texttt{n+k-1-grad}, BLB, and SDB in average width of simultaneous confidence intervals in linear regression (\textbf{top}) and logistic regression (\textbf{bottom}) with varying dimension (\textbf{left}: $d=2^5$, \textbf{right}: $d=2^7$). Black dashed line represents oracle width.}
\label{fig:cmpr_tp}
\end{center}
\vskip -2em
\end{figure}

As \texttt{n+k-1-grad} and BLB appear to be the two best-performing methods, we compare the two into more details. For linear regression, \texttt{n+k-1-grad} performs as well as BLB, except in a few cases of large $k$. 
For logistic regression, the width of both \texttt{n+k-1-grad} and BLB deviate from the oracle width for large $k$, but \texttt{n+k-1-grad} mostly outperforms BLB, because $n/k$ is too small for BLB, while \texttt{n+k-1-grad} improves as the number of communications $\tau$ increases.  

\subsection{Computational time}

Table \ref{tab:comp} shows the computational time of different bootstrap methods. The average run time (in second) is computed with $50$ independent runs, and in each run a bootstrap method is carried out for linear regression model with Toeplitz design. We set $\tau=1$ for \texttt{k-grad} and \texttt{n+k-1-grad}. 

Both BLB and SDB require each worker node to repeatedly resample and re-fit the model, so we expect they require more time. Particularly, Table \ref{tab:comp} shows that BLB is much more computationally expensive than the others, and its computational time greatly increases as $k$ and $d$ grows. SDB has much lower computational time than BLB, but the computational time grows rapidly with the number of machines. On the other hand, computational time of \texttt{k-grad} and \texttt{n+k-1-grad} remains low as $k$ grows, since the bootstrap is done only on the master node. We have even observed a decrease in the run time as $k$ increases for \texttt{k-grad} and \texttt{n+k-1-grad}, which show that our methods can better take advantage of parallelism.

\section{Discussion}

We propose two communication-efficient and computation-efficient bootstrap methods, \texttt{k-grad} and \texttt{n+k-1-grad}, for simultaneous inference on distributed massive data. Our methods are robust to the number of machines. The accuracy and efficiency of the algorithms are theoretically proven and validated through simulations.

Our methods can potentially be extended to high-dimensional input variables, where the problem of simultaneous inference can be even more challenging. 

\begin{table}[H]
\caption{Average run times (in second) of \texttt{k-grad}, \texttt{n+k-1-grad}, SDB, and BLB with different $k$ and $d$ (\textbf{top}: $d=2^3$, \textbf{bottom}: $d=2^7$).}
\label{tab:comp}
\vskip 0.15in
\begin{center}
\begin{small}
\begin{tabular}{lrrr}
\toprule
Methods &       $k=2^2$ &      $k=2^6$ &      $k=2^9$ \\
\midrule
\texttt{k-grad} &  0.29 & 0.29 & 0.30 \\
\texttt{n+k-1-grad} &  0.85 & 0.45 & 0.45 \\
SDB & 0.08 & 0.30 &  5.39 \\
BLB &   22.66 &  35.12 &  159.88 \\
\bottomrule
\end{tabular}
\begin{tabular}{lrrr}
\toprule
Methods &       $k=2^2$ &      $k=2^6$ &      $k=2^9$ \\
\midrule
\texttt{k-grad} &  0.82 & 0.51 & 0.50 \\
\texttt{n+k-1-grad} &  1.49 & 0.67 & 0.64 \\
SDB & 3.44 & 3.83 &  12.66 \\
BLB &   981.17 &  842.50 &  1950.91 \\
\bottomrule
\end{tabular}
\end{small}
\end{center}
\vskip -1em
\end{table}

\bibliography{Reference}
\bibliographystyle{icml2020}


\clearpage
\onecolumn
\appendix

\section{Additional Simulation Results}

\subsection{Simultaneous Confidence Interval} \label{sec:exp3}

Figures \ref{fig:lm_tp_d35} and \ref{fig:glm_tp_d35} display the empirical coverage probability and the average width for the linear regression and logistic regression models under Toeplitz design with $d=2^3$ and $d=2^5$.  Figures \ref{fig:lm_eq} and \ref{fig:glm_eq} display the empirical coverage probability and the average width for the linear regression and logistic regression models under equi-correlation design with $d\in\{2^1,2^3,2^5,2^7\}$.  See Section \ref{sec:exp1} for the full details on the simulation setup.  The observations made in Section \ref{sec:exp1} also apply to all the cases here.  Moreover, we see that the results for equi-correlation design are similar to those for Toeplitz design.

\begin{figure}[h!]
\vskip 0.2in
\begin{center}
\centerline{\includegraphics[width=0.5\columnwidth]{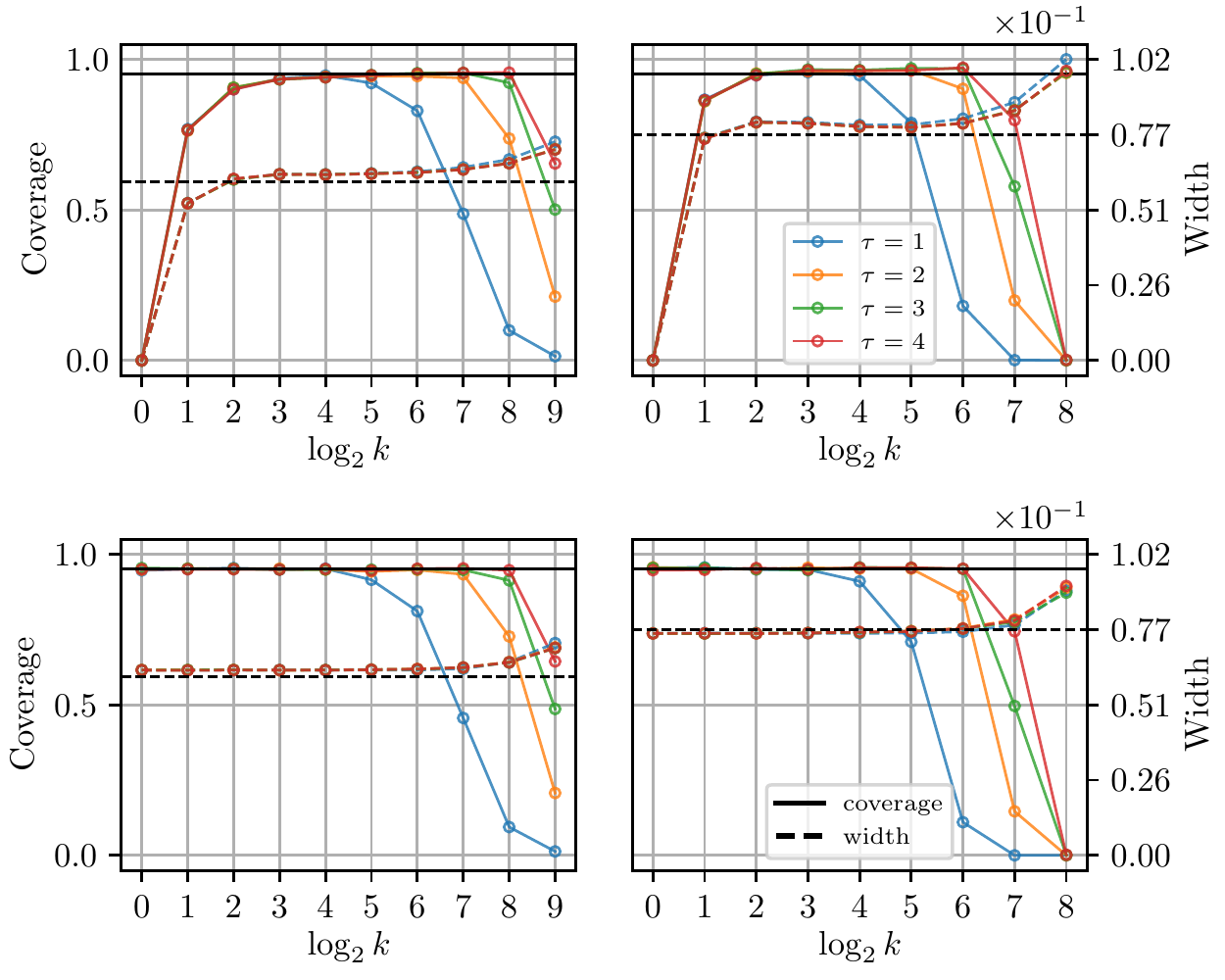}}
\caption{Empirical coverage probability (\textbf{left axis}) and average width (\textbf{right axis}) of simultaneous confidence intervals by \texttt{k-grad} (\textbf{top}) and \texttt{n+k-1-grad} (\textbf{bottom}) in a linear regression model with Toeplitz design and varying dimension (\textbf{left}: $d=2^3$, \textbf{right}: $d=2^5$).  Black solid line represents nominal confidence level ($95\%$) and black dashed line represents oracle width.}
\label{fig:lm_tp_d35}
\end{center}
\vskip -0.2in
\end{figure}

\begin{figure}[h!]
\vskip 0.2in
\begin{center}
\centerline{\includegraphics[width=0.5\columnwidth]{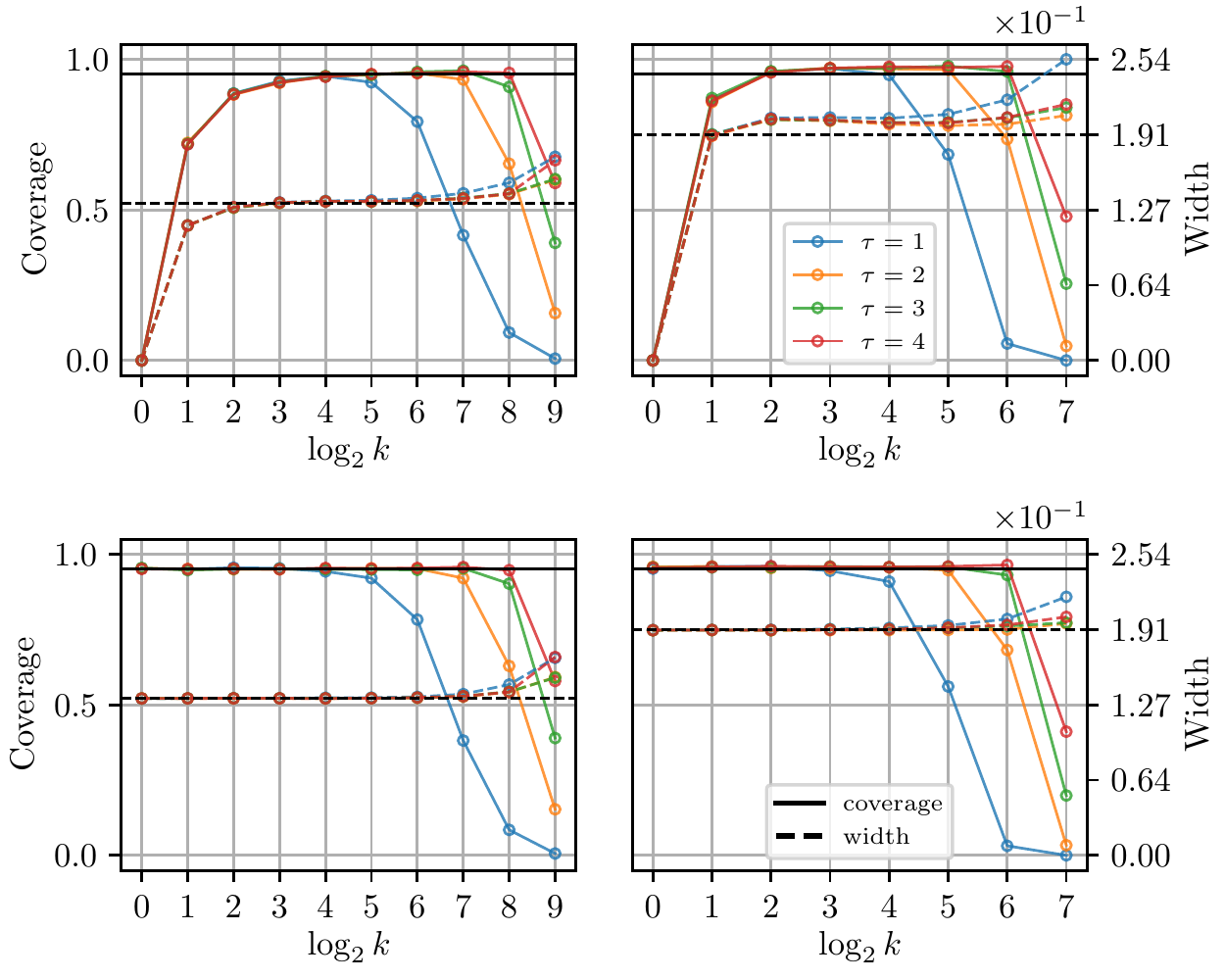}}
\caption{Empirical coverage probability (\textbf{left axis}) and average width (\textbf{right axis}) of simultaneous confidence intervals by \texttt{k-grad} (\textbf{top}) and \texttt{n+k-1-grad} (\textbf{bottom}) in a logistic regression model with Toeplitz design and varying dimension (\textbf{left}: $d=2^3$, \textbf{right}: $d=2^5$).  Black solid line represents nominal confidence level ($95\%$) and black dashed line represents oracle width.}
\label{fig:glm_tp_d35}
\end{center}
\vskip -0.2in
\end{figure}

\begin{figure}[h!]
\vskip 0.2in
\begin{center}
\centerline{\includegraphics[width=\columnwidth]{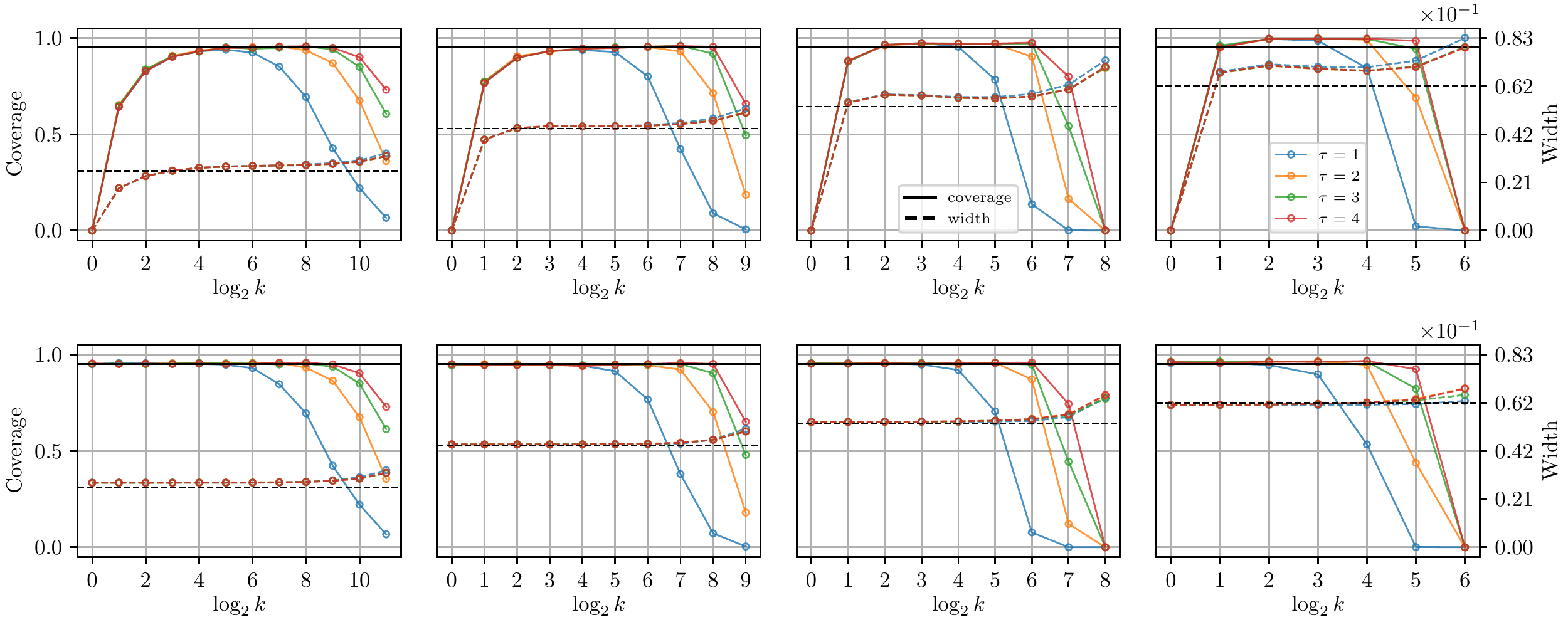}}
\caption{Empirical coverage probability (\textbf{left axis}) and average width (\textbf{right axis}) of simultaneous confidence intervals by \texttt{k-grad} (\textbf{top}) and \texttt{n+k-1-grad} (\textbf{bottom}) in a linear regression model with equi-correlation design and varying dimension (from \textbf{left} to \textbf{right}: $d=2^1,2^3,2^5,2^7$).  Black solid line represents nominal confidence level ($95\%$) and black dashed line represents oracle width.}
\label{fig:lm_eq}
\end{center}
\vskip -0.2in
\end{figure}

\begin{figure}[h!]
\vskip 0.2in
\begin{center}
\centerline{\includegraphics[width=\columnwidth]{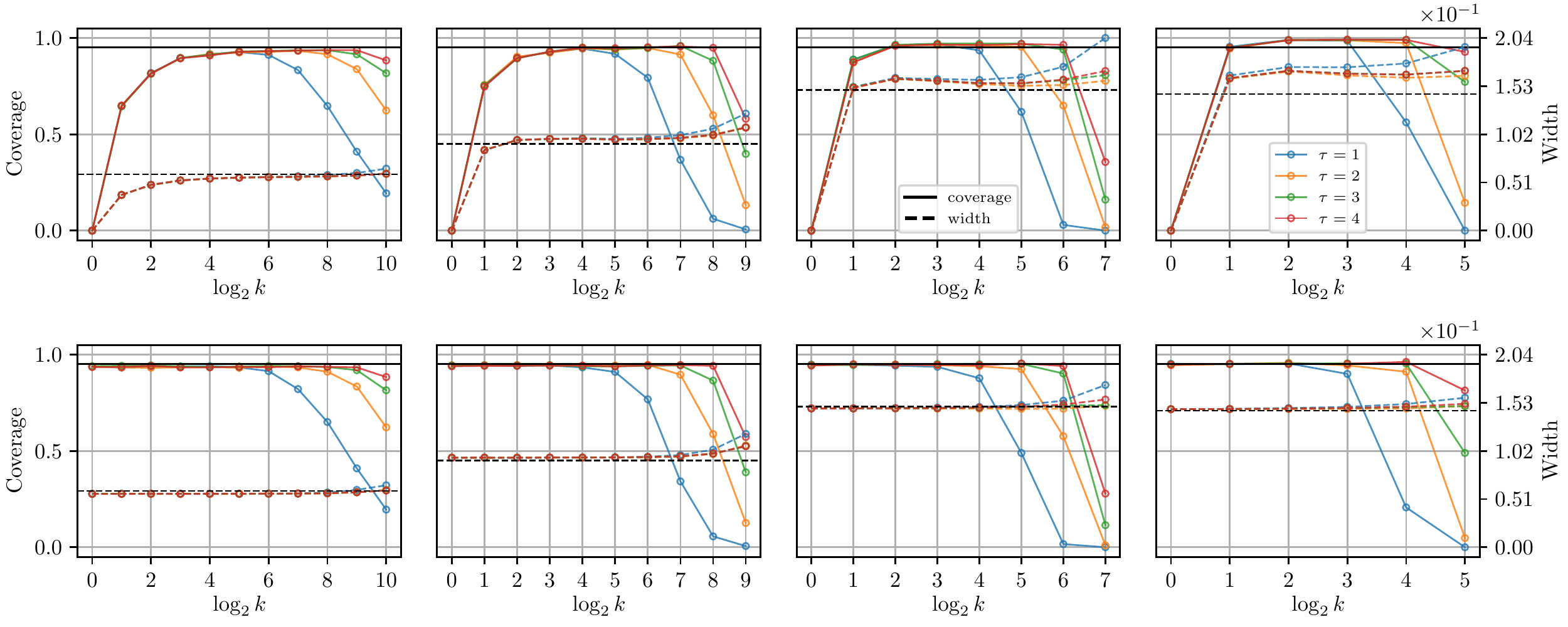}}
\caption{Empirical coverage probability (\textbf{left axis}) and average width (\textbf{right axis}) of simultaneous confidence intervals by \texttt{k-grad} (\textbf{top}) and \texttt{n+k-1-grad} (\textbf{bottom}) in a logistic regression model with equi-correlation design and varying dimension (from \textbf{left} to \textbf{right}: $d=2^1,2^3,2^5,2^7$).  Black solid line represents nominal confidence level ($95\%$) and black dashed line represents oracle width.}
\label{fig:glm_eq}
\end{center}
\vskip -0.2in
\end{figure}

\subsection{Pointwise Confidence Interval}

Figures \ref{fig:lm_tp_d35} and \ref{fig:glm_tp_d35} display the empirical coverage probability and the average width for the linear regression and logistic regression models under Toeplitz design with $d\in\{2^1,2^3,2^5,2^7\}$.  The simulation setup is the same as in Section \ref{sec:exp1}.  All the pointwise confidence intervals are constructed for the second coordinate of $\thetas$.  The algorithm is modified by replacing $\|\cdot\|_\infty$ with $|(\cdot)_2|$ as discussed in Section \ref{sec:alg}.  Comparing the results to those in Sections \ref{sec:exp1} and \ref{sec:exp3}, we see that the performance of \texttt{k-grad} and \texttt{n+k-1-grad} in constructing pointwise confidence intervals is similar to that in constructing simultaneous confidence intervals.  Therefore, the discussions on simultaneous confidence intervals in \ref{sec:exp1} can apply to the cases here.

\begin{figure}[h!]
\vskip 0.2in
\begin{center}
\centerline{\includegraphics[width=\columnwidth]{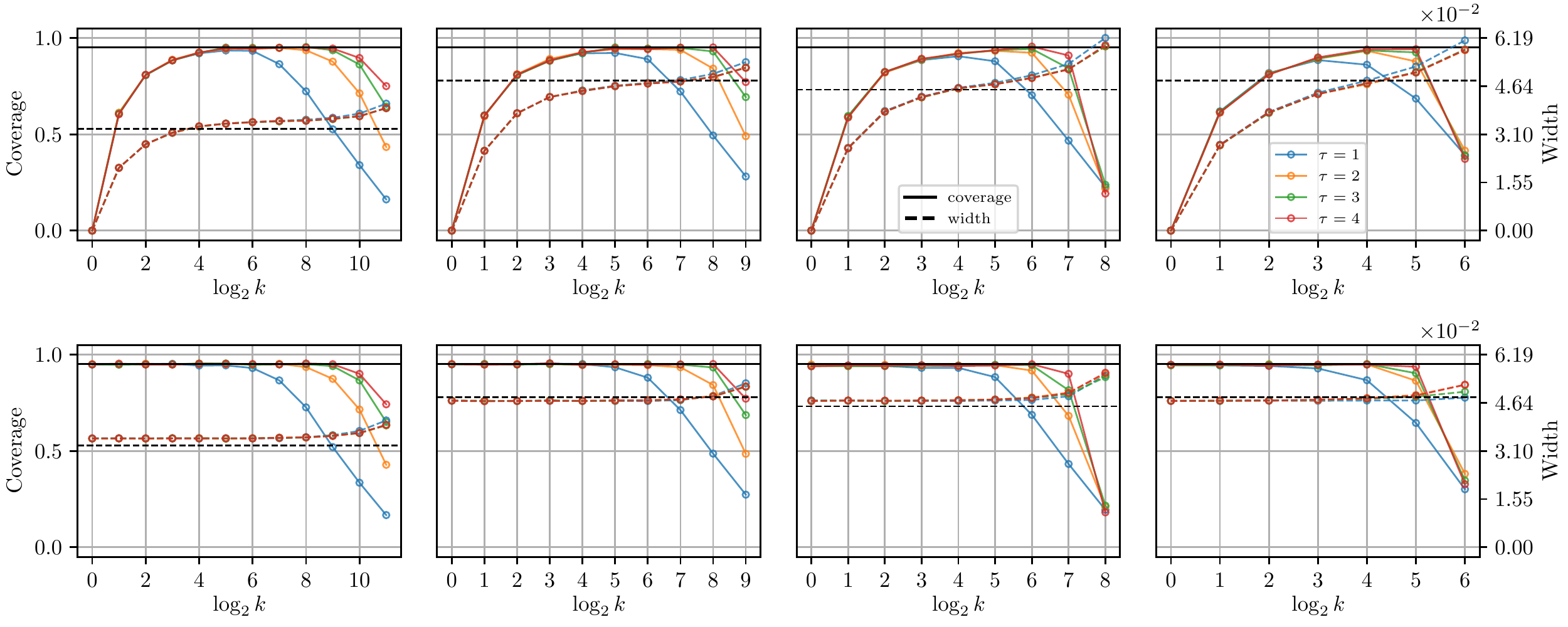}}
\caption{Empirical coverage probability (\textbf{left axis}) and average width (\textbf{right axis}) of pointwise confidence intervals by \texttt{k-grad} (\textbf{top}) and \texttt{n+k-1-grad} (\textbf{bottom}) in a linear regression model with Toeplitz design and varying dimension (from \textbf{left} to \textbf{right}: $d=2^1,2^3,2^5,2^7$).  Black solid line represents nominal confidence level ($95\%$) and black dashed line represents oracle width.}
\label{fig:lm_tp_pt}
\end{center}
\vskip -0.2in
\end{figure}

\begin{figure}[h!]
\vskip 0.2in
\begin{center}
\centerline{\includegraphics[width=\columnwidth]{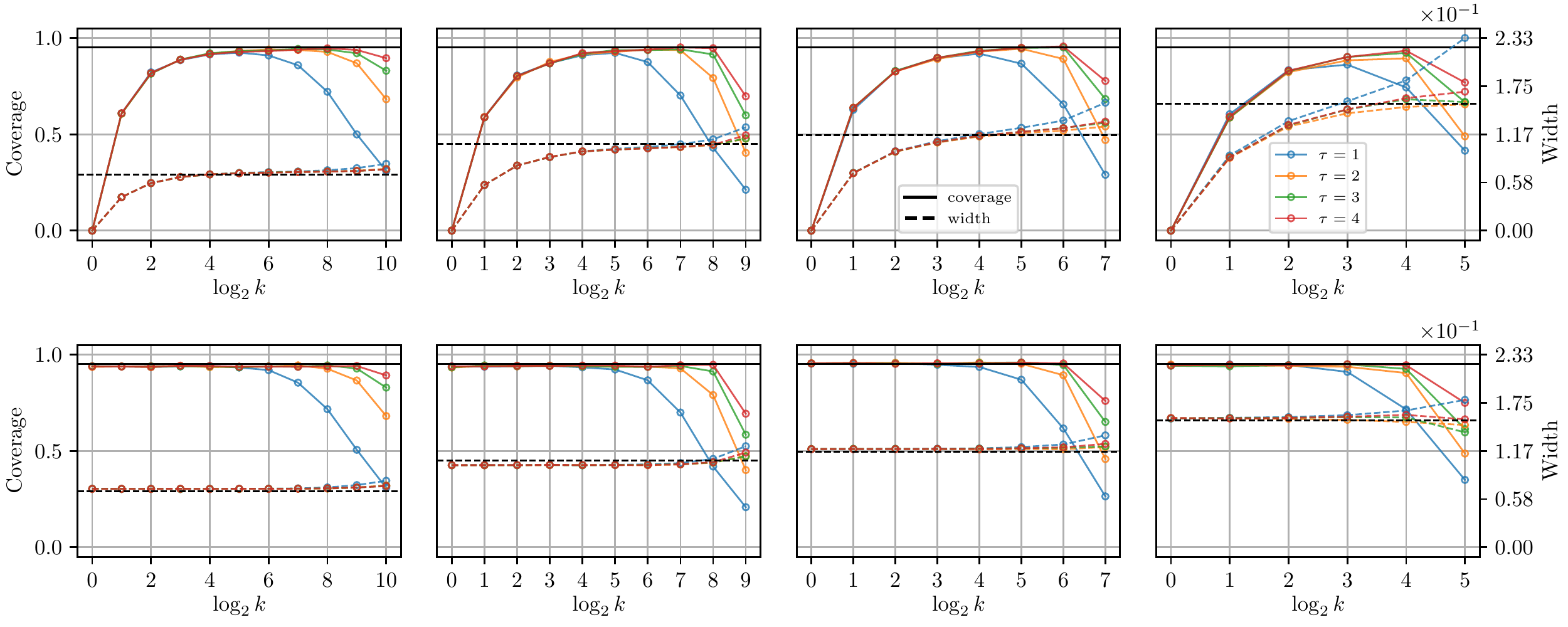}}
\caption{Empirical coverage probability (\textbf{left axis}) and average width (\textbf{right axis}) of pointwise confidence intervals by \texttt{k-grad} (\textbf{top}) and \texttt{n+k-1-grad} (\textbf{bottom}) in a logistic regression model with Toeplitz design and varying dimension (from \textbf{left} to \textbf{right}: $d=2^1,2^3,2^5,2^7$).  Black solid line represents nominal confidence level ($95\%$) and black dashed line represents oracle width.}
\label{fig:glm_tp_pt}
\end{center}
\vskip -0.2in
\end{figure}

\section{Proofs of Main Results}

\begin{proof}[Theorem \ref{theo:ld0_csl}]
	By Lemmas~\ref{lem:m}~and~\ref{lem:csl}, we obtain that
	$$\left\|\ttheta-\htheta\right\|_\infty = \left\|\ttheta^{(\tau)}-\htheta\right\|_\infty \leq \left\|\ttheta^{(\tau)}-\htheta\right\|_2 = O_P\left(\left(\sqrt{\frac dn}\right)^{\tau+1} \sqrt{\log d}\right),\quad\text{and}$$
	$$\left\|\btheta-\thetas\right\|_1 = \left\|\ttheta^{(\tau-1)}-\thetas\right\|_1 \leq \sqrt d\left\|\ttheta^{(\tau-1)}-\htheta\right\|_2 + \sqrt d\left\|\htheta-\thetas\right\|_2 = O_P\left(d\sqrt{\frac{\log d}N}+\left(\sqrt{\frac dn}\right)^{\tau} \sqrt{d\log d}\right),$$
	if $N\gtrsim d\log d$ and $n\gtrsim d$.  Then, by Lemma~\ref{lem:ld0}, we have $\sup_{\alpha\in(0,1)}\left|P(T\leq c_{\overline W}(\alpha))-\alpha\right|=o(1)$ and $\sup_{\alpha\in(0,1)}\left|P(\widehat T\leq c_{\overline W}(\alpha))-\alpha\right|=o(1)$, as long as $n\gg d\log^{4+\kappa} d$, $k\gg d^2\log^{5+\kappa} d$, and
	$$\left(\sqrt{\frac dn}\right)^{\tau+1} \sqrt{\log d} \ll\frac1{\sqrt N\log^{1/2+\kappa}d},\quad\text{and}$$
	$$d\sqrt{\frac{\log d}N}+\left(\sqrt{\frac dn}\right)^{\tau} \sqrt{d\log d}\ll\min\left\{\frac1{d\sqrt{\log k}\log^{2+\kappa}d},\frac1{ \sqrt{nd}\log^{1+\kappa}d}\right\}.$$
	We complete the proof by solving these inequalities for $\tau$.
\end{proof}

\begin{proof}[Theorem \ref{theo:ld_csl}]
	By the argument in the proof of Theorem~\ref{theo:ld0_csl} with applying Lemma~\ref{lem:ld}, we have $\sup_{\alpha\in(0,1)}\left|P(T\leq c_{\widetilde W}(\alpha))-\alpha\right|=o(1)$ and $\sup_{\alpha\in(0,1)}\left|P(\widehat T\leq c_{\widetilde W}(\alpha))-\alpha\right|=o(1)$, as long as $n\gg d\log^{4+\kappa} d$, $n+k\gg d^2\log^{5+\kappa} d$, and
	$$\left(\sqrt{\frac dn}\right)^{\tau+1} \sqrt{\log d}\ll\frac1{\sqrt N\log^{1/2+\kappa}d},\quad\text{and}$$
	$$d\sqrt{\frac{\log d}N}+\left(\sqrt{\frac dn}\right)^{\tau} \sqrt{d\log d}\ll\min\left\{\frac1{d\sqrt{\log((n+k)d)}\log^{2+\kappa}d},\frac1{ \sqrt{d}\log^{1+\kappa}d}\sqrt{\frac1n+\frac1k}\right\}.$$
	We complete the proof by solving these inequalities for $\tau$.
\end{proof}

\begin{proof}[Theorem \ref{theo:ld0_glm_csl}]
	By Lemmas~\ref{lem:m_glm}~and~\ref{lem:csl_glm}, we obtain that
	\begin{align}
	\begin{split}
	\left\|\ttheta-\htheta\right\|_\infty = \left\|\ttheta^{(\tau)}-\htheta\right\|_\infty \leq \left\|\ttheta^{(\tau)}-\htheta\right\|_2 = \begin{cases}
    O_P\left(\frac1{d^{3/2}}\left(d^2\sqrt{\frac{\log d}n}\right)^{2^\tau}\right), & \tau\leq\tau_0,\\
    O_P\left(\frac1{d^{3/2}}\left(d^2\sqrt{\frac{\log d}n}\right)^{2^{\tau_0}}\left(\sqrt{\frac{d\log d}n}\right)^{\tau-\tau_0}\right), & \tau>\tau_0,
\end{cases}
    \end{split} \quad\text{and}\label{eqn:csl_1}
    \end{align}
	\begin{align}
	\begin{split}
	\left\|\btheta-\thetas\right\|_1 &= \left\|\ttheta^{(\tau-1)}-\thetas\right\|_1 \leq \sqrt d\left\|\ttheta^{(\tau-1)}-\thetas\right\|_2 \leq \sqrt d\left\|\ttheta^{(\tau-1)}-\htheta\right\|_2 + \sqrt d\left\|\htheta-\thetas\right\|_2 \\
	&= \begin{cases}
    O_P\left(d\sqrt{\frac{\log d}N}+\frac1d\left(d^2\sqrt{\frac{\log d}n}\right)^{2^{\tau-1}}\right), & \tau\leq\tau_0+1,\\
    O_P\left(d\sqrt{\frac{\log d}N}+\frac1d\left(d^2\sqrt{\frac{\log d}n}\right)^{2^{\tau_0}}\left(\sqrt{\frac{d\log d}n}\right)^{\tau-\tau_0-1}\right), & \tau>\tau_0+1,
\end{cases}
	\end{split} \label{eqn:csl_2}
    \end{align}
	if $n\gtrsim d^4\log d$, where $\tau_0$ is the smallest integer $t$ such that
	$$\left(d^2\sqrt{\frac{\log d}n}\right)^{2^t}\lesssim \sqrt{\frac{d\log d}n},$$
	that is,
	$$\tau_0=\left\lceil\log_2\left(\frac{\log n-\log d-\log\log d}{\log n-\log(d^4)-\log\log d}\right)\right\rceil.$$
	Then, by Lemma~\ref{lem:ld0_glm}, we have $\sup_{\alpha\in(0,1)}\left|P(T\leq c_{\overline W}(\alpha))-\alpha\right|=o(1)$ and $\sup_{\alpha\in(0,1)}\left|P(\widehat T\leq c_{\overline W}(\alpha))-\alpha\right|=o(1)$, as long as $n\gg d^4\log d$, $k\gg d^2\log^{5+\kappa} d$, $nk\gg d^5\log^{3+\kappa} d$,
	$$\text{RHS of }\eqref{eqn:csl_1}\ll\frac1{\sqrt N\log^{1/2+\kappa}d},\quad\text{and}\quad\text{RHS of }\eqref{eqn:csl_2}\ll\frac1{ \sqrt{nd}\log^{1+\kappa}d}.$$
	We complete the proof by solving these inequalities for $\tau$.

\end{proof}

\begin{proof}[Theorem \ref{theo:ld_glm_csl}]
	By the argument in the proof of Theorem~\ref{theo:ld0_csl} with applying Lemma~\ref{lem:ld_glm}, we have $\sup_{\alpha\in(0,1)}\left|P(T\leq c_{\widetilde W}(\alpha))-\alpha\right|=o(1)$ and $\sup_{\alpha\in(0,1)}\left|P(\widehat T\leq c_{\widetilde W}(\alpha))-\alpha\right|=o(1)$, as long as $n\gg d^4\log d$, $n+k\gg d^2\log^{5+\kappa} d$, $nk\gg d^5\log^{3+\kappa} d$,
	$$\text{RHS of }\eqref{eqn:csl_1}\ll\frac1{\sqrt N\log^{1/2+\kappa}d},\quad\text{and}\quad\text{RHS of }\eqref{eqn:csl_2}\ll\min\left\{\frac1{d\log^{11/4+\kappa}d},\frac1{ \sqrt{d}\log^{1+\kappa}d}\sqrt{\frac1n+\frac1k}\right\}.$$
	We complete the proof by solving these inequalities for $\tau$.

\end{proof}

\section{Lemmas on Bounding Bootstrap Errors}

\begin{lemma}[\texttt{k-grad}]\label{lem:ld0}
	In linear model, under Assumptions~\ref{as:design}~and~\ref{as:noise}, if $n\gg d\log^{4+\kappa} d$, $k\gg d^2\log^{5+\kappa} d$,
	$$\left\|\ttheta-\htheta\right\|_\infty\ll\frac1{\sqrt N\log^{1/2+\kappa}d},\quad\text{and}\quad \left\|\btheta-\thetas\right\|_1\ll\min\left\{\frac1{d\sqrt{\log k}\log^{2+\kappa}d},\frac1{ \sqrt{nd}\log^{1+\kappa}d}\right\},$$
	for some $\kappa>0$, then we have that
	\begin{align}
	\sup_{\alpha\in(0,1)}\left|P(T\leq c_{\overline W}(\alpha))-\alpha\right|&=o(1), \quad\text{and} \label{eqn:lem_wb_t} \\
	\sup_{\alpha\in(0,1)}\left|P(\widehat T\leq c_{\overline W}(\alpha))-\alpha\right|&=o(1). \label{eqn:lem_wb_ht}
	\end{align}
\end{lemma}

\begin{proof}[Lemma \ref{lem:ld0}]
    As noted by \cite{zhang2017simultaneous}, since $\|\sqrt N(\ttheta-\thetas)\|_\infty=\max_l\sqrt N|\ttheta_l-\thetas_l|=\sqrt N\max_l\big((\ttheta_l-\thetas_l)\vee(\thetas_l-\ttheta_l)\big)$, the arguments for the bootstrap consistency result with
    \begin{align}
        T&=\max_l\sqrt N(\ttheta-\thetas)_l\quad\text{and} \label{eqn:tp} \\
         \widehat T&=\max_l\sqrt N(\htheta-\thetas)_l \label{eqn:thp}
    \end{align}
    imply the bootstrap consistency result for $T=\|\sqrt N(\ttheta-\thetas)\|_\infty$ and $\widehat T=\|\sqrt N(\htheta-\thetas)\|_\infty$. Hence, from now on, we redefine $T$ and $\widehat T$ as \eqref{eqn:tp} and \eqref{eqn:thp}. Define an oracle multiplier bootstrap statistic as
	\begin{align}
	W^*\defn\max_{1\leq l\leq d}-\frac1{\sqrt{N}}\sum_{i=1}^n\sum_{j=1}^k \left(\nabla^2\cLs(\thetas)^{-1}\nabla\cL(\thetas;Z_{ij})\right)_l\epsilon_{ij}^*, \label{eqn:wsdef}
	\end{align}
	where $\{\epsilon_{ij}^*\}_{i=1,\dots,n;j=1,\dots,k}$ are $N$ independent standard Gaussian variables, also independent of the entire data set.  The proof consists of two steps; the first step is to show that $W^*$ achieves bootstrap consistency, i.e., $\sup_{\alpha\in(0,1)}|P(T\leq c_{W^*}(\alpha))-\alpha|$ converges to $0$, where $c_{W^*}(\alpha)=\inf\{t\in\R:P_\epsilon(W^*\leq t)\geq\alpha\},$ and the second step is to show the bootstrap consistency of our proposed bootstrap statistic by showing the quantiles of $W$ and $W^*$ are close.
	
	Note that $\nabla^2\cLs(\thetas)^{-1}\nabla\cL(\thetas;Z)=\Ee[xx^\top]^{-1}x(x^\top\thetas-y)=\Theta xe$ and
	$$\Ee\left[\left(\nabla^2\cLs(\thetas)^{-1}\nabla\cL(\thetas;Z)\right)\left(\nabla^2\cLs(\thetas)^{-1}\nabla\cL(\thetas;Z)\right)^\top\right]=\Theta\Ee\left[xx^\top e^2\right]\Theta=\sigma^2\Theta\Sigma\Theta=\sigma^2\Theta.$$
	Then, under Assumptions~\ref{as:design}~and~\ref{as:noise},
	\begin{align}
	\min_l\Ee\left[\left(\nabla^2\cLs(\thetas)^{-1}\nabla\cL(\thetas;Z)\right)_l^2\right]=\sigma^2\min_l\Theta_{l,l}\geq\sigma^2\lambdamin(\Theta)=\frac{\sigma^2}{\lambdamax(\Sigma)}, \label{eqn:chkvar}
	\end{align}
	is bounded away from zero.  Under Assumption \ref{as:design}, $x$ is sub-Gaussian, that is, $w^\top x$ is sub-Gaussian with uniformly bounded $\psi_2$-norm for all $w\in S^{d-1}$.  To show $w^\top\Theta x$ is also sub-Gaussian with uniformly bounded $\psi_2$-norm, we write it as
	$$w^\top\Theta x = (\Theta w)^\top x = \left\|\Theta w\right\|_2 \left(\frac{\Theta w}{\left\|\Theta w\right\|_2}\right)^\top x.$$
	Since $\Theta w/\left\|\Theta w\right\|_2\in S^{d-1}$, we have that $\left(\Theta w/\left\|\Theta w\right\|_2\right) x$ is sub-Gaussian with $O(1)$ $\psi_2$-norm, and hence, $w^\top\Theta x$ is sub-Gaussian with $O(\left\|\Theta w\right\|_2)=O(\lambdamax(\Theta))=O(\lambdamin(\Sigma)^{-1})=O(1)$ $\psi_2$-norm, under Assumption~\ref{as:design}.  Since $e$ is also sub-Gaussian under Assumption~\ref{as:noise} and is independent of $w^\top\Theta x$, we have that $w^\top\Theta xe$ is sub-exponential with uniformly bounded $\psi_1$-norm for all $w\in S^{d-1}$, and also, all $\left(\nabla^2\cLs(\thetas)^{-1}\nabla\cL(\thetas;Z)\right)_l$ are sub-exponential with uniformly bounded $\psi_1$-norm.  Combining this with \eqref{eqn:chkvar}, we have verified Assumption~(E.1) of \cite{chernozhukov2013gaussian} for $\nabla^2\cLs(\thetas)^{-1}\nabla\cL(\thetas;Z)$.
	
	Define
	\begin{align}
	T_0\defn\max_{1\leq l\leq d}-\sqrt{N}\left(\nabla^2\cLs(\thetas)^{-1}\nabla\cL_N(\thetas)\right)_l, \label{eqn:t0}
	\end{align}
	which is a Bahadur representation of $T$.  Under the condition $\log^7(dN)/N\lesssim N^{-c}$ for some constant $c>0$, which holds if $N\gtrsim\log^{7+\kappa}d$ for some $\kappa>0$, applying Theorem 3.2 and Corollary 2.1 of \cite{chernozhukov2013gaussian}, we obtain that for some constant $c>0$ and for every $v,\zeta>0$, 
	\begin{align}
	\begin{split}
	\sup_{\alpha\in(0,1)}\left|P(T\leq c_{W^*}(\alpha))-\alpha\right| &\lesssim N^{-c}+v^{1/3}\left(1\vee\log\frac dv\right)^{2/3}+P\left(\matrixnorm{\widehat\Omega-\Omega_0}_{\max}>v\right) \\
	&\quad+\zeta\sqrt{1\vee\log\frac d\zeta}+P\left(|T-T_0|>\zeta\right), \label{eqn:ws}
	\end{split}
	\end{align}
	where
	\begin{align}
	\begin{split}
	\widehat\Omega&\defn\cov_\epsilon\left(-\frac1{\sqrt{N}}\sum_{i=1}^n\sum_{j=1}^k\nabla^2\cLs(\thetas)^{-1}\nabla\cL(\thetas;Z_{ij})\epsilon_{ij}^* \right) \\
	&=\nabla^2\cLs(\thetas)^{-1}\left(\frac1N\sum_{i=1}^n\sum_{j=1}^k\nabla\cL(\thetas;Z_{ij})\nabla\cL(\thetas;Z_{ij})^\top\right)\nabla^2\cLs(\thetas)^{-1}, \quad\text{and} \label{eqn:oh}
	\end{split}
	\end{align}
	\begin{align}
	\Omega_0&\defn\cov\left(-\nabla^2\cLs(\thetas)^{-1}\nabla\cL(\thetas;Z)\right)=\nabla^2\cLs(\thetas)^{-1}\Ee\left[\nabla\cL(\thetas;Z)\nabla\cL(\thetas;Z)^\top\right]\nabla^2\cLs(\thetas)^{-1}. \label{eqn:o0}
	\end{align}
	To show the quantiles of $\overline W$ and $W^*$ are close, we first have that for any $\omega$ such that $\alpha+\omega,\alpha-\omega\in(0,1)$,
	\begin{align*}
	&P(\{T\leq c_{\overline W}(\alpha)\}\ominus\{T\leq c_{W^*}(\alpha)\}) \\
	&\leq 2P(c_{W^*}(\alpha-\omega)<T\leq c_{W^*}(\alpha+\omega)) + P(c_{W^*}(\alpha-\omega)>c_{\overline W}(\alpha)) + P(c_{\overline W}(\alpha)>c_{W^*}(\alpha+\omega)),
	\end{align*}
	where $\ominus$ denotes symmetric difference.
	Following the arguments in the proof of Lemma 3.2 of \cite{chernozhukov2013gaussian}, we have that
	$$P(c_{\overline W}(\alpha)>c_{W^*}(\alpha+\pi(u)))\leq P\left(\matrixnorm{\overline\Omega-\widehat\Omega}_{\max}>u\right),\quad\text{and}$$
	$$P(c_{W^*}(\alpha-\pi(u))>c_{\overline W}(\alpha))\leq P\left(\matrixnorm{\overline\Omega-\widehat\Omega}_{\max}>u\right),$$
	where $\pi(u)\defn u^{1/3}\left(1\vee\log(d/u)\right)^{2/3}$ and
	\begin{align}
	\begin{split}
	\overline\Omega&\defn\cov_\epsilon\left(-\frac1{\sqrt{k}}\sum_{j=1}^k \widetilde\Theta \sqrt n\left(\nabla\cL_j(\btheta)-\nabla\cL_N(\btheta)\right)\epsilon_j \right) \\
	&=\widetilde\Theta\left(\frac1k\sum_{j=1}^k n\left(\nabla\cL_j(\btheta)-\nabla\cL_N(\btheta)\right)\left(\nabla\cL_j(\btheta)-\nabla\cL_N(\btheta)\right)^\top\right)\widetilde\Theta^\top. \label{eqn:ob}
	\end{split}
	\end{align}
	By letting $\omega=\pi(u)$, we have that
	\begin{align*}
	&P(\{T\leq c_{\overline W}(\alpha)\}\ominus\{T\leq c_{W^*}(\alpha)\}) \\
	&\leq 2P(c_{W^*}(\alpha-\pi(u))<T\leq c_{W^*}(\alpha+\pi(u))) + P(c_{W^*}(\alpha-\pi(u))>c_{\overline W}(\alpha)) + P(c_{\overline W}(\alpha)>c_{W^*}(\alpha+\pi(u))) \\
	&\leq 2P(c_{W^*}(\alpha-\pi(u))<T\leq c_{W^*}(\alpha+\pi(u))) + 2P\left(\matrixnorm{\overline\Omega-\widehat\Omega}_{\max}>u\right),
	\end{align*}
	where by \eqref{eqn:ws},
	\begin{align*}
	P(c_{W^*}(\alpha-\pi(u))<T\leq c_{W^*}(\alpha+\pi(u))) &=P(T\leq c_{W^*}(\alpha+\pi(u)))-P(T\leq c_{W^*}(\alpha-\pi(u))) \\
	&\lesssim \pi(u)+N^{-c}+\zeta\sqrt{1\vee\log\frac d\zeta}+P\left(|T-T_0|>\zeta\right),
	\end{align*}
	and then,
	\begin{align}
	\sup_{\alpha\in(0,1)}\left|P(T\leq c_{\overline W}(\alpha))-\alpha\right| &\lesssim N^{-c}+v^{1/3}\left(1\vee\log\frac dv\right)^{2/3}+P\left(\matrixnorm{\widehat\Omega-\Omega_0}_{\max}>v\right) \nonumber \\
	&\quad+\zeta\sqrt{1\vee\log\frac d\zeta}+P\left(|T-T_0|>\zeta\right) + u^{1/3}\left(1\vee\log\frac du\right)^{2/3} + P\left(\matrixnorm{\overline\Omega-\widehat\Omega}_{\max}>u\right). \label{eqn:wbbd}
	\end{align}
	
	Applying Lemmas~\ref{lem:tbd_ld},~\ref{lem:gbdo_ld},~and~\ref{lem:gbd0_ld}, we have that there exist some $\zeta,u,v>0$ such that
	\begin{align}
	\zeta\sqrt{1\vee\log\frac d\zeta}&+P\left(|T-T_0|>\zeta\right)=o(1),\quad\text{and} \label{eqn:zeta} \\
	u^{1/3}\left(1\vee\log\frac du\right)^{2/3}&+P\left(\matrixnorm{\overline\Omega-\widehat\Omega}_{\max}>u\right)=o(1),\quad\text{and} \label{eqn:u} \\
	v^{1/3}\left(1\vee\log\frac dv\right)^{2/3}&+P\left(\matrixnorm{\widehat\Omega-\Omega_0}_{\max}>v\right)=o(1), \label{eqn:v}
	\end{align}
	and hence, after simplifying the conditions, obtain the first result in the lemma. To obtain the second result, we use Lemma~\ref{lem:thbd_ld}, which yields
	\begin{align}
	\xi\sqrt{1\vee\log\frac d\xi}+P\left(|\widehat T-T_0|>\xi\right)=o(1). \label{eqn:xi}
	\end{align}
\end{proof}

\begin{lemma}[\texttt{n+k-1-grad}]\label{lem:ld}
	In linear model, under Assumptions~\ref{as:design}~and~\ref{as:noise}, if $n\gg d\log^{4+\kappa} d$, $n+k\gg d^2\log^{5+\kappa} d$,
	$$\left\|\ttheta-\htheta\right\|_\infty\ll\frac1{\sqrt N\log^{1/2+\kappa}d},\quad\text{and}\quad\left\|\btheta-\thetas\right\|_1\ll\min\left\{\frac1{d\sqrt{\log((n+k)d)}\log^{2+\kappa}d},\frac1{ \sqrt{d}\log^{1+\kappa}d}\sqrt{\frac1n+\frac1k}\right\},$$
	for some $\kappa>0$, then we have that
	\begin{align}
	\sup_{\alpha\in(0,1)}\left|P(T\leq c_{\widetilde W}(\alpha))-\alpha\right|&=o(1),\quad\text{and} \label{eqn:lem_wt_t} \\
	\sup_{\alpha\in(0,1)}\left|P(\widehat T\leq c_{\widetilde W}(\alpha))-\alpha\right|&=o(1). \label{eqn:lem_wt_ht}
	\end{align}
\end{lemma}

\begin{proof}[Lemma \ref{lem:ld}]
	By the argument in the proof of Lemma~\ref{lem:ld0}, we have that
	\begin{align}
	\sup_{\alpha\in(0,1)}\left|P(T\leq c_{\widetilde W}(\alpha))-\alpha\right| &\lesssim N^{-c}+v^{1/3}\left(1\vee\log\frac dv\right)^{2/3}+P\left(\matrixnorm{\widehat\Omega-\Omega_0}_{\max}>v\right) \nonumber \\
	&\quad+\zeta\sqrt{1\vee\log\frac d\zeta}+P\left(|T-T_0|>\zeta\right) + u^{1/3}\left(1\vee\log\frac du\right)^{2/3} + P\left(\matrixnorm{\widetilde\Omega-\widehat\Omega}_{\max}>u\right), \label{eqn:wbbd2}
	\end{align}
	where
	\begin{align}
	\begin{split}
	\widetilde\Omega&\defn\cov_\epsilon\left(-\frac1{\sqrt{n+k-1}}\left(\sum_{i=1}^n\widetilde\Theta\left(\nabla\cL(\btheta;Z_{i1})-\nabla\cL_N(\btheta)\right)\epsilon_{i1}+\sum_{j=2}^k\widetilde\Theta\sqrt n\left(\nabla\cL_j(\btheta)-\nabla\cL_N(\btheta)\right)\epsilon_j\right)\right) \\
	&=\widetilde\Theta\frac1{n+k-1}\Bigg(\sum_{i=1}^n\left(\nabla\cL(\theta;Z_{i1})-\nabla\cL_N(\theta)\right)\left(\nabla\cL(\theta;Z_{i1})-\nabla\cL_N(\theta)\right)^\top \\
	&\quad+\sum_{j=2}^k n\left(\nabla\cL_j(\theta)-\nabla\cL_N(\theta)\right)\left(\nabla\cL_j(\theta)-\nabla\cL_N(\theta)\right)^\top\Bigg)\widetilde\Theta^\top, \label{eqn:ot}
	\end{split}
	\end{align}
	if $N\gtrsim\log^{7+\kappa}d$ for some $\kappa>0$.  Applying Lemmas~\ref{lem:tbd_ld},~\ref{lem:gbdo_ld},~and~\ref{lem:gbd_ld}, we have that there exist some $\zeta,u,v>0$ such that \eqref{eqn:zeta},
	\begin{align}
	u^{1/3}\left(1\vee\log\frac du\right)^{2/3}+P\left(\matrixnorm{\widetilde\Omega-\widehat\Omega}_{\max}>u\right)=o(1), \label{eqn:ut}
	\end{align}
	and \eqref{eqn:v} hold, and hence, after simplifying the conditions, obtain the first result in the lemma. To obtain the second result, we use Lemma~\ref{lem:thbd_ld}, which yields \eqref{eqn:xi}.
\end{proof}

\begin{lemma}[\texttt{k-grad}]\label{lem:ld0_glm}
	In GLM, under Assumptions~\ref{as:smth_glm}--\ref{as:subexp_glm}, if $n\gg d\log^{5+\kappa} d$, $k\gg d^2\log^{5+\kappa} d$, $nk\gg d^5\log^{3+\kappa} d$,
	$$\left\|\ttheta-\htheta\right\|_\infty\ll\frac1{\sqrt N\log^{1/2+\kappa}d},\quad\text{and}\quad
	\left\|\btheta-\thetas\right\|_1\ll\frac1{ \sqrt{nd}\log^{1+\kappa}d},$$
	for some $\kappa>0$, then we have that \eqref{eqn:lem_wb_t} and \eqref{eqn:lem_wb_ht} hold.
\end{lemma}

\begin{proof}[Lemma \ref{lem:ld0_glm}]
    We redefine $T$ and $\widehat T$ as \eqref{eqn:tp} and \eqref{eqn:thp}.  We define an oracle multiplier bootstrap statistic as in \eqref{eqn:wsdef}.  Under Assumption~\ref{as:hes_glm},
	\begin{align*}
	\min_l\Ee\left[\left(\nabla^2\cLs(\thetas)^{-1}\nabla\cL(\thetas;Z)\right)_l^2\right]&=\min_l \left(\nabla^2\cLs(\thetas)^{-1}\Ee\left[\nabla\cL(\thetas;Z)\nabla\cL(\thetas;Z)^\top\right]\nabla^2\cLs(\thetas)^{-1}\right)_{l,l} \\
	&\geq\lambdamin\left(\nabla^2\cLs(\thetas)^{-1}\Ee\left[\nabla\cL(\thetas;Z)\nabla\cL(\thetas;Z)^\top\right]\nabla^2\cLs(\thetas)^{-1}\right) \\
	&\geq\lambdamin\left(\nabla^2\cLs(\thetas)^{-1}\right)^2\lambdamin\left(\Ee\left[\nabla\cL(\thetas;Z)\nabla\cL(\thetas;Z)^\top\right]\right) \\
	&=\frac{\lambdamin\left(\Ee\left[\nabla\cL(\thetas;Z)\nabla\cL(\thetas;Z)^\top\right]\right)}{\lambdamax\left(\nabla^2\cLs(\thetas)\right)^2}
	\end{align*}
	is bounded away from zero.  Combining this with Assumption~\ref{as:subexp_glm}, we have verified Assumption~(E.1) of \cite{chernozhukov2013gaussian} for $\nabla^2\cLs(\thetas)^{-1}\nabla\cL(\thetas;Z)$.  Then, we use the same argument as in the proof of Lemma~\ref{lem:ld0}, and obtain \eqref{eqn:wbbd} with
	\begin{align}
	\begin{split}
	\overline\Omega&\defn\widetilde\Theta(\btheta)\left(\frac1k\sum_{j=1}^k n\left(\nabla\cL_j(\btheta)-\nabla\cL_N(\btheta)\right)\left(\nabla\cL_j(\btheta)-\nabla\cL_N(\btheta)\right)^\top\right)\widetilde\Theta(\btheta)^\top, \label{eqn:ob_glm}
	\end{split}
	\end{align}
	under the condition $\log^7(dN)/N\lesssim N^{-c}$ for some constant $c>0$, which holds if $N\gtrsim\log^{7+\kappa}d$ for some $\kappa>0$.  Applying Lemmas~\ref{lem:tbd_ld_glm},~\ref{lem:gbdo_ld_glm},~and~\ref{lem:gbd0_ld_glm}, we have that there exist some $\zeta,u,v>0$ such that \eqref{eqn:zeta}, \eqref{eqn:u}, and \eqref{eqn:v} hold, and hence, after simplifying the conditions, obtain the first result in the lemma. To obtain the second result, we use Lemma~\ref{lem:thbd_ld_glm}, which yields \eqref{eqn:xi}.
\end{proof}

\begin{lemma}[\texttt{n+k-1-grad}]\label{lem:ld_glm}
	In GLM, under Assumptions~\ref{as:smth_glm}--\ref{as:subexp_glm}, if $n\gg d\log^{5+\kappa} d$, $n+k\gg d^2\log^{5+\kappa} d$, $nk\gg d^5\log^{3+\kappa} d$,
	$$\left\|\ttheta-\htheta\right\|_\infty\ll\frac1{\sqrt N\log^{1/2+\kappa}d},\quad\text{and}$$
	$$\left\|\btheta-\thetas\right\|_1\ll\min\left\{\frac{n+k}{d\left(n+k\sqrt{\log d}+k^{3/4}\log^{3/4}d\right)\log^{2+\kappa}d},\frac1{ \sqrt{d}\log^{1+\kappa}d}\sqrt{\frac1n+\frac1k}\right\},$$
	for some $\kappa>0$, then we have that \eqref{eqn:lem_wt_t} and \eqref{eqn:lem_wt_ht} hold.
\end{lemma}

\begin{proof}[Lemma \ref{lem:ld_glm}]
	By the argument in the proof of Lemma~\ref{lem:ld0_glm}, we obtain \eqref{eqn:wbbd2} with
	\begin{align}
	\begin{split}
	\widetilde\Omega&\defn\widetilde\Theta(\btheta)\frac1{n+k-1}\Bigg(\sum_{i=1}^n\left(\nabla\cL(\theta;Z_{i1})-\nabla\cL_N(\theta)\right)\left(\nabla\cL(\theta;Z_{i1})-\nabla\cL_N(\theta)\right)^\top \\
	&\quad+\sum_{j=2}^k n\left(\nabla\cL_j(\theta)-\nabla\cL_N(\theta)\right)\left(\nabla\cL_j(\theta)-\nabla\cL_N(\theta)\right)^\top\Bigg)\widetilde\Theta(\btheta)^\top, \label{eqn:ot_glm}
	\end{split}
	\end{align}
	if $N\gtrsim\log^{7+\kappa}d$ for some $\kappa>0$.  Applying Lemmas~\ref{lem:tbd_ld_glm},~\ref{lem:gbdo_ld_glm},~and~\ref{lem:gbd_ld_glm}, we have that there exist some $\zeta,u,v>0$ such that \eqref{eqn:zeta}, \eqref{eqn:ut}, and \eqref{eqn:v} hold, and hence, after simplifying the conditions, obtain the first result in the lemma. To obtain the second result, we use Lemma~\ref{lem:thbd_ld_glm}, which yields \eqref{eqn:xi}.
\end{proof}

\section{Lemmas on Bounding Bahadur Representation Errors}

For both linear model and GLM, we denote the global design matrix and the local design matrices by $X_N=(X_1^\top,\dots,X_k^\top)^\top\in\R^{N\times d}$ and $X_j=(x_{1j},\dots,x_{nj})^\top\in\R^{n\times d}$ for $j=1,\dots,k$. We write each covariate vector as $x_{ij}=(x_{ij,1},\dots,x_{ij,d})^\top\in\R^{d\times1}$ for $i=1,\dots,n$ and $j=1,\dots,k$. Also, we denote the global response vector and the local response vectors by $y_N=(y_1^\top,\dots,y_k^\top)^\top\in\R^{N\times1}$ and $y_j=(y_{1j},\dots,y_{nj})\in\R^{n\times1}$ for $j=1,\dots,k$. For linear model, we define the global noise vector and the local noise vectors by $e_N=(e_1^\top,\dots,e_k^\top)^\top\in\R^{N\times1}$ and $e_j=(e_{1j},\dots,e_{nj})\in\R^{n\times1}$ for $j=1,\dots,k$.

\begin{lemma}\label{lem:tbd_ld}
	$T$ and $T_0$ are defined as in \eqref{eqn:tp} and \eqref{eqn:t0} respectively.  In linear model, under Assumptions~\ref{as:design}~and~\ref{as:noise}, provided that $\left\|\ttheta-\htheta\right\|_\infty=O_P(r_{\ttheta})$, we have that
	$$|T-T_0| = O_P\left(r_{\ttheta} \sqrt{N} + \frac{d\sqrt{\log d}}{\sqrt N}\right).$$
	Moreover, if $N\gg d^2\log^{2+\kappa} d$ and
	$$\left\|\ttheta-\htheta\right\|_\infty\ll\frac1{\sqrt N\log^{1/2+\kappa}d},$$
	for some $\kappa>0$, then there exists some $\zeta>0$ such that \eqref{eqn:zeta} holds.
\end{lemma}

\begin{proof}[Lemma \ref{lem:tbd_ld}]
	First, we note that
	\begin{align*}
	|T-T_0| &\leq \max_{1\leq l\leq d} \left|\sqrt{N}(\ttheta-\thetas)_l+\sqrt{N}\left(\nabla^2\cLs(\thetas)^{-1}\nabla\cL_N(\thetas)\right)_l\right| = \sqrt{N}\left\|\ttheta-\thetas+\nabla^2\cLs(\thetas)^{-1}\nabla\cL_N(\thetas)\right\|_\infty \\
	&\leq \sqrt N\left(\left\|\ttheta-\htheta\right\|_\infty+\left\|\htheta-\thetas+\nabla^2\cLs(\thetas)^{-1}\nabla\cL_N(\thetas)\right\|_\infty\right).
	\end{align*}
	Now, we bound $\left\|\htheta-\thetas+\nabla^2\cLs(\thetas)^{-1}\nabla\cL_N(\thetas)\right\|_\infty$.  In linear model, we have that $\htheta=\left(X_N^\top X_N\right)^{-1}X_N^\top y_N=\thetas+\left(X_N^\top X_N\right)^{-1}X_N^\top e_N$, and then,
	\begin{align*}
	\left\|\htheta-\thetas+\nabla^2\cLs(\thetas)^{-1}\nabla\cL_N(\thetas)\right\|_\infty &= \left\|\left(\frac{X_N^\top X_N}N\right)^{-1}\frac{X_N^\top e_N}N-\Theta\frac{X_N^\top e_N}N\right\|_\infty \leq \matrixnorm{\left(\frac{X_N^\top X_N}N\right)^{-1}-\Theta}_\infty \left\|\frac{X_N^\top e_N}N\right\|_\infty.
	\end{align*}
	Under Assumptions~\ref{as:design}~and~\ref{as:noise}, each $x_{ij,l}$ and $e_{ij}$ are sub-Gaussian, and therefore, their product $x_{ij,l}e_{ij}$ is sub-exponential.  Applying Bernstein's inequality, we have that for any $\delta\in(0,1)$,
	$$P\left(\left|\frac{(X_N^\top e_N)_l}N\right| > \sqrt{\Sigma_{l,l}}\sigma\left(\frac{\log\frac{2d}\delta}{cN}\vee\sqrt{\frac{\log\frac{2d}\delta}{cN}}\right)\right)\leq \frac\delta d,$$
	for some constant $c>0$.  Then, by the union bound, we have that
	\begin{align}
	P\left(\left\|\frac{X_N^\top e_N}N\right\|_\infty > \max_l \sqrt{\Sigma_{l,l}}\sigma\left(\frac{\log\frac{2d}\delta}{cN}\vee\sqrt{\frac{\log\frac{2d}\delta}{cN}}\right)\right)\leq \delta. \label{eqn:bern2}
	\end{align}
	Under Assumption~\ref{as:design}, we have that $\max_l \Sigma_{l,l} \leq \matrixnorm{\Sigma}_{\max} = O(1)$, and then,
	$$\left\|\frac{X_N^\top e_N}N\right\|_\infty = O_P\left(\sqrt{\frac{\log d}N}\right).$$
	Using the same argument for obtaining \eqref{eqn:inv_conc}, we have that
	$$\matrixnorm{\left(\frac{X_N^\top X_N}N\right)^{-1}-\Theta}_\infty\leq\sqrt d\matrixnorm{\left(\frac{X_N^\top X_N}N\right)^{-1}-\Theta}_2=O_P\left(\frac d{\sqrt{N}}\right),$$
	and therefore,
	$$\left\|\htheta-\thetas+\nabla^2\cLs(\thetas)^{-1}\nabla\cL_N(\thetas)\right\|_\infty=O_P\left(\frac{d\sqrt{\log d}}N\right).$$
	Putting together the preceding bounds leads to the first result in the lemma. Choosing
	$$\zeta= \left(r_{\ttheta} \sqrt{N} + \frac{d\sqrt{\log d}}{\sqrt N}\right)^{1-\kappa},$$
	with any $\kappa>0$, we deduce that $P\left(|T-T_0|>\zeta\right)=o(1)$.
	We also have that $$\zeta\sqrt{1\vee\log\frac d\zeta},\quad\text{if}\quad\left(r_{\ttheta} \sqrt{N} + \frac{d\sqrt{\log d}}{\sqrt N}\right) \log^{1/2+\kappa} d =o(1).$$
	We complete the proof by simplifying the conditions.
\end{proof}

\begin{lemma}\label{lem:thbd_ld}
	$\widehat T$ and $T_0$ are defined as in \eqref{eqn:thp} and \eqref{eqn:t0} respectively.  In linear model, under Assumptions~\ref{as:design}~and~\ref{as:noise}, we have that
	$$|\widehat T-T_0| = O_P\left(\frac{d\sqrt{\log d}}{\sqrt N}\right).$$
	Moreover, if $N\gg d^2\log^{2+\kappa} d$ for some $\kappa>0$, then there exists some $\xi>0$ such that \eqref{eqn:xi} holds.
\end{lemma}

\begin{proof}[Lemma \ref{lem:thbd_ld}]
	By the proof of Lemma~\ref{lem:tbd_ld}, we obtain that
	\begin{align*}
	|\widehat T-T_0| &\leq \max_{1\leq l\leq d} \left|\sqrt{N}(\htheta-\thetas)_l+\sqrt{N}\left(\nabla^2\cLs(\thetas)^{-1}\nabla\cL_N(\thetas)\right)_l\right| = \sqrt{N}\left\|\htheta-\thetas+\nabla^2\cLs(\thetas)^{-1}\nabla\cL_N(\thetas)\right\|_\infty \\
	&=O_P\left(\frac{d\sqrt{\log d}}{\sqrt N}\right).
	\end{align*}
	Choosing
	$$\xi= \left(\frac{d\sqrt{\log d}}{\sqrt N}\right)^{1-\kappa},$$
	with any $\kappa>0$, we deduce that
	$P\left(|\widehat T-T_0|>\xi\right)=o(1)$.
	We also have that $$\xi\sqrt{1\vee\log\frac d\xi},\quad\text{if}\quad\left(\frac{d\sqrt{\log d}}{\sqrt N}\right) \log^{1/2+\kappa} d =o(1),$$
	which holds if $N\gg d^2\log^{2+\kappa} d$.
	
\end{proof}

\begin{lemma}\label{lem:tbd_ld_glm}
	$T$ and $T_0$ are defined as in \eqref{eqn:tp} and \eqref{eqn:t0} respectively.  In GLM, under Assumptions~\ref{as:smth_glm}--\ref{as:hes_glm}, provided that $\left\|\ttheta-\htheta\right\|_\infty=O_P(r_{\ttheta})$ and $N\gtrsim d^4\log d$, we have that
	$$|T-T_0| = O_P\left(r_{\ttheta} \sqrt{N} + \frac{d^{5/2}\log d}{\sqrt N}\right).$$
	Moreover, if $N\gg d^5\log^{3+\kappa} d$ and
	$$\left\|\ttheta-\htheta\right\|_\infty\ll\frac1{\sqrt N\log^{1/2+\kappa}d},$$
	for some $\kappa>0$, then there exists some $\zeta>0$ such that \eqref{eqn:zeta} holds.
\end{lemma}

\begin{proof}[Lemma \ref{lem:tbd_ld_glm}]
	First, we note that
	\begin{align*}
	|T-T_0| &\leq \max_{1\leq l\leq d} \left|\sqrt{N}(\ttheta-\thetas)_l+\sqrt{N}\left(\nabla^2\cLs(\thetas)^{-1}\nabla\cL_N(\thetas)\right)_l\right| = \sqrt{N}\left\|\ttheta-\thetas+\nabla^2\cLs(\thetas)^{-1}\nabla\cL_N(\thetas)\right\|_\infty \\
	&\leq \sqrt N\left(\left\|\ttheta-\htheta\right\|_\infty+\left\|\htheta-\thetas+\nabla^2\cLs(\thetas)^{-1}\nabla\cL_N(\thetas)\right\|_\infty\right).
	\end{align*}
	Now, we bound $\left\|\htheta-\thetas+\nabla^2\cLs(\thetas)^{-1}\nabla\cL_N(\thetas)\right\|_\infty$.  Note by an expression of remainder of the first order Taylor expansion that
	\begin{align*}
	\left\|\htheta-\thetas+\nabla^2\cLs(\thetas)^{-1}\nabla\cL_N(\thetas)\right\|_\infty &= \left\|\htheta-\thetas-\Theta(\nabla\cL_N(\htheta)-\nabla\cL_N(\thetas))\right\|_\infty \\
	&= \left\|\htheta-\thetas-\Theta\int_0^1\nabla^2\cL_N(\thetas+s(\htheta-\thetas))ds (\htheta-\thetas)\right\|_\infty \\
	&= \left\|\Theta\int_0^1\left(\nabla^2\cLs(\thetas)-\nabla^2\cL_N(\thetas+s(\htheta-\thetas))\right)ds (\htheta-\thetas)\right\|_\infty \\
	&\leq \matrixnorm{\Theta}_\infty \int_0^1\matrixnorm{\nabla^2\cLs(\thetas)-\nabla^2\cL_N(\thetas+s(\htheta-\thetas))}_{\max} ds \left\|\htheta-\thetas\right\|_1.
	\end{align*}
	
	Under Assumption~\ref{as:smth_glm}, we have by an expression of remainder of the first order Taylor expansion that
	\begin{align*}
	\left|g''(y_{ij},x_{ij}^\top(\thetas+s(\htheta-\thetas)))-g''(y_{ij},x_{ij}^\top\thetas)\right| &= \left|\int_0^1 g'''(y_{ij},x_{ij}^\top(\thetas+st(\htheta-\thetas))) dt \cdot tx_{ij}^\top(\htheta-\thetas)\right| \lesssim \left|x_{ij}^\top(\htheta-\thetas)\right|,
	\end{align*}
	and then,
	\begin{align}
	\matrixnorm{\nabla^2\cL_N(\thetas)-\nabla^2\cL_N(\thetas+s(\htheta-\thetas))}_{\max} &=  \matrixnorm{\frac1N\sum_{i=1}^n\sum_{j=1}^k x_{ij} x_{ij}^\top \left(g''(y_{ij},x_{ij}^\top(\thetas+s(\htheta-\thetas)))-g''(y_{ij},x_{ij}^\top\thetas)\right)}_{\max} \notag \\
	&\leq \frac1N\sum_{i=1}^n\sum_{j=1}^k \matrixnorm{x_{ij} x_{ij}^\top \left(g''(y_{ij},x_{ij}^\top(\thetas+s(\htheta-\thetas)))-g''(y_{ij},x_{ij}^\top\thetas)\right)}_{\max} \notag \\
	&= \frac1N\sum_{i=1}^n\sum_{j=1}^k \matrixnorm{x_{ij} x_{ij}^\top}_{\max} \left|g''(y_{ij},x_{ij}^\top(\thetas+s(\htheta-\thetas)))-g''(y_{ij},x_{ij}^\top\thetas)\right| \notag \\
	&\lesssim \frac1N\sum_{i=1}^n\sum_{j=1}^k \|x_{ij}\|_\infty^2 \left|x_{ij}^\top(\htheta-\thetas)\right| \leq \frac1N\sum_{i=1}^n\sum_{j=1}^k \|x_{ij}\|_\infty^3 \|\htheta-\thetas\|_1 \notag \\
	&\lesssim \left\|\htheta-\thetas\right\|_1, \label{eqn:lip}
	\end{align}
	where we use that $\|x_{ij}\|_\infty=O(1)$ under Assumption~\ref{as:design_glm} in the last inequality.  Note that
	\begin{align*}
	\matrixnorm{\nabla^2\cL_N(\thetas)-\nabla^2\cLs(\thetas)}_{\max} = \matrixnorm{\frac1N\sum_{i=1}^n\sum_{j=1}^k g''(y_{ij},x_{ij}^\top\thetas) x_{ij} x_{ij}^\top-\Ee\left[g''(y,x^\top\thetas)xx^\top\right]}_{\max},
	\end{align*}
	and $g''(y_{ij},x_{ij}^\top\thetas)=O(1)$ under Assumption~\ref{as:smth_glm}.  Then, we have that by Hoeffding's inequality,
	$$P\left(\frac{\sum_{i=1}^n\sum_{j=1}^k g''(y_{ij},x_{ij}^\top\thetas) x_{ij,l} x_{ij,l'}}N-\Ee\left[g''(y,x^\top\thetas) x_l x_{l'}\right]>\sqrt{\frac{2\log(\frac{2d^2}\delta)}N}\right)\leq\frac\delta{d^2},$$
	and by the union bound, for any $\delta\in(0,1)$, with probability at least $1-\delta$,
	$$\matrixnorm{\nabla^2\cL_N(\thetas)-\nabla^2\cLs(\thetas)}_{\max}\leq\sqrt{\frac{2\log(\frac{2d^2}\delta)}N},$$
	which implies that
	\begin{align}
	\matrixnorm{\nabla^2\cL_N(\thetas)-\nabla^2\cLs(\thetas)}_{\max}=O_P\left(\sqrt{\frac{\log d}N}\right). \label{eqn:hoef3}
	\end{align}
	Then, by the triangle inequality, we have that
	\begin{align*}
	&\matrixnorm{\nabla^2\cLs(\thetas)-\nabla^2\cL_N(\thetas+s(\htheta-\thetas))}_{\max} \\
	&\leq \matrixnorm{\nabla^2\cL_N(\thetas+s(\htheta-\thetas))-\nabla^2\cL_N(\thetas)}_{\max}+\matrixnorm{\nabla^2\cL_N(\thetas)-\nabla^2\cLs(\thetas)}_{\max} \lesssim \left\|\htheta-\thetas\right\|_1 + O_P\left(\sqrt{\frac{\log d}N}\right).
	\end{align*}
	Note that $\matrixnorm{\Theta}_\infty\leq\sqrt d\matrixnorm{\Theta}_2=O\left(\sqrt d\right)$.  By Lemma~\ref{lem:m_glm}, if $N\gtrsim d^4\log d$, we have that
	$$\left\|\htheta-\thetas\right\|_1\leq\sqrt d\left\|\htheta-\thetas\right\|_2=O_P\left(\frac{d\sqrt{\log d}}{\sqrt{N}}\right),$$
	and therefore,
	$$\left\|\htheta-\thetas+\nabla^2\cLs(\thetas)^{-1}\nabla\cL_N(\thetas)\right\|_\infty=O_P\left(\frac{d^{5/2}\log d}N\right).$$
	Putting together the preceding bounds leads to the first result in the lemma. Choosing
	$$\zeta= \left(r_{\ttheta} \sqrt{N} + \frac{d^{5/2}\log d}{\sqrt N}\right)^{1-\kappa},$$
	with any $\kappa>0$, we deduce that
	$P\left(|T-T_0|>\zeta\right)=o(1)$.
	We also have that $$\zeta\sqrt{1\vee\log\frac d\zeta},\quad\text{if}\quad\left(r_{\ttheta} \sqrt{N} + \frac{d^{5/2}\log d}{\sqrt N}\right) \log^{1/2+\kappa} d =o(1).$$
	We complete the proof by simplifying the conditions.
\end{proof}

\begin{lemma}\label{lem:thbd_ld_glm}
	$\widehat T$ and $T_0$ are defined as in \eqref{eqn:thp} and \eqref{eqn:t0} respectively.  In GLM, under Assumptions~\ref{as:smth_glm}--\ref{as:hes_glm}, provided that $\left\|\ttheta-\htheta\right\|_\infty=O_P(r_{\ttheta})$ and $N\gtrsim d^4\log d$, we have that
	$$|\widehat T-T_0| = O_P\left(r_{\ttheta} \sqrt{N} + \frac{d^{5/2}\log d}{\sqrt N}\right).$$
	Moreover, if $N\gg d^5\log^{3+\kappa} d$ for some $\kappa>0$, then there exists some $\xi>0$ such that \eqref{eqn:xi} holds.
\end{lemma}

\begin{proof}[Lemma \ref{lem:thbd_ld_glm}]
	By the proof of Lemma~\ref{lem:tbd_ld_glm}, we obtain that if $N\gtrsim d^4\log d$,
	\begin{align*}
	|\widehat T-T_0| &\leq \max_{1\leq l\leq d} \left|\sqrt{N}(\htheta-\thetas)_l+\sqrt{N}\left(\nabla^2\cLs(\thetas)^{-1}\nabla\cL_N(\thetas)\right)_l\right| \\
	&= \sqrt{N}\left\|\htheta-\thetas+\nabla^2\cLs(\thetas)^{-1}\nabla\cL_N(\thetas)\right\|_\infty =O_P\left(\frac{d^{5/2}\log d}{\sqrt N}\right).
	\end{align*}
	Choosing
	$$\xi=\left(\frac{d^{5/2}\log d}{\sqrt N}\right)^{1-\kappa},$$
	with any $\kappa>0$, we deduce that
	$P\left(|\widehat T-T_0|>\xi\right)=o(1)$.
	We also have that $$\xi\sqrt{1\vee\log\frac d\xi},\quad\text{if}\quad\left(\frac{d^{5/2}\log d}{\sqrt N}\right) \log^{1/2+\kappa} d =o(1),$$
	which holds if
	$N\gg d^5\log^{3+\kappa} d$.
\end{proof}

\section{Lemmas on Bounding Variance Estimation Errors}

\begin{lemma}\label{lem:gbd0_ld}
	$\overline\Omega$ and $\widehat\Omega$ are defined as in \eqref{eqn:ob} and \eqref{eqn:oh} respectively.  In linear model, under Assumptions~\ref{as:design}~and~\ref{as:noise}, provided that $\left\|\btheta-\thetas\right\|_1=O_P(r_{\btheta})$, $r_{\btheta}\sqrt{\log(kd)}\lesssim 1$, $n\gtrsim d$, and $k\gtrsim\log^2(dk)\log d$, we have that
	$$\matrixnorm{\overline\Omega-\widehat\Omega}_{\max} = O_P\left(d\left(\sqrt{\frac{\log d}k} + \frac{\log^2(dk)\log d}k + \sqrt{\log(kd)}r_{\btheta} + nr_{\btheta}^2\right) + \sqrt{\frac dn}\right).$$
	Moreover, if $n\gg d\log^{4+\kappa} d$, $k\gg d^2\log^{5+\kappa} d$, and
	$$\left\|\btheta-\thetas\right\|_1\ll\min\left\{\frac1{d\sqrt{\log(kd)}\log^{2+\kappa}d},\frac1{ \sqrt{nd}\log^{1+\kappa}d}\right\},$$
	for some $\kappa>0$, then there exists some $u>0$ such that \eqref{eqn:u} holds.
\end{lemma}

\begin{proof}[Lemma \ref{lem:gbd0_ld}]
	Note by the triangle inequality that
	$$\matrixnorm{\overline\Omega-\widehat\Omega}_{\max}\leq\matrixnorm{\overline\Omega-\Omega_0}_{\max} + \matrixnorm{\widehat\Omega-\Omega_0}_{\max},$$
	where $\Omega_0$ is defined as in \eqref{eqn:o0}.  First, we bound $\matrixnorm{\widehat\Omega-\Omega_0}_{\max}$.  With Assumption~(E.1) of \cite{chernozhukov2013gaussian} verified for $\nabla^2\cLs(\thetas)^{-1}\nabla\cL(\thetas;Z)$ in the proof of Lemma~\ref{lem:ld0}, by the proof of Corollary 3.1 of \cite{chernozhukov2013gaussian}, we have that
	\begin{align*}
	\Ee\left[\matrixnorm{\widehat\Omega-\Omega_0}_{\max}\right]\lesssim \sqrt{\frac{\log d}N} + \frac{\log^2(dN)\log d}N,
	\end{align*}
	which implies that
	$$\matrixnorm{\widehat\Omega-\Omega_0}_{\max} = O_P\left(\sqrt{\frac{\log d}N} + \frac{\log^2(dN)\log d}N\right).$$

	Next, we bound $\matrixnorm{\overline\Omega-\Omega_0}_{\max}$.  By the triangle inequality, we have that
	\begin{align*}
	&\matrixnorm{\overline\Omega-\Omega_0}_{\max}\\
	&=\matrixnorm{\widetilde\Theta\left(\frac1k\sum_{j=1}^k n\left(\nabla\cL_j(\btheta)-\nabla\cL_N(\btheta)\right)\left(\nabla\cL_j(\btheta)-\nabla\cL_N(\btheta)\right)^\top\right)\widetilde\Theta^\top-\Theta\Ee\left[\nabla\cL(\thetas;Z)\nabla\cL(\thetas;Z)^\top\right]\Theta}_{\max} \\
	&\leq \matrixnorm{\widetilde\Theta\left(\frac1k\sum_{j=1}^k n\left(\nabla\cL_j(\btheta)-\nabla\cL_N(\btheta)\right)\left(\nabla\cL_j(\btheta)-\nabla\cL_N(\btheta)\right)^\top-\Ee\left[\nabla\cL(\thetas;Z)\nabla\cL(\thetas;Z)^\top\right]\right)\widetilde\Theta}_{\max} \\
	&\quad+ \matrixnorm{\widetilde\Theta\Ee\left[\nabla\cL(\thetas;Z)\nabla\cL(\thetas;Z)^\top\right]\widetilde\Theta^\top-\Theta\Ee\left[\nabla\cL(\thetas;Z)\nabla\cL(\thetas;Z)^\top\right]\Theta}_{\max} \\
	&\defn I_1(\btheta) + I_2.
	\end{align*}
	To bound $I_1(\btheta)$, we use the fact that for any two matrices $A$ and $B$ with compatible dimensions, $\matrixnorm{AB}_{\max}\leq\matrixnorm{A}_\infty\matrixnorm{B}_{\max}$ and $\matrixnorm{AB}_{\max}\leq\matrixnorm{A}_{\max}\matrixnorm{B}_1$, and obtain that
	\begin{align*}
	I_1(\btheta) &\leq \matrixnorm{\widetilde\Theta}_\infty \matrixnorm{\frac1k\sum_{j=1}^k n\left(\nabla\cL_j(\btheta)-\nabla\cL_N(\btheta)\right)\left(\nabla\cL_j(\btheta)-\nabla\cL_N(\btheta)\right)^\top-\Ee\left[\nabla\cL(\thetas;Z)\nabla\cL(\thetas;Z)^\top\right]}_{\max} \matrixnorm{\widetilde\Theta^\top}_1 \\
	&= \matrixnorm{\widetilde\Theta}_\infty^2 \matrixnorm{\frac1k\sum_{j=1}^k n\left(\nabla\cL_j(\btheta)-\nabla\cL_N(\btheta)\right)\left(\nabla\cL_j(\btheta)-\nabla\cL_N(\btheta)\right)^\top-\Ee\left[\nabla\cL(\thetas;Z)\nabla\cL(\thetas;Z)^\top\right]}_{\max}.
	\end{align*}
	Under Assumption~\ref{as:design}, by Lemma~\ref{lem:hes_ld}, if $n\gtrsim d$, we have that $\matrixnorm{\widetilde\Theta}_\infty=O_P\left( \sqrt{d}\right)$,  Then, applying Lemma~\ref{lem:vcov0_reg}, we have that
	\begin{align*}
	I_1(\btheta) &= O_P\left(d\right) O_P\left(\sqrt{\frac{\log d}k} + \frac{\log^2(dk)\log d}k + \sqrt{\log(kd)}r_{\btheta} + nr_{\btheta}^2\right) \\
	&= O_P\left(d\left(\sqrt{\frac{\log d}k} + \frac{\log^2(dk)\log d}k + \sqrt{\log(kd)}r_{\btheta} + nr_{\btheta}^2\right)\right),
	\end{align*}
	under Assumptions~\ref{as:design}~and~\ref{as:noise}, provided that $\left\|\btheta-\thetas\right\|_1=O_P(r_{\btheta})$, $r_{\btheta}\sqrt{\log(kd)}\lesssim 1$, and $k\gtrsim\log^2(dk)\log d$.
	
	It remains to bound $I_2$.  In linear model, we have that
	$$I_2=\matrixnorm{\widetilde\Theta\left(\sigma^2\Sigma\right)\widetilde\Theta^\top-\Theta\left(\sigma^2\Sigma\right)\Theta}_{\max}=\sigma^2\matrixnorm{\widetilde\Theta\Sigma\widetilde\Theta^\top-\Theta}_{\max},$$
	and by the triangle inequality,
	\begin{align*}
	I_2&=\sigma^2\matrixnorm{(\widetilde\Theta-\Theta+\Theta)\Sigma(\widetilde\Theta-\Theta+\Theta)^\top-\Theta}_{\max} \\
	&=\sigma^2\matrixnorm{(\widetilde\Theta-\Theta)\Sigma(\widetilde\Theta-\Theta)^\top+\Theta\Sigma(\widetilde\Theta-\Theta)^\top+(\widetilde\Theta-\Theta)\Sigma\Theta+\Theta\Sigma\Theta-\Theta}_{\max} \\
	&\leq\sigma^2\matrixnorm{(\widetilde\Theta-\Theta)\Sigma(\widetilde\Theta-\Theta)^\top}_{\max} + 2\sigma^2\matrixnorm{\widetilde\Theta-\Theta}_{\max}.
	\end{align*}
	By Lemma~\ref{lem:hes_ld}, we have that
	$$\matrixnorm{\widetilde\Theta-\Theta}_{\max}\leq\max_l\left\|\widetilde\Theta_l-\Theta_l\right\|_2=O_P\left( \sqrt{\frac dn}\right),\quad\text{and}$$
	\begin{align*}
	\matrixnorm{(\widetilde\Theta-\Theta)\Sigma(\widetilde\Theta-\Theta)^\top}_{\max}&\leq\matrixnorm{\Sigma}_2\max_l\left\|\widetilde\Theta_l-\Theta_l\right\|_2^2=O_P\left( \frac dn\right),
	\end{align*}
	where we use that $\matrixnorm{\Sigma}_{\max}\leq\matrixnorm{\Sigma}_2=O(1)$ under Assumption~\ref{as:design}.  Then, we obtain that
	$$I_2 = O_P\left( \frac dn\right) + O_P\left( \sqrt{\frac dn}\right) = O_P\left( \sqrt{\frac dn}\right).$$
	Putting all the preceding bounds together, we obtain that
	\begin{align*}
        \matrixnorm{\overline\Omega-\Omega_0}_{\max} &= O_P\left(d\left(\sqrt{\frac{\log d}k} + \frac{\log^2(dk)\log d}k + \sqrt{\log(kd)}r_{\btheta} + nr_{\btheta}^2\right) + \sqrt{\frac dn}\right),
        \end{align*}
        and finally the first result in the lemma.
	Choosing
	$$u= \left(d\sqrt{\frac{\log d}k} + \frac{d\log^2(dk)\log d}k + d\sqrt{\log(kd)}r_{\btheta} + nd r_{\btheta}^2 + \sqrt{\frac dn}\right)^{1-\kappa},$$
	with any $\kappa>0$, we deduce that $P\left(\matrixnorm{\overline\Omega-\widehat\Omega}_{\max}>u\right)=o(1)$. We also have that $$u^{1/3}\left(1\vee\log\frac du\right)^{2/3},\quad\text{if}\quad\left(d\sqrt{\frac{\log d}k} + \frac{d\log^2(dk)\log d}k + d\sqrt{\log(kd)}r_{\btheta} + nd r_{\btheta}^2 + \sqrt{\frac dn}\right) \log^{2+\kappa} d =o(1).$$
	We complete the proof by simplifying the conditions.
	
\end{proof}

\begin{lemma}\label{lem:gbdo_ld}
	$\widehat\Omega$ and $\Omega_0$ is defined as in \eqref{eqn:oh} and \eqref{eqn:o0} respectively.  In linear model, under Assumptions~\ref{as:design}~and~\ref{as:noise}, we have that
	$$\matrixnorm{\widehat\Omega-\Omega_0}_{\max} = O_P\left(\sqrt{\frac{\log d}N} + \frac{\log^2(dN)\log d}N\right).$$
	Moreover, if $N\gg\log^{5+\kappa}d$ for some $\kappa>0$, then there exists some $v>0$ such that \eqref{eqn:v} holds.
\end{lemma}

\begin{proof}[Lemma \ref{lem:gbdo_ld}]
	The first result is derived in the proof of Lemma~\ref{lem:gbd0_ld}. Choosing
	$$v=\left(\sqrt{\frac{\log d}N} + \frac{\log^2(dN)\log d}N\right)^{1-\kappa},$$
	with any $\kappa>0$, we deduce that $P\left(\matrixnorm{\widehat\Omega-\Omega_0}_{\max}>v\right)=o(1)$. We also have that
	$$v^{1/3}\left(1\vee\log\frac dv\right)^{2/3},\quad\text{if}\quad\left(\sqrt{\frac{\log d}N} + \frac{\log^2(dN)\log d}N\right) \log^{2+\kappa} d=o(1),$$
	which holds if $N\gg\log^{5+\kappa}d$.
\end{proof}

\begin{lemma}\label{lem:gbd_ld}
	$\widetilde\Omega$ and $\widehat\Omega$ are defined as in \eqref{eqn:ot} and \eqref{eqn:oh} respectively.  In linear model, under Assumptions~\ref{as:design}~and~\ref{as:noise}, provided that $\left\|\btheta-\thetas\right\|_1=O_P(r_{\btheta})$, $r_{\btheta}\sqrt{\log((n+k)d)}\lesssim 1$, and $n\gtrsim d$, we have that
	$$\matrixnorm{\widetilde\Omega-\widehat\Omega}_{\max} = O_P\left(d\left(\sqrt{\frac{\log d}{n+k}} + \frac{\log^2(d(n+k))\log d}{n+k} + \sqrt{\log((n+k)d)} r_{\btheta} + \frac{nk}{n+k}r_{\btheta}^2\right) + \sqrt{\frac dn}\right).$$
	Moreover, if $n\gg d\log^{4+\kappa} d$, $n+k\gg d^2\log^{5+\kappa} d$, and
	$$\left\|\btheta-\thetas\right\|_1\ll\min\left\{\frac1{d\sqrt{\log((n+k)d)}\log^{2+\kappa}d},\frac1{ \sqrt{d}\log^{1+\kappa}d}\sqrt{\frac1n+\frac1k}\right\},$$
	for some $\kappa>0$, then there exists some $u>0$ such that \eqref{eqn:ut} holds.
\end{lemma}

\begin{proof}[Lemma \ref{lem:gbd_ld}]
	Note by the triangle inequality that
	$$\matrixnorm{\widetilde\Omega-\widehat\Omega}_{\max}\leq\matrixnorm{\widetilde\Omega-\Omega_0}_{\max} + \matrixnorm{\widehat\Omega-\Omega_0}_{\max},$$
	where $\Omega_0$ is defined as in \eqref{eqn:o0}.  By the proof of Lemma~\ref{lem:gbd0_ld}, we have that
	$$\matrixnorm{\widehat\Omega-\Omega_0}_{\max} = O_P\left(\sqrt{\frac{\log d}N} + \frac{\log^2(dN)\log d}N\right).$$
	Next, we bound $\matrixnorm{\widetilde\Omega-\Omega_0}_{\max}$ using the same argument as in the proof of Lemma~\ref{lem:gbd0_ld}.  By the triangle inequality, we have that
	\begin{align*}
	&\matrixnorm{\widetilde\Omega-\Omega_0}_{\max}\\
	&=\Bigg|\!\Bigg|\!\Bigg|\widetilde\Theta\frac1{n+k-1}\Bigg(\sum_{i=1}^n\left(\nabla\cL(\theta;Z_{i1})-\nabla\cL_N(\btheta)\right)\left(\nabla\cL(\theta;Z_{i1})-\nabla\cL_N(\btheta)\right)^\top \\
	&\quad+\sum_{j=2}^k n\left(\nabla\cL_j(\btheta)-\nabla\cL_N(\btheta)\right)\left(\nabla\cL_j(\btheta)-\nabla\cL_N(\btheta)\right)^\top\Bigg)\widetilde\Theta^\top-\Theta\Ee\left[\nabla\cL(\thetas;Z)\nabla\cL(\thetas;Z)^\top\right]\Theta\Bigg|\!\Bigg|\!\Bigg|_{\max} \\
	&\leq \Bigg|\!\Bigg|\!\Bigg|\widetilde\Theta\Bigg(\frac1{n+k-1}\Bigg(\sum_{i=1}^n\left(\nabla\cL(\theta;Z_{i1})-\nabla\cL_N(\btheta)\right)\left(\nabla\cL(\theta;Z_{i1})-\nabla\cL_N(\btheta)\right)^\top \\
	&\quad+\sum_{j=2}^k n\left(\nabla\cL_j(\btheta)-\nabla\cL_N(\btheta)\right)\left(\nabla\cL_j(\btheta)-\nabla\cL_N(\btheta)\right)^\top\Bigg)-\Ee\left[\nabla\cL(\thetas;Z)\nabla\cL(\thetas;Z)^\top\right]\Bigg)\widetilde\Theta^\top\Bigg|\!\Bigg|\!\Bigg|_{\max} \\
	&\quad+ \matrixnorm{\widetilde\Theta\Ee\left[\nabla\cL(\thetas;Z)\nabla\cL(\thetas;Z)^\top\right]\widetilde\Theta^\top-\Theta\Ee\left[\nabla\cL(\thetas;Z)\nabla\cL(\thetas;Z)^\top\right]\Theta}_{\max} \\
	&\defn I_1'(\btheta) + I_2.
	\end{align*}
	We have shown in the proof of Lemma~\ref{lem:gbd0_ld} that
	$$I_2 = O_P\left( \sqrt{\frac dn}\right).$$
	To bound $I_1'(\btheta)$, we note that
	\begin{align*}
	I_1'(\btheta) &\leq \matrixnorm{\widetilde\Theta}_\infty^2 \Bigg|\!\Bigg|\!\Bigg|\frac1{n+k-1}\Bigg(\sum_{i=1}^n\left(\nabla\cL(\btheta;Z_{i1})-\nabla\cL_N(\btheta)\right)\left(\nabla\cL(\btheta;Z_{i1})-\nabla\cL_N(\btheta)\right)^\top \\
	&\quad+\sum_{j=2}^k n\left(\nabla\cL_j(\btheta)-\nabla\cL_N(\btheta)\right)\left(\nabla\cL_j(\btheta)-\nabla\cL_N(\btheta)\right)^\top\Bigg) -\Ee\left[\nabla\cL(\thetas;Z)\nabla\cL(\thetas;Z)^\top\right]\Bigg|\!\Bigg|\!\Bigg|_{\max}.
	\end{align*}
	Under Assumption~\ref{as:design}, by Lemma~\ref{lem:hes_ld}, if $n\gtrsim d$, we have that
	$$\matrixnorm{\widetilde\Theta}_\infty=O_P\left( \sqrt{d}\right).$$
	Then, applying Lemma~\ref{lem:vcov_reg}, we have that
	\begin{align*}
	I_1'(\btheta) = O_P\left(d\left(\sqrt{\frac{\log d}{n+k}} + \frac{\log^2(d(n+k))\log d}{n+k} + \sqrt{\log((n+k)d)} r_{\btheta} + \frac{nk}{n+k}r_{\btheta}^2\right)\right),
	\end{align*}
	under Assumptions~\ref{as:design}~and~\ref{as:noise}, provided that $\left\|\btheta-\thetas\right\|_1=O_P(r_{\btheta})$, $r_{\btheta}\sqrt{\log((n+k)d)}\lesssim 1$, and $n+k\gtrsim\log^2(d(n+k))\log d$.
	Putting all the preceding bounds together, we obtain that
	\begin{align*}
        \matrixnorm{\widetilde\Omega-\Omega_0}_{\max} &= O_P\left(d\left(\sqrt{\frac{\log d}{n+k}} + \frac{\log^2(d(n+k))\log d}{n+k} + \sqrt{\log((n+k)d)} r_{\btheta} + \frac{nk}{n+k}r_{\btheta}^2\right) + \sqrt{\frac dn}\right),
        \end{align*}
        and finally the first result in the lemma.
        Choosing
	$$u= \left(d\sqrt{\frac{\log d}{n+k}} + \frac{d\log^2(d(n+k))\log d}{n+k} + d\sqrt{\log((n+k)d)} r_{\btheta} + \frac{nkd}{n+k}r_{\btheta}^2 + \sqrt{\frac dn}\right)^{1-\kappa},$$
	with any $\kappa>0$, we deduce that $P\left(\matrixnorm{\widetilde\Omega-\widehat\Omega}_{\max}>u\right)=o(1)$.
	We also have that
	$$u^{1/3}\left(1\vee\log\frac du\right)^{2/3},\quad\text{if}$$
	$$\left(d\sqrt{\frac{\log d}{n+k}} + \frac{d\log^2(d(n+k))\log d}{n+k} + d\sqrt{\log((n+k)d)} r_{\btheta} + \frac{nkd}{n+k}r_{\btheta}^2 + \sqrt{\frac dn}\right) \log^{2+\kappa} d =o(1).$$
	We complete the proof by simplifying the conditions.
	
\end{proof}

\begin{lemma}\label{lem:gbd0_ld_glm}
	$\overline\Omega$ and $\widehat\Omega$ are defined as in \eqref{eqn:ob_glm} and \eqref{eqn:oh} respectively.  In GLM, under Assumptions~\ref{as:smth_glm}--\ref{as:subexp_glm}, provided that $\left\|\btheta-\thetas\right\|_1=O_P(r_{\btheta})$, $r_{\btheta}\lesssim 1$, $n\gtrsim d\log d$, and $k\gtrsim\log d$, we have that
	$$\matrixnorm{\overline\Omega-\widehat\Omega}_{\max} = O_P\left(d\left(\sqrt{\frac{\log d}k} + \sqrt{\log d} r_{\btheta} + n r_{\btheta}^2\right) + \sqrt{\frac{d\log d}n}\right).$$
	Moreover, if $n\gg d\log^{5+\kappa} d$, $k\gg d^2\log^{5+\kappa} d$, and
	$$\left\|\btheta-\thetas\right\|_1\ll\min\left\{\frac1{d\log^{5/2+\kappa}d},\frac1{ \sqrt{nd}\log^{1+\kappa}d}\right\},$$
	for some $\kappa>0$, then there exists some $u>0$ such that \eqref{eqn:u} holds.
\end{lemma}

\begin{proof}[Lemma \ref{lem:gbd0_ld_glm}]
	We use the same argument as in the proof of Lemma~\ref{lem:gbd0_ld}.  Note by the triangle inequality that
	$$\matrixnorm{\overline\Omega-\widehat\Omega}_{\max}\leq\matrixnorm{\overline\Omega-\Omega_0}_{\max} + \matrixnorm{\widehat\Omega-\Omega_0}_{\max},$$
	where $\Omega_0$ is defined as in \eqref{eqn:o0}.  First, we bound $\matrixnorm{\widehat\Omega-\Omega_0}_{\max}$.  With Assumption~(E.1) of \cite{chernozhukov2013gaussian} verified for $\nabla^2\cLs(\thetas)^{-1}\nabla\cL(\thetas;Z)$ in the proof of Lemma~\ref{lem:ld0_glm}, by the proof of Corollary 3.1 of \cite{chernozhukov2013gaussian}, we have that
	$$\matrixnorm{\widehat\Omega-\Omega_0}_{\max} = O_P\left(\sqrt{\frac{\log d}N} + \frac{\log^2(dN)\log d}N\right).$$

	Next, we bound $\matrixnorm{\overline\Omega-\Omega_0}_{\max}$.  By the triangle inequality, we have that
	\begin{align*}
	&\matrixnorm{\overline\Omega-\Omega_0}_{\max}\\
	&=\matrixnorm{\widetilde\Theta(\btheta)\left(\frac1k\sum_{j=1}^k n\left(\nabla\cL_j(\btheta)-\nabla\cL_N(\btheta)\right)\left(\nabla\cL_j(\btheta)-\nabla\cL_N(\btheta)\right)^\top\right)\widetilde\Theta(\btheta)^\top -\Theta\Ee\left[\nabla\cL(\thetas;Z)\nabla\cL(\thetas;Z)^\top\right]\Theta}_{\max} \\
	&\leq \matrixnorm{\widetilde\Theta(\btheta)\left(\frac1k\sum_{j=1}^k n\left(\nabla\cL_j(\btheta)-\nabla\cL_N(\btheta)\right)\left(\nabla\cL_j(\btheta)-\nabla\cL_N(\btheta)\right)^\top-\Ee\left[\nabla\cL(\thetas;Z)\nabla\cL(\thetas;Z)^\top\right]\right)\widetilde\Theta(\btheta)^\top}_{\max} \\
	&\quad + \matrixnorm{\widetilde\Theta(\btheta)\Ee\left[\nabla\cL(\thetas;Z)\nabla\cL(\thetas;Z)^\top\right]\widetilde\Theta(\btheta)^\top -\Theta\Ee\left[\nabla\cL(\thetas;Z)\nabla\cL(\thetas;Z)^\top\right]\Theta}_{\max} \\
	&\defn I_1(\btheta) + I_2.
	\end{align*}
	Note that
	\begin{align*}
	&\widetilde\Theta(\btheta)\Ee\left[\nabla\cL(\thetas;Z)\nabla\cL(\thetas;Z)^\top\right]\widetilde\Theta(\btheta)^\top \\
	&= \left(\widetilde\Theta(\btheta)-\Theta\right)\Ee\left[\nabla\cL(\thetas;Z)\nabla\cL(\thetas;Z)^\top\right]\left(\widetilde\Theta(\btheta)-\Theta\right)^\top + \Theta\Ee\left[\nabla\cL(\thetas;Z)\nabla\cL(\thetas;Z)^\top\right]\left(\widetilde\Theta(\btheta)-\Theta\right)^\top \\
	&\quad + \left(\widetilde\Theta(\btheta)-\Theta\right)\Ee\left[\nabla\cL(\thetas;Z)\nabla\cL(\thetas;Z)^\top\right]\Theta + \Theta\Ee\left[\nabla\cL(\thetas;Z)\nabla\cL(\thetas;Z)^\top\right]\Theta.
	\end{align*}
	By the triangle inequality, we have that
	\begin{align*}
	I_2 &\leq\matrixnorm{\left(\widetilde\Theta(\btheta)-\Theta\right)\Ee\left[\nabla\cL(\thetas;Z)\nabla\cL(\thetas;Z)^\top\right]\left(\widetilde\Theta(\btheta)-\Theta\right)^\top}_{\max} \\
	&\quad + 2\matrixnorm{\Theta\Ee\left[\nabla\cL(\thetas;Z)\nabla\cL(\thetas;Z)^\top\right]\left(\widetilde\Theta(\btheta)-\Theta\right)^\top}_{\max} \\
	&\leq\matrixnorm{\Ee\left[\nabla\cL(\thetas;Z)\nabla\cL(\thetas;Z)^\top\right]}_2 \max_l \left\|\widetilde\Theta(\btheta)_l-\Theta_l\right\|_2^2 \\
	&\quad + 2\matrixnorm{\Ee\left[\nabla\cL(\thetas;Z)\nabla\cL(\thetas;Z)^\top\right]}_2 \max_l \left\|\Theta_l\right\|_2 \max_l \left\|\widetilde\Theta(\btheta)_l-\Theta_l\right\|_2.
	\end{align*}
	Note that $\max_l \left\|\Theta_l\right\|_2\leq\matrixnorm{\Theta}_2=O(1)$ under Assumption~\ref{as:hes_glm}.  By Lemma~\ref{lem:hes_ld_glm}, provided that $n\gtrsim d\log d$ and $r_{\btheta}\lesssim1$, we have that
	\begin{align*}
	I_2&=O_P\left(\frac{d\log d}n+r_{\btheta}^2+\sqrt{\frac{d\log d}n}+r_{\btheta}\right) =O_P\left(\sqrt{\frac{d\log d}n}+r_{\btheta}\right).
	\end{align*}
	To bound $I_1(\btheta)$, we note that
	\begin{align*}
	I_1(\btheta) &\leq \matrixnorm{\widetilde\Theta(\btheta)}_\infty^2 \matrixnorm{\frac1k\sum_{j=1}^k n\left(\nabla\cL_j(\btheta)-\nabla\cL_N(\btheta)\right)\left(\nabla\cL_j(\btheta)-\nabla\cL_N(\btheta)\right)^\top-\Ee\left[\nabla\cL(\thetas;Z)\nabla\cL(\thetas;Z)^\top\right]}_{\max}.
	\end{align*}
	By Lemma~\ref{lem:hes_ld_glm}, we have that
	$$\matrixnorm{\widetilde\Theta(\btheta)}_\infty=O_P\left( \sqrt{d}\right).$$
	Then, applying Lemma~\ref{lem:vcov0_reg_glm}, we obtain that
	$$I_1(\btheta) = O_P\left(d\left(\sqrt{\frac{\log d}k} + \sqrt{\log d} r_{\btheta} + n r_{\btheta}^2\right)\right),$$
	provided that $\left\|\btheta-\thetas\right\|_1=O_P(r_{\btheta})$, $r_{\btheta}\lesssim 1$, $n\gtrsim\log d$, and $k\gtrsim\log d$.
	Putting all the preceding bounds together, we obtain that
        \begin{align*}
        \matrixnorm{\overline\Omega-\Omega_0}_{\max} &= O_P\left(d\left(\sqrt{\frac{\log d}k} + \sqrt{\log d} r_{\btheta} + n r_{\btheta}^2\right) + \sqrt{\frac{d\log d}n}\right),
        \end{align*}
        and finally the first result in the lemma.
        Choosing
	$$u= \left(d\sqrt{\frac{\log d}k} + d\sqrt{\log d} r_{\btheta} + nd r_{\btheta}^2 + \sqrt{\frac{d\log d}n}\right)^{1-\kappa},$$
	with any $\kappa>0$, we deduce that $P\left(\matrixnorm{\overline\Omega-\widehat\Omega}_{\max}>u\right)=o(1)$.
	We also have that
	$$u^{1/3}\left(1\vee\log\frac du\right)^{2/3},\quad\text{if}\quad\left(d\sqrt{\frac{\log d}k} + d\sqrt{\log d} r_{\btheta} + nd r_{\btheta}^2 + \sqrt{\frac{d\log d}n}\right) \log^{2+\kappa} d =o(1).$$
	We complete the proof by simplifying the conditions.
	
\end{proof}

\begin{lemma}\label{lem:gbdo_ld_glm}
	$\widehat\Omega$ and $\Omega_0$ is defined as in \eqref{eqn:oh} and \eqref{eqn:o0} respectively.  In GLM, under Assumptions~\ref{as:hes_glm}--\ref{as:subexp_glm}, we have that
	$$\matrixnorm{\widehat\Omega-\Omega_0}_{\max} = O_P\left(\sqrt{\frac{\log d}N} + \frac{\log^2(dN)\log d}N\right).$$
	Moreover, if $N\gg\log^{5+\kappa}d$ for some $\kappa>0$, then there exists some $v>0$ such that \eqref{eqn:v} holds.
\end{lemma}

\begin{proof}[Lemma \ref{lem:gbdo_ld_glm}]
	The first result is derived in the proof of Lemma~\ref{lem:gbd0_ld_glm}. Choosing
	$$v=\left(\sqrt{\frac{\log d}N} + \frac{\log^2(dN)\log d}N\right)^{1-\kappa},$$
	with any $\kappa>0$, we deduce that
	$P\left(\matrixnorm{\widehat\Omega-\Omega_0}_{\max}>v\right)=o(1)$.
	We also have that
	$$v^{1/3}\left(1\vee\log\frac dv\right)^{2/3},\quad\text{if}\quad\left(\sqrt{\frac{\log d}N} + \frac{\log^2(dN)\log d}N\right) \log^{2+\kappa} d=o(1),$$
	which holds if $N\gg\log^{5+\kappa}d$.
\end{proof}

\begin{lemma}\label{lem:gbd_ld_glm}
	$\widetilde\Omega$ and $\widehat\Omega$ are defined as in \eqref{eqn:ot_glm} and \eqref{eqn:oh} respectively.  In GLM, under Assumptions~\ref{as:smth_glm}--\ref{as:subexp_glm}, provided that $\left\|\btheta-\thetas\right\|_1=O_P(r_{\btheta})$, $r_{\btheta}\lesssim 1$, and $n\gtrsim d\log d$, we have that
	\begin{align*}
	\matrixnorm{\widetilde\Omega-\widehat\Omega}_{\max} &= O_P\left(d\left(\sqrt{\frac{\log d}{n+k}} +\frac{n+k\sqrt{\log d}+k^{3/4}\log^{3/4}d}{n+k} r_{\btheta} + \frac{nk}{n+k}r_{\btheta}^2\right) + \sqrt{\frac{d\log d}n}\right).
        \end{align*}
	Moreover, if $n\gg d\log^{5+\kappa} d$, $n+k\gg d^2\log^{5+\kappa} d$, and
	$$\left\|\btheta-\thetas\right\|_1\ll\min\left\{\frac{n+k}{d\left(n+k\sqrt{\log d}+k^{3/4}\log^{3/4}d\right)\log^{2+\kappa}d},\frac1{ \sqrt{d}\log^{1+\kappa}d}\sqrt{\frac1n+\frac1k}\right\},$$
	for some $\kappa>0$, then there exists some $u>0$ such that \eqref{eqn:ut} holds.
\end{lemma}

\begin{proof}[Lemma \ref{lem:gbd_ld_glm}]
	Note by the triangle inequality that
	$$\matrixnorm{\widetilde\Omega-\widehat\Omega}_{\max}\leq\matrixnorm{\widetilde\Omega-\Omega_0}_{\max} + \matrixnorm{\widehat\Omega-\Omega_0}_{\max},$$
	where $\Omega_0$ is defined as in \eqref{eqn:o0}.  By the proof of Lemma~\ref{lem:gbd0_ld_glm}, we have that
	$$\matrixnorm{\widehat\Omega-\Omega_0}_{\max} = O_P\left(\sqrt{\frac{\log d}N} + \frac{\log^2(dN)\log d}N\right).$$
	Next, we bound $\matrixnorm{\widetilde\Omega-\Omega_0}_{\max}$ using the same argument as in the proof of Lemma~\ref{lem:gbd0_ld_glm}.  By the triangle inequality, we have that
	\begin{align*}
	&\matrixnorm{\widetilde\Omega-\Omega_0}_{\max}\\
	&=\Bigg|\!\Bigg|\!\Bigg|\widetilde\Theta(\btheta)\frac1{n+k-1}\Bigg(\sum_{i=1}^n\left(\nabla\cL(\theta;Z_{i1})-\nabla\cL_N(\btheta)\right)\left(\nabla\cL(\theta;Z_{i1})-\nabla\cL_N(\btheta)\right)^\top \\
	&\quad+\sum_{j=2}^k n\left(\nabla\cL_j(\btheta)-\nabla\cL_N(\btheta)\right)\left(\nabla\cL_j(\btheta)-\nabla\cL_N(\btheta)\right)^\top\Bigg)\widetilde\Theta(\btheta)^\top-\Theta\Ee\left[\nabla\cL(\thetas;Z)\nabla\cL(\thetas;Z)^\top\right]\Theta\Bigg|\!\Bigg|\!\Bigg|_{\max} \\
	&\leq \Bigg|\!\Bigg|\!\Bigg|\widetilde\Theta(\btheta)\Bigg(\frac1{n+k-1}\Bigg(\sum_{i=1}^n\left(\nabla\cL(\theta;Z_{i1})-\nabla\cL_N(\btheta)\right)\left(\nabla\cL(\theta;Z_{i1})-\nabla\cL_N(\btheta)\right)^\top \\
	&\quad+\sum_{j=2}^k n\left(\nabla\cL_j(\btheta)-\nabla\cL_N(\btheta)\right)\left(\nabla\cL_j(\btheta)-\nabla\cL_N(\btheta)\right)^\top\Bigg)-\Ee\left[\nabla\cL(\thetas;Z)\nabla\cL(\thetas;Z)^\top\right]\Bigg)\widetilde\Theta(\btheta)^\top\Bigg|\!\Bigg|\!\Bigg|_{\max} \\
	&\quad+ \matrixnorm{\widetilde\Theta(\btheta)\Ee\left[\nabla\cL(\thetas;Z)\nabla\cL(\thetas;Z)^\top\right]\widetilde\Theta(\btheta)^\top-\Theta\Ee\left[\nabla\cL(\thetas;Z)\nabla\cL(\thetas;Z)^\top\right]\Theta}_{\max} \\
	&\defn I_1'(\btheta) + I_2.
	\end{align*}
	We have shown in the proof of Lemma~\ref{lem:gbd0_ld_glm} that
	$$I_2 = O_P\left(\sqrt{\frac{d\log d}n}+r_{\btheta}\right),$$
	provided that $n\gtrsim d\log d$ and $r_{\btheta}\lesssim1$.  To bound $I_1'(\btheta)$, we note that
	\begin{align*}
	I_1'(\btheta) &\leq \matrixnorm{\widetilde\Theta(\btheta)}_\infty^2 \Bigg|\!\Bigg|\!\Bigg|\frac1{n+k-1}\Bigg(\sum_{i=1}^n\left(\nabla\cL(\btheta;Z_{i1})-\nabla\cL_N(\btheta)\right)\left(\nabla\cL(\btheta;Z_{i1})-\nabla\cL_N(\btheta)\right)^\top \\
	&\quad+\sum_{j=2}^k n\left(\nabla\cL_j(\btheta)-\nabla\cL_N(\btheta)\right)\left(\nabla\cL_j(\btheta)-\nabla\cL_N(\btheta)\right)^\top\Bigg) -\Ee\left[\nabla\cL(\thetas;Z)\nabla\cL(\thetas;Z)^\top\right]\Bigg|\!\Bigg|\!\Bigg|_{\max}.
	\end{align*}
	By Lemma~\ref{lem:hes_ld_glm}, we have that
	$$\matrixnorm{\widetilde\Theta(\btheta)}_\infty=O_P\left( \sqrt{d}\right).$$
	Then, applying Lemma~\ref{lem:vcov_reg_glm}, we have that
	\begin{align*}
	I_1'(\btheta) = O_P\left(d\left(\sqrt{\frac{\log d}{n+k}} + \frac{n+k\sqrt{\log d}+k^{3/4}\log^{3/4}d}{n+k} r_{\btheta} + \frac{nk}{n+k}r_{\btheta}^2\right)\right),
	\end{align*}
	under Assumptions~\ref{as:smth_glm}--\ref{as:hes_glm}, provided that $\left\|\btheta-\thetas\right\|_1=O_P(r_{\btheta})$, $r_{\btheta}\lesssim 1$, and $n+k\gtrsim\log d$.
	
	Putting all the preceding bounds together, we obtain that
	\begin{align*}
        \matrixnorm{\widetilde\Omega-\Omega_0}_{\max} &= O_P\left(d\left(\sqrt{\frac{\log d}{n+k}} +\frac{n+k\sqrt{\log d}+k^{3/4}\log^{3/4}d}{n+k} r_{\btheta} + \frac{nk}{n+k}r_{\btheta}^2\right) + \sqrt{\frac{d\log d}n}\right),
        \end{align*}
        and finally the first result in the lemma.
        Choosing
	$$u= \left(d\sqrt{\frac{\log d}{n+k}} + \frac{n+k\sqrt{\log d}+k^{3/4}\log^{3/4}d}{n+k} d r_{\btheta} + \frac{nkd}{n+k}r_{\btheta}^2 + \sqrt{\frac{d\log d}n}\right)^{1-\kappa},$$
	with any $\kappa>0$, we deduce that
	$P\left(\matrixnorm{\widetilde\Omega-\widehat\Omega}_{\max}>u\right)=o(1)$.
	We also have that
	$$u^{1/3}\left(1\vee\log\frac du\right)^{2/3},\quad\text{if}\quad\left(d\sqrt{\frac{\log d}{n+k}} + \frac{n+k\sqrt{\log d}+k^{3/4}\log^{3/4}d}{n+k} d r_{\btheta} + \frac{nkd}{n+k}r_{\btheta}^2 + \sqrt{\frac{d\log d}n}\right) \log^{2+\kappa} d=o(1).$$
	We complete the proof by simplifying the conditions.
\end{proof}

\section{Technical Lemmas}

\begin{lemma}\label{lem:kgrad_decomp}
	For any $\theta$, we have that
	\begin{align*}
	&\matrixnorm{\frac1k\sum_{j=1}^k n\left(\nabla\cL_j(\theta)-\nabla\cL_N(\theta)\right)\left(\nabla\cL_j(\theta)-\nabla\cL_N(\theta)\right)^\top-\Ee\left[\nabla\cL(\thetas;Z)\nabla\cL(\thetas;Z)^\top\right]}_{\max} \leq U_1(\theta) + U_2 + U_3(\theta),
	\end{align*}
	\begin{align*}
	&\text{where}\quad U_1(\theta)\defn\matrixnorm{\frac1k\sum_{j=1}^k n\left(\nabla\cL_j(\theta)-\nabla\cLs(\theta)\right)\left(\nabla\cL_j(\theta)-\nabla\cLs(\theta)\right)^\top-n\nabla\cL_j(\thetas)\nabla\cL_j(\thetas)^\top}_{\max}, \\
	&U_2\defn\matrixnorm{\frac1k\sum_{j=1}^k n\nabla\cL_j(\thetas)\nabla\cL_j(\thetas)^\top-\Ee\left[\nabla\cL(\thetas;Z)\nabla\cL(\thetas;Z)^\top\right]}_{\max},
	\quad\text{and}\quad
	U_3(\theta)\defn n\left\|\nabla\cL_N(\theta)-\nabla\cLs(\theta)\right\|_\infty^2.
	\end{align*}
\end{lemma}

\begin{proof}[Lemma \ref{lem:kgrad_decomp}]
		We write $\nabla\cL_j(\theta)-\nabla\cLs(\theta)$ as $\left(\nabla\cL_j(\theta)-\nabla\cL_N(\theta)\right)+\left(\nabla\cL_N(\theta)-\nabla\cLs(\theta)\right)$, and have that
	\begin{align*}
	&\sum_{j=1}^k n\left(\nabla\cL_j(\theta)-\nabla\cLs(\theta)\right)\left(\nabla\cL_j(\theta)-\nabla\cLs(\theta)\right)^\top \\
	&= \sum_{j=1}^k n\left(\nabla\cL_j(\theta)-\nabla\cL_N(\theta)\right)\left(\nabla\cL_j(\theta)-\nabla\cL_N(\theta)\right)^\top + nk\left(\nabla\cL_N(\theta)-\nabla\cLs(\theta)\right)\left(\nabla\cL_N(\theta)-\nabla\cLs(\theta)\right)^\top \\
	&\quad + n\left(\nabla\cL_N(\theta)-\nabla\cLs(\theta)\right) \sum_{j=1}^k \left(\nabla\cL_j(\theta)-\nabla\cL_N(\theta)\right)^\top + n\sum_{j=1}^k \left(\nabla\cL_j(\theta)-\nabla\cL_N(\theta)\right) \left(\nabla\cL_N(\theta)-\nabla\cLs(\theta)\right)^\top \\
	&= \sum_{j=1}^k n\left(\nabla\cL_j(\theta)-\nabla\cL_N(\theta)\right)\left(\nabla\cL_j(\theta)-\nabla\cL_N(\theta)\right)^\top + nk\left(\nabla\cL_N(\theta)-\nabla\cLs(\theta)\right)\left(\nabla\cL_N(\theta)-\nabla\cLs(\theta)\right)^\top,
	\end{align*}
	where we use $\nabla\cL_N(\theta)=\frac1k\sum_{j=1}^k\nabla\cL_j(\theta)$ in the last equality.  Then, we have that
	\begin{align*}
	&\sum_{j=1}^k n\left(\nabla\cL_j(\theta)-\nabla\cL_N(\theta)\right)\left(\nabla\cL_j(\theta)-\nabla\cL_N(\theta)\right)^\top \\
	&= \sum_{j=1}^k n\left(\nabla\cL_j(\theta)-\nabla\cLs(\theta)\right)\left(\nabla\cL_j(\theta)-\nabla\cLs(\theta)\right)^\top - nk\left(\nabla\cL_N(\theta)-\nabla\cLs(\theta)\right)\left(\nabla\cL_N(\theta)-\nabla\cLs(\theta)\right)^\top,
	\end{align*}
	and by the triangle inequality,
	\begin{align*}
	&\matrixnorm{\frac1k\sum_{j=1}^k n\left(\nabla\cL_j(\theta)-\nabla\cL_N(\theta)\right)\left(\nabla\cL_j(\theta)-\nabla\cL_N(\theta)\right)^\top-\Ee\left[\nabla\cL(\thetas;Z)\nabla\cL(\thetas;Z)^\top\right]}_{\max} \\
	&\leq U_1(\theta) + U_2 + n\matrixnorm{\left(\nabla\cL_N(\theta)-\nabla\cLs(\theta)\right)\left(\nabla\cL_N(\theta)-\nabla\cLs(\theta)\right)^\top}_{\max}.
	\end{align*}
	By the fact that $\matrixnorm{aa^\top}_{\max}=\|a\|_\infty^2$ for any vector $a$, we have that
	$\matrixnorm{\left(\nabla\cL_N(\theta)-\nabla\cLs(\theta)\right)\left(\nabla\cL_N(\theta)-\nabla\cLs(\theta)\right)^\top}_{\max}=n^{-1}U_3(\theta)$.
\end{proof}

\begin{lemma}\label{lem:vcov0_reg}
	In linear model, under Assumptions~\ref{as:design}~and~\ref{as:noise}, provided that $\left\|\btheta-\thetas\right\|_1=O_P(r_{\btheta})$, we have that
	\begin{align*}
	&\matrixnorm{\frac1k\sum_{j=1}^k n\left(\nabla\cL_j(\btheta)-\nabla\cL_N(\btheta)\right)\left(\nabla\cL_j(\btheta)-\nabla\cL_N(\btheta)\right)^\top-\Ee\left[\nabla\cL(\thetas;Z)\nabla\cL(\thetas;Z)^\top\right]}_{\max} \\
	&= O_P\Bigg(\sqrt{\frac{\log d}k} + \frac{\log^2(dk)\log d}k + \Bigg(1+\left(\frac{\log d}k\right)^{1/4} + \sqrt{\frac{\log^2(dk)\log d}k}\Bigg)\sqrt{\log(kd)} r_{\btheta} \\
	&\quad + \left(n + \sqrt{\frac{n\log d}k} + \log(kd)\right) r_{\btheta}^2\Bigg).
	\end{align*}
\end{lemma}

\begin{proof}[Lemma \ref{lem:vcov0_reg}]
	By Lemma~\ref{lem:kgrad_decomp}, it suffices to bound $U_1(\btheta)$, $U_2$, and $U_3(\btheta)$.  We begin by bounding $U_2$.  In linear model, we have that
	\begin{align*}
	U_2&=\matrixnorm{\frac1k\sum_{j=1}^k n \left(\frac{X_j^\top e_j}n\right)\left(\frac{X_j^\top e_j}n\right)^\top-\sigma^2\Sigma}_{\max}.
	\end{align*}
	Note that
	$$\Ee\left[\left(\frac{(X_j^\top e_j)_l}{\sqrt n}\right)^2\right]=\Ee\left[\frac{\sum_{i=1}^n X_{ij,l}^2 e_{ij}^2}n\right]=\sigma^2\Sigma_{l,l}$$
	is bounded away from zero, under Assumptions~\ref{as:design}~and~\ref{as:noise}.  Also, using same argument for obtaining \eqref{eqn:bern2}, we have that 
	$$P\left(\left|\frac{(X_j^\top e_j)_l}{\sqrt n}\right| > t\right)\leq 2\exp\left(-c\left(\frac{t^2}{\Sigma_{l,l}\sigma^2}\wedge\frac{t\sqrt n}{\sqrt{\Sigma_{l,l}}\sigma}\right)\right)\leq C\exp\left(-c't\right),$$
	for some positive constants $c$, $c'$, and $C$, that is, $(X_j^\top e_j)_l/\sqrt n$ is sub-exponential with $O(1)$ $\psi_1$-norm for each $(j,l)$.  Then, by the proof of Corollary 3.1 of \cite{chernozhukov2013gaussian}, we have that
	\begin{align*}
	\Ee[U_2]&=\Ee\left[\matrixnorm{\frac1k\sum_{j=1}^k \left(\frac{X_j^\top e_j}{\sqrt n}\right)\left(\frac{X_j^\top e_j}{\sqrt n}\right)^\top-\sigma^2\Sigma}_{\max}\right] \lesssim \sqrt{\frac{\log d}k} + \frac{\log^2(dk)\log d}k,
	\end{align*}
	which implies by Markov's inequality that
	$$U_2 = O_P\left(\sqrt{\frac{\log d}k} + \frac{\log^2(dk)\log d}k\right).$$
	
	Next, we bound $U_3(\btheta)$.  By the triangle inequality and the fact that for any matrix $A$ and vector $a$ with compatible dimensions, $\|Aa\|_\infty\leq\matrixnorm{A}_{\max}\|a\|_1$, we have that
	\begin{align*}
	\left\|\nabla\cL_N(\btheta)-\nabla\cLs(\btheta)\right\|_\infty
	&\leq \left\|\nabla\cL_N(\btheta)-\nabla\cL_N(\thetas)\right\|_\infty + \left\|\nabla\cL_N(\thetas)\right\|_\infty + \left\|\nabla\cLs(\btheta)\right\|_\infty \\
	&= \left\|\frac{X_N^\top(X_N\btheta-y_N)}N-\frac{X_N^\top(X_N\thetas-y_N)}N\right\|_\infty + \left\|\frac{X_N^\top(X_N\thetas-y_N)}N\right\|_\infty + \left\|\Sigma(\btheta-\thetas)\right\|_\infty \\
	&= \left\|\frac{X_N^\top X_N}N(\btheta-\thetas)\right\|_\infty + \left\|\frac{X_N^\top e_N}N\right\|_\infty + \left\|\Sigma(\btheta-\thetas)\right\|_\infty \\
	&\leq \matrixnorm{\frac{X_N^\top X_N}N}_{\max} \left\|\btheta-\thetas\right\|_1 + \left\|\frac{X_N^\top e_N}N\right\|_\infty + \matrixnorm{\Sigma}_{\max} \left\|\btheta-\thetas\right\|_1 \\
	&\lesssim \matrixnorm{\frac{X_N^\top X_N}N-\Sigma}_{\max} \left\|\btheta-\thetas\right\|_1 + \left\|\frac{X_N^\top e_N}N\right\|_\infty + \matrixnorm{\Sigma}_{\max} \left\|\btheta-\thetas\right\|_1.
	\end{align*}
	Under Assumption~\ref{as:design}, each $x_{ij,l}$ is sub-Gaussian, and therefore, the product $x_{ij,l}x_{ij,l'}$ of any two is sub-exponential.  By Bernstein's inequality, we have that for any $\delta\in(0,1)$,
	$$P\left(\left|\frac{(X_N^\top X_N)_{l,l'}}N-\Sigma_{l,l'}\right| > |\Sigma_{l,l'}|\left(\frac{\log\frac{2d^2}\delta}{cN}\vee\sqrt{\frac{\log\frac{2d^2}\delta}{cN}}\right)\right)\leq \frac\delta{d^2},$$
	for some constant $c>0$.  Then, by the union bound, we have that
	\begin{align}
	P\left(\matrixnorm{\frac{X_N^\top X_N}N-\Sigma}_{\max} > \matrixnorm{\Sigma}_{\max}\left(\frac{\log\frac{2d^2}\delta}{cN}\vee\sqrt{\frac{\log\frac{2d^2}\delta}{cN}}\right)\right)\leq \delta. \label{eqn:bern1}
	\end{align}
	Similarly, we have that
	\begin{align}
	P\left(\matrixnorm{\frac{X_1^\top X_1}n-\Sigma}_{\max} > \matrixnorm{\Sigma}_{\max}\left(\frac{\log\frac{2d^2}\delta}{cn}\vee\sqrt{\frac{\log\frac{2d^2}\delta}{cn}}\right)\right)\leq \delta. \label{eqn:bern3}
	\end{align}
	By \eqref{eqn:bern1} and \eqref{eqn:bern2}, we have that
	\begin{align*}
	\matrixnorm{\frac{X_N^\top X_N}N-\Sigma}_{\max} \leq \matrixnorm{\Sigma}_{\max}\left(\frac{\log\frac{2d^2}\delta}{cN}\vee\sqrt{\frac{\log\frac{2d^2}\delta}{cN}}\right)=O_P\left(\sqrt{\frac{\log d}N}\right), \quad\text{and} \\
	\left\|\frac{X_N^\top e_N}N\right\|_\infty \leq \max_l \sqrt{\Sigma_{l,l}}\sigma\left(\frac{\log\frac{2d}\delta}{cN}\vee\sqrt{\frac{\log\frac{2d}\delta}{cN}}\right)=O_P\left(\sqrt{\frac{\log d}N}\right),
	\end{align*}
	where $\max_l \sqrt{\Sigma_{l,l}}\leq\matrixnorm{\Sigma}_{\max}=O(1)$ under Assumption~\ref{as:design}.  Then, assuming that $\left\|\btheta-\thetas\right\|_1=O_P(r_{\btheta})$, we have that
	\begin{align*}
	\left\|\nabla\cL_N(\btheta)-\nabla\cLs(\btheta)\right\|_\infty
	&= \left(O(1)+O_P\left(\sqrt{\frac{\log d}N}\right)\right) O_P(r_{\btheta}) + O_P\left(\sqrt{\frac{\log d}N}\right) \\
	&= O_P\left(\left(1+\sqrt{\frac{\log d}N}\right) r_{\btheta}+\sqrt{\frac{\log d}N}\right),
	\end{align*}
	and then,
	$$U_3(\btheta) =  O_P\left(\left(1+\sqrt{\frac{\log d}N}\right) n r_{\btheta}^2 + \frac{\log d}k\right).$$
	
	Lastly, we bound $U_1(\btheta)$.  We write $\nabla\cL_j(\btheta)-\nabla\cLs(\btheta)$ as $\left(\nabla\cL_j(\btheta)-\nabla\cLs(\btheta)-\nabla\cL_j(\thetas)\right)+\nabla\cL_j(\thetas)$, and obtain by the triangle inequality that
        \begin{align*}
        U_1(\btheta) &\leq \matrixnorm{\frac1k\sum_{j=1}^k n\left(\nabla\cL_j(\btheta)-\nabla\cLs(\btheta)-\nabla\cL_j(\thetas)\right)\left(\nabla\cL_j(\btheta)-\nabla\cLs(\btheta)-\nabla\cL_j(\thetas)\right)^\top}_{\max} \\
        &\quad+ \matrixnorm{\frac1k\sum_{j=1}^k n\nabla\cL_j(\thetas)\left(\nabla\cL_j(\btheta)-\nabla\cLs(\btheta)-\nabla\cL_j(\thetas)\right)^\top}_{\max} \\
        &\quad+ \matrixnorm{\frac1k\sum_{j=1}^k n\left(\nabla\cL_j(\btheta)-\nabla\cLs(\btheta)-\nabla\cL_j(\thetas)\right)\nabla\cL_j(\thetas)^\top}_{\max} \\
        &= \matrixnorm{\frac1k\sum_{j=1}^k n\left(\nabla\cL_j(\btheta)-\nabla\cLs(\btheta)-\nabla\cL_j(\thetas)\right)\left(\nabla\cL_j(\btheta)-\nabla\cLs(\btheta)-\nabla\cL_j(\thetas)\right)^\top}_{\max} \\
        &\quad+ 2\matrixnorm{\frac1k\sum_{j=1}^k n\nabla\cL_j(\thetas)\left(\nabla\cL_j(\btheta)-\nabla\cLs(\btheta)-\nabla\cL_j(\thetas)\right)^\top}_{\max} \\
        &\defn U_{11}(\btheta) + 2 U_{12}(\btheta).
        \end{align*}
        To bound $U_{12}(\btheta)$, we first define an inner product $\langle A,B\rangle=\matrixnorm{AB^\top}_{\max}$ for any $A,B\in\R^{d\times k}$, the validity of which is easy to check.  We then apply Cauchy-Schwarz inequality on $\langle A,B\rangle$ with
        \begin{align*}
        A&=\sqrt{\frac nk}
        \begin{bmatrix}
	\nabla\cL_1(\thetas) \quad \dots \quad \nabla\cL_k(\thetas)
	\end{bmatrix} \quad\text{and}\\
	    B&=\sqrt{\frac nk}
        \begin{bmatrix}
	\nabla\cL_1(\btheta)-\nabla\cLs(\btheta)-\nabla\cL_1(\thetas) \quad \dots \quad \nabla\cL_k(\btheta)-\nabla\cLs(\btheta)-\nabla\cL_k(\thetas))
	\end{bmatrix}
	    \end{align*}
        and obtain that
         \begin{align*}
        U_{12}(\btheta) &\leq \matrixnorm{\frac1k\sum_{j=1}^k n\nabla\cL_j(\thetas)\nabla\cL_j(\thetas)^\top}_{\max}^{1/2} \\
        &\quad \cdot\matrixnorm{\frac1k\sum_{j=1}^k n\left(\nabla\cL_j(\btheta)-\nabla\cLs(\btheta)-\nabla\cL_j(\thetas)\right)\left(\nabla\cL_j(\btheta)-\nabla\cLs(\btheta)-\nabla\cL_j(\thetas)\right)^\top}_{\max}^{1/2} \\
        &= \matrixnorm{\frac1k\sum_{j=1}^k n\nabla\cL_j(\thetas)\nabla\cL_j(\thetas)^\top}_{\max}^{1/2} U_{11}(\btheta)^{1/2}.
        \end{align*}
        By the triangle inequality, we have that
        \begin{align*}
        &\matrixnorm{\frac1k\sum_{j=1}^k n\nabla\cL_j(\thetas)\nabla\cL_j(\thetas)^\top}_{\max} \\
        &\leq \matrixnorm{\frac1k\sum_{j=1}^k n\nabla\cL_j(\thetas)\nabla\cL_j(\thetas)^\top-\Ee\left[\nabla\cL(\thetas;Z)\nabla\cL(\thetas;Z)^\top\right]}_{\max} + \matrixnorm{\Ee\left[\nabla\cL(\thetas;Z)\nabla\cL(\thetas;Z)^\top\right]}_{\max} \\
        &= U_2 + \sigma^2\matrixnorm{\Sigma}_{\max} = O_P\left(1+\sqrt{\frac{\log d}k} + \frac{\log^2(dk)\log d}k\right).
        \end{align*}
        It remains to bound $U_{11}(\btheta)$.  Note that
        \begin{align*}
        \nabla\cL_j(\btheta)-\nabla\cLs(\btheta)-\nabla\cL_j(\thetas)
        &= \frac{X_j^\top(X_j\btheta-y_j)}n - \Sigma(\btheta-\thetas) + \frac{X_j^\top(X_j\thetas-y_j)}n = \left(\frac{X_j^\top X_j}n-\Sigma\right)(\btheta-\thetas).
        \end{align*}
        Then, we have that
        \begin{align*}
        U_{11}(\btheta) &= \matrixnorm{\frac1k\sum_{j=1}^k n\left(\frac{X_j^\top X_j}n-\Sigma\right)(\btheta-\thetas)(\btheta-\thetas)^\top\left(\frac{X_j^\top X_j}n-\Sigma\right)}_{\max} \\
        &\leq \frac1k\sum_{j=1}^k n\matrixnorm{\left(\frac{X_j^\top X_j}n-\Sigma\right)(\btheta-\thetas)(\btheta-\thetas)^\top\left(\frac{X_j^\top X_j}n-\Sigma\right)}_{\max} \\
        &= \frac1k\sum_{j=1}^k n\matrixnorm{\left(\frac{X_j^\top X_j}n-\Sigma\right)(\btheta-\thetas)}_\infty^2 \leq \frac1k\sum_{j=1}^k n\matrixnorm{\frac{X_j^\top X_j}n-\Sigma}_{\max}^2 \left\|\btheta-\thetas\right\|_1^2,
        \end{align*}
        where we use the triangle inequality and the fact that $\matrixnorm{aa^\top}_{\max}=\|a\|_\infty^2$ for any vector $a$, and $\|Aa\|_\infty\leq\matrixnorm{A}_{\max}\|a\|_1$ for any matrix $A$ and vector $a$ with compatible dimensions.  By \eqref{eqn:bern3}, we have that
        $$P\left(\matrixnorm{\frac{X_j^\top X_j}n-\Sigma}_{\max} > \matrixnorm{\Sigma}_{\max}\left(\frac{\log\frac{2kd^2}\delta}{cn}\vee\sqrt{\frac{\log\frac{2kd^2}\delta}{cn}}\right)\right)\leq \frac\delta k,$$
        which implies by the union bound that
        $$\max_j\matrixnorm{\frac{X_j^\top X_j}n-\Sigma}_{\max} = O_P\left(\sqrt{\frac{\log(kd)}n}\right).$$
        Putting all the preceding bounds together, we obtain that
        $$U_{11}(\btheta) = O_P\left(\log(kd) r_{\btheta}^2\right),$$
        $$U_{12}(\btheta) = O_P\left(\left(1+\left(\frac{\log d}k\right)^{1/4} + \sqrt{\frac{\log^2(dk)\log d}k}\right)\sqrt{\log(kd)} r_{\btheta}\right),$$
        $$U_1(\btheta) = O_P\left(\left(1+\left(\frac{\log d}k\right)^{1/4} + \sqrt{\frac{\log^2(dk)\log d}k}\right)\sqrt{\log(kd)} r_{\btheta} + \log(kd) r_{\btheta}^2\right),$$
        and finally the bound in the lemma.
\end{proof}

\begin{lemma}\label{lem:nk1grad_decomp}
	For any $\theta$, we have that
	\begin{align*}
	&\Bigg|\!\Bigg|\!\Bigg|\frac1{n+k-1}\Bigg(\sum_{i=1}^n\left(\nabla\cL(\theta;Z_{i1})-\nabla\cL_N(\theta)\right)\left(\nabla\cL(\theta;Z_{i1})-\nabla\cL_N(\theta)\right)^\top \\
	&\quad+\sum_{j=2}^k n\left(\nabla\cL_j(\theta)-\nabla\cL_N(\theta)\right)\left(\nabla\cL_j(\theta)-\nabla\cL_N(\theta)\right)^\top\Bigg) -\Ee\left[\nabla\cL(\thetas;Z)\nabla\cL(\thetas;Z)^\top\right]\Bigg|\!\Bigg|\!\Bigg|_{\max} \\
	&\leq V_1(\theta) + V_1'(\theta) + V_2 + V_2' + V_3(\theta),
	\end{align*}
	$$\text{where}\quad V_1(\theta)\defn\frac{k-1}{n+k-1}\matrixnorm{\frac1{k-1}\sum_{j=2}^k n\left(\nabla\cL_j(\theta)-\nabla\cLs(\theta)\right)\left(\nabla\cL_j(\theta)-\nabla\cLs(\theta)\right)^\top-n\nabla\cL_j(\thetas)\nabla\cL_j(\thetas)^\top}_{\max},$$
	$$V_1'(\theta)\defn\frac n{n+k-1}\matrixnorm{\frac1n\sum_{i=1}^n\left(\nabla\cL(\theta;Z_{i1})-\nabla\cLs(\theta)\right)\left(\nabla\cL(\theta;Z_{i1})-\nabla\cLs(\theta)\right)^\top-\nabla\cL(\thetas;Z_{i1})\nabla\cL(\thetas;Z_{i1})^\top}_{\max},$$
	$$V_2\defn\frac{k-1}{n+k-1}\matrixnorm{\frac1{k-1}\sum_{j=2}^k n\nabla\cL_j(\thetas)\nabla\cL_j(\thetas)^\top-\Ee\left[\nabla\cL(\thetas;Z)\nabla\cL(\thetas;Z)^\top\right]}_{\max},$$
	$$V_2'\defn\frac n{n+k-1}\matrixnorm{\frac1n\sum_{i=1}^n\nabla\cL(\thetas;Z_{i1})\nabla\cL(\thetas;Z_{i1})^\top-\Ee\left[\nabla\cL(\thetas;Z)\nabla\cL(\thetas;Z)^\top\right]}_{\max},\quad\text{and}$$
	$$V_3(\theta)\defn \frac{nk}{n+k-1} \left\|\nabla\cL_N(\theta)-\nabla\cLs(\theta)\right\|_\infty^2.$$
\end{lemma}

\begin{proof}[Lemma \ref{lem:nk1grad_decomp}]
	We write $\nabla\cL(\theta;Z_{i1})-\nabla\cLs(\theta)$ as $\left(\nabla\cL(\theta;Z_{i1})-\nabla\cL_N(\theta)\right)+\left(\nabla\cL_N(\theta)-\nabla\cLs(\theta)\right)$ and $\nabla\cL_j(\theta)-\nabla\cLs(\theta)$ as $\left(\nabla\cL_j(\theta)-\nabla\cL_N(\theta)\right)+\left(\nabla\cL_N(\theta)-\nabla\cLs(\theta)\right)$, and have that
	\begin{align*}
	&\sum_{i=1}^n\left(\nabla\cL(\theta;Z_{i1})-\nabla\cLs(\theta)\right)\left(\nabla\cL(\theta;Z_{i1})-\nabla\cLs(\theta)\right)^\top \notag \\
	&= \sum_{i=1}^n\left(\nabla\cL(\theta;Z_{i1})-\nabla\cL_N(\theta)\right)\left(\nabla\cL(\theta;Z_{i1})-\nabla\cL_N(\theta)\right)^\top + n\left(\nabla\cL_N(\theta)-\nabla\cLs(\theta)\right)\left(\nabla\cL_N(\theta)-\nabla\cLs(\theta)\right)^\top \notag \\
	&\quad + \left(\nabla\cL_N(\theta)-\nabla\cLs(\theta)\right) \sum_{i=1}^n\left(\nabla\cL(\theta;Z_{i1})-\nabla\cL_N(\theta)\right)^\top + \sum_{i=1}^n\left(\nabla\cL(\theta;Z_{i1})-\nabla\cL_N(\theta)\right) \left(\nabla\cL_N(\theta)-\nabla\cLs(\theta)\right)^\top \notag \\
	&= \sum_{i=1}^n\left(\nabla\cL(\theta;Z_{i1})-\nabla\cL_N(\theta)\right)\left(\nabla\cL(\theta;Z_{i1})-\nabla\cL_N(\theta)\right)^\top + n\left(\nabla\cL_N(\theta)-\nabla\cLs(\theta)\right)\left(\nabla\cL_N(\theta)-\nabla\cLs(\theta)\right)^\top \notag \\
	&\quad + n\left(\nabla\cL_N(\theta)-\nabla\cLs(\theta)\right) \left(\nabla\cL_1(\theta)-\nabla\cL_N(\theta)\right)^\top + n\left(\nabla\cL_1(\theta)-\nabla\cL_N(\theta)\right) \left(\nabla\cL_N(\theta)-\nabla\cLs(\theta)\right)^\top,
	\end{align*}
	and
	\begin{align*}
	&\sum_{j=2}^k n\left(\nabla\cL_j(\theta)-\nabla\cLs(\theta)\right)\left(\nabla\cL_j(\theta)-\nabla\cLs(\theta)\right)^\top \notag \\
	&= \sum_{j=2}^k n\left(\nabla\cL_j(\theta)-\nabla\cL_N(\theta)\right)\left(\nabla\cL_j(\theta)-\nabla\cL_N(\theta)\right)^\top + n(k-1)\left(\nabla\cL_N(\theta)-\nabla\cLs(\theta)\right)\left(\nabla\cL_N(\theta)-\nabla\cLs(\theta)\right)^\top \notag \\
	&\quad + n\left(\nabla\cL_N(\theta)-\nabla\cLs(\theta)\right) \sum_{j=2}^k \left(\nabla\cL_j(\theta)-\nabla\cL_N(\theta)\right)^\top + n\sum_{j=2}^k \left(\nabla\cL_j(\theta)-\nabla\cL_N(\theta)\right) \left(\nabla\cL_N(\theta)-\nabla\cLs(\theta)\right)^\top.
	\end{align*}
	Adding up the two preceding equations, we obtain that
	\begin{align*}
	&\sum_{i=1}^n\left(\nabla\cL(\theta;Z_{i1})-\nabla\cLs(\theta)\right)\left(\nabla\cL(\theta;Z_{i1})-\nabla\cLs(\theta)\right)^\top + \sum_{j=2}^k n\left(\nabla\cL_j(\theta)-\nabla\cLs(\theta)\right)\left(\nabla\cL_j(\theta)-\nabla\cLs(\theta)\right)^\top \\
	&= \sum_{i=1}^n\left(\nabla\cL(\theta;Z_{i1})-\nabla\cL_N(\theta)\right)\left(\nabla\cL(\theta;Z_{i1})-\nabla\cL_N(\theta)\right)^\top + \sum_{j=2}^k n\left(\nabla\cL_j(\theta)-\nabla\cL_N(\theta)\right)\left(\nabla\cL_j(\theta)-\nabla\cL_N(\theta)\right)^\top \\
	&\quad + nk\left(\nabla\cL_N(\theta)-\nabla\cLs(\theta)\right)\left(\nabla\cL_N(\theta)-\nabla\cLs(\theta)\right)^\top \\
	&\quad + n\left(\nabla\cL_N(\theta)-\nabla\cLs(\theta)\right) \sum_{j=1}^k \left(\nabla\cL_j(\theta)-\nabla\cL_N(\theta)\right)^\top + n\sum_{j=1}^k \left(\nabla\cL_j(\theta)-\nabla\cL_N(\theta)\right) \left(\nabla\cL_N(\theta)-\nabla\cLs(\theta)\right)^\top \\
	&= \sum_{i=1}^n\left(\nabla\cL(\theta;Z_{i1})-\nabla\cL_N(\theta)\right)\left(\nabla\cL(\theta;Z_{i1})-\nabla\cL_N(\theta)\right)^\top + \sum_{j=2}^k n\left(\nabla\cL_j(\theta)-\nabla\cL_N(\theta)\right)\left(\nabla\cL_j(\theta)-\nabla\cL_N(\theta)\right)^\top \\
	&\quad + nk\left(\nabla\cL_N(\theta)-\nabla\cLs(\theta)\right)\left(\nabla\cL_N(\theta)-\nabla\cLs(\theta)\right)^\top,
	\end{align*}
	where we use $\nabla\cL_N(\theta)=\frac1k\sum_{j=1}^k\nabla\cL_j(\theta)$ in the last equality.  Then, we have that
	\begin{align*}
	&\sum_{i=1}^n\left(\nabla\cL(\theta;Z_{i1})-\nabla\cL_N(\theta)\right)\left(\nabla\cL(\theta;Z_{i1})-\nabla\cL_N(\theta)\right)^\top + \sum_{j=2}^k n\left(\nabla\cL_j(\theta)-\nabla\cL_N(\theta)\right)\left(\nabla\cL_j(\theta)-\nabla\cL_N(\theta)\right)^\top \notag \\
	&= \sum_{i=1}^n\left(\nabla\cL(\theta;Z_{i1})-\nabla\cLs(\theta)\right)\left(\nabla\cL(\theta;Z_{i1})-\nabla\cLs(\theta)\right)^\top + \sum_{j=2}^k n\left(\nabla\cL_j(\theta)-\nabla\cLs(\theta)\right)\left(\nabla\cL_j(\theta)-\nabla\cLs(\theta)\right)^\top \\
	&\quad - nk\left(\nabla\cL_N(\theta)-\nabla\cLs(\theta)\right)\left(\nabla\cL_N(\theta)-\nabla\cLs(\theta)\right)^\top,
	\end{align*}
	and by the triangle inequality,
	\begin{align*}
	&\Bigg|\!\Bigg|\!\Bigg|\frac1{n+k-1}\Bigg(\sum_{i=1}^n\left(\nabla\cL(\theta;Z_{i1})-\nabla\cL_N(\theta)\right)\left(\nabla\cL(\theta;Z_{i1})-\nabla\cL_N(\theta)\right)^\top \\
	&\quad+\sum_{j=2}^k n\left(\nabla\cL_j(\theta)-\nabla\cL_N(\theta)\right)\left(\nabla\cL_j(\theta)-\nabla\cL_N(\theta)\right)^\top\Bigg) -\Ee\left[\nabla\cL(\thetas;Z)\nabla\cL(\thetas;Z)^\top\right]\Bigg|\!\Bigg|\!\Bigg|_{\max} \\
	&\leq \frac n{n+k-1}\matrixnorm{\sum_{i=1}^n\left(\nabla\cL(\theta;Z_{i1})-\nabla\cLs(\theta)\right)\left(\nabla\cL(\theta;Z_{i1})-\nabla\cLs(\theta)\right)^\top-\Ee\left[\nabla\cL(\thetas;Z)\nabla\cL(\thetas;Z)^\top\right]}_{\max} \\
	&\quad + \frac{k-1}{n+k-1}\matrixnorm{\sum_{j=2}^k \left(\sqrt n\nabla\cL_j(\theta)-\sqrt n\nabla\cLs(\theta)\right)\left(\sqrt n\nabla\cL_j(\theta)-\sqrt n\nabla\cLs(\theta)\right)^\top-\Ee\left[\nabla\cL(\thetas;Z)\nabla\cL(\thetas;Z)^\top\right]}_{\max} \\
	&\quad + \frac{nk}{n+k-1} \matrixnorm{\left(\nabla\cL_N(\theta)-\nabla\cLs(\theta)\right)\left(\nabla\cL_N(\theta)-\nabla\cLs(\theta)\right)^\top}_{\max} \\
	&\defn A(\theta) + B(\theta) + \frac{nk}{n+k-1} \matrixnorm{\left(\nabla\cL_N(\theta)-\nabla\cLs(\theta)\right)\left(\nabla\cL_N(\theta)-\nabla\cLs(\theta)\right)^\top}_{\max}.
	\end{align*}
	By the fact that $\matrixnorm{aa^\top}_{\max}=\|a\|_\infty^2$ for any vector $a$, we have that $\matrixnorm{\left(\nabla\cL_N(\theta)-\nabla\cLs(\theta)\right)\left(\nabla\cL_N(\theta)-\nabla\cLs(\theta)\right)^\top}_{\max}=(n+k-1)(nk)^{-1}V_3(\theta)$.
	We apply the triangle inequality to further decompose $A(\theta)$ and $B(\theta)$ and obtain that $B(\theta)\leq V_1(\theta)+V_2$ and $A(\theta)\leq V_1'(\theta)+V_2'$.
	
\end{proof}

\begin{lemma}\label{lem:vcov_reg}
	In linear model, under Assumptions~\ref{as:design}~and~\ref{as:noise}, provided that $\left\|\btheta-\thetas\right\|_1=O_P(r_{\btheta})$, we have that
	\begin{align*}
	&\Bigg|\!\Bigg|\!\Bigg|\frac1{n+k-1}\Bigg(\sum_{i=1}^n\left(\nabla\cL(\btheta;Z_{i1})-\nabla\cL_N(\btheta)\right)\left(\nabla\cL(\btheta;Z_{i1})-\nabla\cL_N(\btheta)\right)^\top \\
	&\quad+\sum_{j=2}^k n\left(\nabla\cL_j(\btheta)-\nabla\cL_N(\btheta)\right)\left(\nabla\cL_j(\btheta)-\nabla\cL_N(\btheta)\right)^\top\Bigg) -\Ee\left[\nabla\cL(\thetas;Z)\nabla\cL(\thetas;Z)^\top\right]\Bigg|\!\Bigg|\!\Bigg|_{\max} \\
	&= O_P\Bigg(\sqrt{\frac{\log d}{n+k}} + \frac{\log^2(d(n+k))\log d}{n+k} + \left(\left(1+\sqrt{\frac{\log d}N}\right)\frac{nk}{n+k}+\log((n+k)d)\right)r_{\btheta}^2 \\
	&\quad + \Bigg(\sqrt{\log((n+k)d)} + \frac{\log^{1/4} d\sqrt{\log((n+k)d)}}{(n+k)^{1/4}} + \sqrt{\frac{\log^3(d(n+k))\log d}{n+k}}\Bigg) r_{\btheta}\Bigg).
	\end{align*}
\end{lemma}

\begin{proof}[Lemma \ref{lem:vcov_reg}]
	By Lemma~\ref{lem:nk1grad_decomp}, it suffices to bound $V_1(\btheta)$, $V_1'(\btheta)$, $V_2$, $V_2'$, and $V_3(\btheta)$.  By the proof of Lemma~\ref{lem:vcov0_reg}, we have that under Assumptions~\ref{as:design}~and~\ref{as:noise}, assuming that $\left\|\btheta-\thetas\right\|_1=O_P(r_{\btheta})$,
	\begin{align*}
	V_1(\btheta) &= \frac{k-1}{n+k-1} O_P\left(\left(1+\left(\frac{\log d}k\right)^{1/4} + \sqrt{\frac{\log^2(dk)\log d}k}\right)\sqrt{\log(kd)} r_{\btheta} + \log(kd) r_{\btheta}^2\right) \\
	&= O_P\left(\left(1+\left(\frac{\log d}k\right)^{1/4} + \sqrt{\frac{\log^2(dk)\log d}k}\right)\frac{k\sqrt{\log(kd)}}{n+k} r_{\btheta} + \frac{k\log(kd)}{n+k} r_{\btheta}^2\right),
	\end{align*}
	$$V_2 = \frac{k-1}{n+k-1} O_P\left(\sqrt{\frac{\log d}k} + \frac{\log^2(dk)\log d}k\right) = O_P\left(\frac{\sqrt{k\log d}}{n+k} + \frac{\log^2(dk)\log d}{n+k}\right),\quad\text{and}$$
	$$V_3(\btheta) =  \frac{nk}{n+k-1} O_P\left(\left(1+\sqrt{\frac{\log d}N}\right) r_{\btheta}^2 + \frac{\log d}N\right) = O_P\left(\left(1+\sqrt{\frac{\log d}N}\right)\frac{nk}{n+k}r_{\btheta}^2 + \frac{\log d}{n+k}\right).$$
	
	It remains to bound $V_1'(\btheta)$ and $V_2'$.To bound $V_2'$, we have that in linear model, under Assumptions~\ref{as:design}~and~\ref{as:noise},
	\begin{align*}
	V_2'&=\frac n{n+k-1}\matrixnorm{\frac1n\sum_{i=1}^n\left(x_{i1}e_{i1}\right)\left(x_{i1}e_{i1}\right)^\top-\sigma^2\Sigma}_{\max}.
	\end{align*}
	Note that $\Ee\left[\left(x_{i1}e_{i1}\right)_l^2\right]=\sigma^2\Sigma_{l,l}$ is bounded away from zero, and also that $\left(x_{i1}e_{i1}\right)_l$ is sub-exponential with O(1) $\psi_1$-norm for each $(i,l)$.  Then, by the proof of Corollary 3.1 of \cite{chernozhukov2013gaussian}, we have that
	\begin{align*}
	\Ee\left[\matrixnorm{\frac1n\sum_{i=1}^n\left(x_{i1}e_{i1}\right)\left(x_{i1}e_{i1}\right)^\top-\sigma^2\Sigma}_{\max}\right] \lesssim \sqrt{\frac{\log d}n} + \frac{\log^2(dn)\log d}n,
	\end{align*}
	which implies by Markov's inequality that
	$$V_2' = \frac n{n+k-1} O_P\left(\sqrt{\frac{\log d}n} + \frac{\log^2(dn)\log d}n\right) = O_P\left(\frac{\sqrt{n\log d}}{n+k} + \frac{\log^2(dn)\log d}{n+k}\right).$$
	
	Lastly, we bound $V_1'(\btheta)$ using the same argument as in bounding $U_1(\btheta)$ in the proof of Lemma~\ref{lem:vcov0_reg}.  We write $\nabla\cL(\theta;Z_{i1})-\nabla\cLs(\theta)$ as $\left(\nabla\cL(\theta;Z_{i1})-\nabla\cLs(\theta)-\nabla\cL(\thetas;Z_{i1})\right)+\nabla\cL(\thetas;Z_{i1})$, and obtain by the triangle inequality that
        \begin{align*}
        \frac{n+k-1}n V_1'(\btheta) &\leq \matrixnorm{\frac1n\sum_{i=1}^n \left(\nabla\cL(\theta;Z_{i1})-\nabla\cLs(\theta)-\nabla\cL(\thetas;Z_{i1})\right)\left(\nabla\cL(\theta;Z_{i1})-\nabla\cLs(\theta)-\nabla\cL(\thetas;Z_{i1})\right)^\top}_{\max} \\
        &\quad+ \matrixnorm{\frac1n\sum_{i=1}^n \nabla\cL(\thetas;Z_{i1})\left(\nabla\cL(\theta;Z_{i1})-\nabla\cLs(\theta)-\nabla\cL(\thetas;Z_{i1})\right)^\top}_{\max} \\
        &\quad+ \matrixnorm{\frac1n\sum_{i=1}^n \left(\nabla\cL(\theta;Z_{i1})-\nabla\cLs(\theta)-\nabla\cL(\thetas;Z_{i1})\right)\nabla\cL(\thetas;Z_{i1})^\top}_{\max} \\
        &= \matrixnorm{\frac1n\sum_{i=1}^n \left(\nabla\cL(\theta;Z_{i1})-\nabla\cLs(\theta)-\nabla\cL(\thetas;Z_{i1})\right)\left(\nabla\cL(\theta;Z_{i1})-\nabla\cLs(\theta)-\nabla\cL(\thetas;Z_{i1})\right)^\top}_{\max} \\
        &\quad+ 2\matrixnorm{\frac1n\sum_{i=1}^n \nabla\cL(\thetas;Z_{i1})\left(\nabla\cL(\theta;Z_{i1})-\nabla\cLs(\theta)-\nabla\cL(\thetas;Z_{i1})\right)^\top}_{\max} \\
        &\defn V_{11}'(\btheta) + 2 V_{12}'(\btheta).
        \end{align*}
        Applying Cauchy-Schwarz inequality, we obtain that
        \begin{align*}
        V_{12}'(\btheta) &\leq \matrixnorm{\frac1n\sum_{i=1}^n \nabla\cL(\thetas;Z_{i1})\nabla\cL(\thetas;Z_{i1})^\top}_{\max}^{1/2} \\
        &\quad \cdot\matrixnorm{\frac1n\sum_{i=1}^n \left(\nabla\cL(\theta;Z_{i1})-\nabla\cLs(\theta)-\nabla\cL(\thetas;Z_{i1})\right)\left(\nabla\cL(\theta;Z_{i1})-\nabla\cLs(\theta)-\nabla\cL(\thetas;Z_{i1})\right)^\top}_{\max}^{1/2} \\
        &= \matrixnorm{\frac1n\sum_{i=1}^n \nabla\cL(\thetas;Z_{i1})\nabla\cL(\thetas;Z_{i1})^\top}_{\max}^{1/2}V_{11}'(\btheta)^{1/2}.
        \end{align*}
        By the triangle inequality, we have that
        \begin{align*}
        &\matrixnorm{\frac1n\sum_{i=1}^n \nabla\cL(\thetas;Z_{i1})\nabla\cL(\thetas;Z_{i1})^\top}_{\max} \\
        &\leq \matrixnorm{\frac1n\sum_{i=1}^n \nabla\cL(\thetas;Z_{i1})\nabla\cL(\thetas;Z_{i1})^\top-\Ee\left[\nabla\cL(\thetas;Z)\nabla\cL(\thetas;Z)^\top\right]}_{\max} + \matrixnorm{\Ee\left[\nabla\cL(\thetas;Z)\nabla\cL(\thetas;Z)^\top\right]}_{\max} \\
        &= \frac{n+k-1}n V_2' + \sigma^2\matrixnorm{\Sigma}_{\max} = O_P\left(1+\sqrt{\frac{\log d}n} + \frac{\log^2(dn)\log d}n\right).
        \end{align*}
        It remains to bound $V_{11}'(\btheta)$.  Note that
        \begin{align*}
        \nabla\cL(\theta;Z_{i1})-\nabla\cLs(\theta)-\nabla\cL(\thetas;Z_{i1})
        &= x_{ij}(x_{ij}^\top\btheta-y_{ij}) - \Sigma(\btheta-\thetas) + x_{ij}(x_{ij}^\top\thetas-y_{ij}) = \left(x_{ij} x_{ij}^\top-\Sigma\right)(\btheta-\thetas).
        \end{align*}
        Then, we have by the triangle inequality that
        \begin{align*}
        V_{11}'(\btheta) &= \matrixnorm{\frac1n\sum_{i=1}^n \left(x_{i1} x_{i1}^\top-\Sigma\right)(\btheta-\thetas)(\btheta-\thetas)^\top\left(x_{i1} x_{i1}^\top-\Sigma\right)}_{\max} \\
        &\leq \frac1n\sum_{i=1}^n \matrixnorm{\left(x_{i1} x_{i1}^\top-\Sigma\right)(\btheta-\thetas)(\btheta-\thetas)^\top\left(x_{i1} x_{i1}^\top-\Sigma\right)}_{\max} \\
        &= \frac1n\sum_{i=1}^n \matrixnorm{\left(x_{i1} x_{i1}^\top-\Sigma\right)(\btheta-\thetas)}_\infty^2 \leq \frac1n\sum_{i=1}^n \matrixnorm{x_{i1} x_{i1}^\top-\Sigma}_{\max}^2 \left\|\btheta-\thetas\right\|_1^2.
        \end{align*}
        Similarly to obtaining \eqref{eqn:bern3}, we have that
        $$P\left(\matrixnorm{x_{i1} x_{i1}^\top-\Sigma}_{\max} > \matrixnorm{\Sigma}_{\max}\left(\frac{\log\frac{2nd^2}\delta}c\vee\sqrt{\frac{\log\frac{2nd^2}\delta}c}\right)\right)\leq \frac\delta n,$$
        which implies by the union bound that
        $$\max_i\matrixnorm{x_{i1} x_{i1}^\top-\Sigma}_{\max} = O_P\left(\sqrt{\log(nd)}\right).$$
        Putting all the preceding bounds together, we obtain that
        $$V_{11}'(\btheta) = O_P\left(\log(nd) r_{\btheta}^2\right),$$
        $$V_{12}'(\btheta) = O_P\left(\left(1+\left(\frac{\log d}n\right)^{1/4} + \sqrt{\frac{\log^2(dn)\log d}n}\right)\sqrt{\log(nd)} r_{\btheta}\right),$$
        \begin{align*}
        V_1'(\btheta) &= \frac n{n+k-1} O_P\left(\left(1+\left(\frac{\log d}n\right)^{1/4} + \sqrt{\frac{\log^2(dn)\log d}n}\right)\sqrt{\log(nd)} r_{\btheta} + \log(nd) r_{\btheta}^2\right) \\
        &= O_P\left(\left(1+\left(\frac{\log d}n\right)^{1/4} + \sqrt{\frac{\log^2(dn)\log d}n}\right)\frac{n\sqrt{\log(nd)}}{n+k} r_{\btheta} + \frac{n\log(nd)}{n+k} r_{\btheta}^2\right),
        \end{align*}
        and finally the bound in the lemma.
\end{proof}

\begin{lemma}\label{lem:vcov0_reg_glm}
	In GLM, under Assumptions~\ref{as:smth_glm}--\ref{as:hes_glm}, provided that $\left\|\btheta-\thetas\right\|_1=O_P(r_{\btheta})$, we have that
	\begin{align*}
	&\matrixnorm{\frac1k\sum_{j=1}^k n\left(\nabla\cL_j(\btheta)-\nabla\cL_N(\btheta)\right)\left(\nabla\cL_j(\btheta)-\nabla\cL_N(\btheta)\right)^\top-\Ee\left[\nabla\cL(\thetas;Z)\nabla\cL(\thetas;Z)^\top\right]}_{\max} \\
	&= O_P\left(\sqrt{\frac{\log d}k} + \frac{\log d}k + \left(1+\left(\frac{\log d}k\right)^{1/4}\right)\left(\sqrt{\log d} + \sqrt{n}r_{\btheta}\right) r_{\btheta} + \left(n + \log d + nr_{\btheta}^2\right) r_{\btheta}^2\right).
	\end{align*}
\end{lemma}

\begin{proof}[Lemma \ref{lem:vcov0_reg_glm}]
	By Lemma~\ref{lem:kgrad_decomp}, it suffices to bound $U_1(\btheta)$, $U_2$, and $U_3(\btheta)$.  We begin by bounding $U_2$.  Note that $\nabla\cL_N(\thetas)=\sum_{i=1}^n\sum_{j=1}^k g'(y_{ij},x_{ij}^\top\thetas) x_{ij}/N$ and $g'(y_{ij},x_{ij}^\top\thetas) x_{ij,l}=O(1)$ for each $l=1,\dots,d$ under Assumptions~\ref{as:smth_glm}~and~\ref{as:design_glm}.  Then, by Hoeffding's inequality, we have that for any $t>0$,
	$$P\left(\sqrt n\left|\nabla\cL_j(\thetas)_l\right|>t\right)\leq2\exp\left(-\frac{t^2}c\right),$$
	that is, $\sqrt n \nabla\cL_j(\thetas)_l$ is sub-Gaussian with $O(1)$ $\psi_2$-norm.  Therefore, $n \nabla\cL_j(\thetas)_l\nabla\cL_j(\thetas)_{l'}$ is sub-exponential with $O(1)$ $\psi_1$-norm.  Note that $\Ee[n \nabla\cL_j(\thetas)_l\nabla\cL_j(\thetas)_{l'}]=\Ee[\nabla\cL(\thetas;Z)_l\nabla\cL(\thetas;Z)_{l'}]$.  Then, we apply Bernstein's inequality and obtain that for any $\delta\in(0,1)$,
	$$P\left(\left|\frac1k\sum_{j=1}^k n \nabla\cL_j(\thetas)_l\nabla\cL_j(\thetas)_{l'}-\Ee\left[\nabla\cL(\thetas;Z)_l\nabla\cL(\thetas;Z)_{l'}\right]\right|>\sqrt{\frac{\log\frac{2d^2}\delta}{ck}}\vee\frac{\log\frac{2d^2}\delta}{ck}\right)\leq\frac\delta{d^2},$$
	which implies by the union bound that
	$$U_2=O_P\left(\sqrt{\frac{\log d}k}\right).$$
	
	Next, we bound $U_3(\btheta)$.  By the triangle inequality, we have that
	\begin{align*}
	\left\|\nabla\cL_N(\btheta)-\nabla\cLs(\btheta)\right\|_\infty
	&\leq \left\|\nabla\cL_N(\btheta)-\nabla\cL_N(\thetas)\right\|_\infty + \left\|\nabla\cL_N(\thetas)\right\|_\infty + \left\|\nabla\cLs(\btheta)\right\|_\infty.
	\end{align*}
	By an expression of remainder of the first order Taylor expansion, we have that
	\begin{align*}
	\nabla\cL_N(\btheta)-\nabla\cL_N(\thetas)&=\int_0^1\nabla^2\cL_N(\thetas+t(\btheta-\thetas))dt(\btheta-\thetas) \\
	&=\int_0^1\frac1N\sum_{i=1}^n\sum_{j=1}^k g''(y_{ij},x_{ij}^\top(\thetas+t(\btheta-\thetas)))x_{ij}x_{ij}^\top dt(\btheta-\thetas),
	\end{align*}
	and then, under Assumptions~\ref{as:smth_glm}~and~\ref{as:design_glm},
	\begin{align*}
	\left\|\nabla\cL_N(\btheta)-\nabla\cL_N(\thetas)\right\|_\infty&=\int_0^1\frac1N\sum_{i=1}^n\sum_{j=1}^k \left|g''(y_{ij},x_{ij}^\top(\thetas+t(\btheta-\thetas)))\right|\left\|x_{ij}\right\|_\infty^2 dt \left\|\btheta-\thetas\right\|_\infty \lesssim \left\|\btheta-\thetas\right\|_\infty.
	\end{align*}
	Note that for any $\theta$,
	\begin{align*}
	\left\|\nabla\cLs(\theta)\right\|_\infty&=\left\|\nabla\cLs(\theta)-\nabla\cLs(\thetas)\right\|_\infty =\left\|\Ee\left[\left(g'(y,x^\top\theta)-g'(y,x^\top\thetas))\right)x\right]\right\|_\infty \\
	&=\left\|\Ee\left[\int_0^1 g''(y,x^\top(\thetas+t(\theta-\thetas)))dt xx^\top(\theta-\thetas)\right]\right\|_\infty \\ &\leq\Ee\left[\int_0^1 \left|g''(y,x^\top(\thetas+t(\theta-\thetas)))\right|dt \left\|x\right\|_\infty^2 \left\|\theta-\thetas\right\|_\infty\right] \lesssim \left\|\theta-\thetas\right\|_\infty.
	\end{align*}
	Therefore, $\left\|\nabla\cLs(\btheta)\right\|_\infty\lesssim \left\|\btheta-\thetas\right\|_\infty$. By \eqref{eqn:hoef2}, we have that
	$$\left\|\nabla\cL_N(\thetas)\right\|_\infty=O_P\left(\sqrt{\frac{\log d}N}\right).$$
	Then, assuming that $\left\|\btheta-\thetas\right\|_1=O_P(r_{\btheta})$, we have that
	\begin{align*}
	\left\|\nabla\cL_N(\btheta)-\nabla\cLs(\btheta)\right\|_\infty
	&=O_P\left(r_{\btheta}+\sqrt{\frac{\log d}N}\right),
	\end{align*}
	and then,
	$$U_3(\btheta) = O_P\left(n r_{\btheta}^2 + \frac{\log d}k\right).$$
	
	Lastly, we bound $U_1(\btheta)$.  As in the proof of Lemma~\ref{lem:vcov0_reg}, we have that
        \begin{align*}
        U_1(\btheta) &\leq \matrixnorm{\frac1k\sum_{j=1}^k n\left(\nabla\cL_j(\btheta)-\nabla\cLs(\btheta)-\nabla\cL_j(\thetas)\right)\left(\nabla\cL_j(\btheta)-\nabla\cLs(\btheta)-\nabla\cL_j(\thetas)\right)^\top}_{\max} \\
        &\quad+ 2\matrixnorm{\frac1k\sum_{j=1}^k n\nabla\cL_j(\thetas)\left(\nabla\cL_j(\btheta)-\nabla\cLs(\btheta)-\nabla\cL_j(\thetas)\right)^\top}_{\max} \\
        &\defn U_{11}(\btheta) + 2 U_{12}(\btheta),
        \end{align*}
        and
        \begin{align*}
        U_{12}(\btheta) &\leq \matrixnorm{\frac1k\sum_{j=1}^k n\nabla\cL_j(\thetas)\nabla\cL_j(\thetas)^\top}_{\max}^{1/2} U_{11}(\btheta)^{1/2}.
        \end{align*}
        Note that $\matrixnorm{\Ee\left[\nabla\cL(\thetas;Z)\nabla\cL(\thetas;Z)^\top\right]}_{\max} = O(1)$ under Assumption~\ref{as:hes_glm}.  Then, by the triangle inequality, we have that
        \begin{align*}
        &\matrixnorm{\frac1k\sum_{j=1}^k n\nabla\cL_j(\thetas)\nabla\cL_j(\thetas)^\top}_{\max} \\
        &\leq \matrixnorm{\frac1k\sum_{j=1}^k n\nabla\cL_j(\thetas)\nabla\cL_j(\thetas)^\top-\Ee\left[\nabla\cL(\thetas;Z)\nabla\cL(\thetas;Z)^\top\right]}_{\max} + \matrixnorm{\Ee\left[\nabla\cL(\thetas;Z)\nabla\cL(\thetas;Z)^\top\right]}_{\max} \\
        &=U_2 + \matrixnorm{\Ee\left[\nabla\cL(\thetas;Z)\nabla\cL(\thetas;Z)^\top\right]}_{\max} = O_P\left(1+\sqrt{\frac{\log d}k}\right).
        \end{align*}
        It remains to bound $U_{11}(\btheta)$.  Note that
	\begin{align*}
	\nabla\cL_j(\btheta)-\nabla\cL_j(\thetas)&=\int_0^1\nabla^2\cL_j(\thetas+t(\btheta-\thetas))dt(\btheta-\thetas) =\int_0^1\frac1n\sum_{i=1}^n g''(y_{ij},x_{ij}^\top(\thetas+t(\btheta-\thetas)))x_{ij}x_{ij}^\top dt(\btheta-\thetas),
	\end{align*}
	and
	\begin{align*}
	g''(y_{ij},x_{ij}^\top(\thetas+t(\btheta-\thetas)))&=g''(y_{ij},x_{ij}^\top\thetas)+\int_0^1 g'''(y_{ij},x_{ij}^\top(\thetas+st(\btheta-\thetas)))ds x_{ij}^\top(t(\btheta-\thetas)),
	\end{align*}
	and then
	\begin{align*}
	\nabla\cL_j(\btheta)-\nabla\cL_j(\thetas)&=\frac1n\sum_{i=1}^n g''(y_{ij},x_{ij}^\top\thetas)x_{ij}x_{ij}^\top (\btheta-\thetas) \\
	&\quad+ \int_0^1\int_0^1 \frac1n\sum_{i=1}^n g'''(y_{ij},x_{ij}^\top(\thetas+st(\btheta-\thetas))) x_{ij}^\top t(\btheta-\thetas) x_{ij}x_{ij}^\top dtds (\btheta-\thetas).
	\end{align*}
	In a similar way, we have that
	\begin{align*}
	\nabla\cLs(\btheta)&=\nabla\cLs(\btheta)-\nabla\cLs(\thetas) \\
	&= \Ee\left[g''(y,x^\top\thetas)xx^\top\right] (\btheta-\thetas) + \int_0^1\int_0^1 \Ee_{x,y}\left[g'''(y,x^\top(\thetas+st(\btheta-\thetas))) x^\top t(\btheta-\thetas) xx^\top\right] dtds (\btheta-\thetas),
	\end{align*}
	and then,
	\begin{align*}
        \nabla\cL_j(\btheta)-\nabla\cLs(\btheta)-\nabla\cL_j(\thetas)
        &=\left(\frac1n\sum_{i=1}^n g''(y_{ij},x_{ij}^\top\thetas)x_{ij}x_{ij}^\top-\Ee\left[g''(y,x^\top\thetas)xx^\top\right]\right) (\btheta-\thetas) \\
	&\quad+ \int_0^1\int_0^1 \frac1n\sum_{i=1}^n g'''(y_{ij},x_{ij}^\top(\thetas+st(\btheta-\thetas))) x_{ij}^\top t(\btheta-\thetas) x_{ij}x_{ij}^\top \\
	&\quad - \Ee_{x,y}\left[g'''(y,x^\top(\thetas+st(\btheta-\thetas))) x^\top t(\btheta-\thetas) xx^\top\right] dtds (\btheta-\thetas) \\
	&\defn U_{111,j}+U_{112,j}(\btheta).
        \end{align*}
        Then, we have by the triangle inequality that
        	\begin{align*}
	U_{11}(\btheta)&=\matrixnorm{\frac1k\sum_{j=1}^k n\left(U_{111,j}+U_{112,j}(\btheta)\right)\left(U_{111,j}+U_{112,j}(\btheta)\right)^\top}_{\max} \\
	&\leq \frac1k\sum_{j=1}^k n \matrixnorm{\left(U_{111,j}+U_{112,j}(\btheta)\right)\left(U_{111,j}+U_{112,j}(\btheta)\right)^\top}_{\max} \\
	&= \frac1k\sum_{j=1}^k n \left\|U_{111,j}+U_{112,j}(\btheta)\right\|_\infty^2 \leq \frac2k\sum_{j=1}^k n \left(\left\|U_{111,j}\right\|_\infty^2 + \left\|U_{112,j}(\btheta)\right\|_\infty^2\right)
	\end{align*}
	Using the argument for obtaining \eqref{eqn:hoef3}, we have that
	\begin{align*}
	\left\|U_{111,j}\right\|_\infty&=\left\|\left(\nabla^2\cL_j(\thetas)-\nabla^2\cLs(\thetas)\right)(\btheta-\thetas)\right\|_\infty \leq\matrixnorm{\nabla^2\cL_j(\thetas)-\nabla^2\cLs(\thetas)}_{\max}\left\|\btheta-\thetas\right\|_1 \\
	&=O_P\left(\sqrt{\frac{\log d}n}\right) O_P\left(r_{\btheta}\right) =O_P\left(\sqrt{\frac{\log d}n} r_{\btheta}\right).
	\end{align*}
	Under Assumptions~\ref{as:smth_glm}~and~\ref{as:design_glm}, we have that
	\begin{align*}
	\left\|U_{112,j}(\btheta)\right\|_\infty&\leq\int_0^1\int_0^1 \frac1n\sum_{i=1}^n \left|g'''(y_{ij},x_{ij}^\top(\thetas+st(\btheta-\thetas)))\right| \left\|x_{ij}\right\|_\infty t\left\|\btheta-\thetas\right\|_1 \left\|x_{ij}\right\|_\infty^2 \\
	&\quad + \Ee_{x,y}\left[\left|g'''(y,x^\top(\thetas+st(\btheta-\thetas)))\right| \left\|x\right\|_\infty t\left\|\btheta-\thetas\right\|_1 \left\|x\right\|_\infty^2\right] dtds \left\|\btheta-\thetas\right\|_1 \\
	&\lesssim \left\|\btheta-\thetas\right\|_1^2 =O_P\left(r_{\btheta}^2\right).
	\end{align*}
	Hence, we have that
	$$U_{11}(\btheta) = n\left(O_P\left(\frac{\log d}n r_{\btheta}^2\right) + O_P\left(r_{\btheta}^4\right)\right) = O_P\left(\left(\log d + nr_{\btheta}^2\right) r_{\btheta}^2\right).$$
        Putting all the preceding bounds together, we obtain that
        $$U_{12}(\btheta) = O_P\left(\left(1+\left(\frac{\log d}k\right)^{1/4}\right)\left(\sqrt{\log d} + \sqrt{n}r_{\btheta}\right) r_{\btheta}\right),$$
        $$U_1(\btheta) = O_P\left(\left(1+\left(\frac{\log d}k\right)^{1/4}\right)\left(\sqrt{\log d} + \sqrt{n}r_{\btheta}\right) r_{\btheta} + \left(\log d + nr_{\btheta}^2\right) r_{\btheta}^2\right),$$
        and finally the bound in the lemma.
\end{proof}

\begin{lemma}\label{lem:vcov_reg_glm}
	In GLM, under Assumptions~\ref{as:smth_glm}--\ref{as:hes_glm}, provided that $\left\|\btheta-\thetas\right\|_1=O_P(r_{\btheta})$, we have that
	\begin{align*}
	&\Bigg|\!\Bigg|\!\Bigg|\frac1{n+k-1}\Bigg(\sum_{i=1}^n\left(\nabla\cL(\btheta;Z_{i1})-\nabla\cL_N(\btheta)\right)\left(\nabla\cL(\btheta;Z_{i1})-\nabla\cL_N(\btheta)\right)^\top \\
	&\quad+\sum_{j=2}^k n\left(\nabla\cL_j(\btheta)-\nabla\cL_N(\btheta)\right)\left(\nabla\cL_j(\btheta)-\nabla\cL_N(\btheta)\right)^\top\Bigg) -\Ee\left[\nabla\cL(\thetas;Z)\nabla\cL(\thetas;Z)^\top\right]\Bigg|\!\Bigg|\!\Bigg|_{\max} \\
	&= O_P\Bigg(\sqrt{\frac{\log d}{n+k}} + \frac{\log d}{n+k} + \frac{nk}{n+k}r_{\btheta}^2 + \left(1+\left(\frac{\log d}n\right)^{1/4}\right)\frac{n}{n+k} \left(r_{\btheta} + r_{\btheta}^2\right) + \frac{n}{n+k} r_{\btheta}^4 \\
	&\quad + \left(1+\left(\frac{\log d}k\right)^{1/4}\right)\frac{k\sqrt{\log d}+k\sqrt n r_{\btheta}}{n+k} r_{\btheta} + \frac{k\log d+knr_{\btheta}^2}{n+k} r_{\btheta}^2\Bigg).
	\end{align*}
\end{lemma}

\begin{proof}[Lemma \ref{lem:vcov_reg_glm}]
	By Lemma~\ref{lem:nk1grad_decomp}, it suffices to bound $V_1(\btheta)$, $V_1'(\btheta)$, $V_2$, $V_2'$, and $V_3(\btheta)$.  By the proof of Lemma~\ref{lem:vcov0_reg_glm}, we have that under Assumptions~\ref{as:smth_glm}--\ref{as:hes_glm}, assuming that $\left\|\btheta-\thetas\right\|_1=O_P(r_{\btheta})$,
	\begin{align*}
	V_1(\btheta) &= \frac{k-1}{n+k-1} O_P\left(\left(1+\left(\frac{\log d}k\right)^{1/4}\right)\left(\sqrt{\log d} + \sqrt{n}r_{\btheta}\right) r_{\btheta} + \left(\log d + nr_{\btheta}^2\right) r_{\btheta}^2\right) \\
	&= O_P\left(\left(1+\left(\frac{\log d}k\right)^{1/4}\right)\frac{k\sqrt{\log d}+k\sqrt n r_{\btheta}}{n+k} r_{\btheta} + \frac{k\log d+knr_{\btheta}^2}{n+k} r_{\btheta}^2\right),
	\end{align*}
	$$V_2 = \frac{k-1}{n+k-1} O_P\left(\sqrt{\frac{\log d}k}\right) = O_P\left(\frac{\sqrt{k\log d}}{n+k}\right),\quad\text{and}$$
	$$V_3(\btheta) =  \frac{nk}{n+k-1} O_P\left(r_{\btheta}^2+\frac{\log d}N\right) = O_P\left(\frac{nk}{n+k}r_{\btheta}^2 + \frac{\log d}{n+k}\right).$$
	It remains to bound $V_1'(\btheta)$ and $V_2'$.
	
	To bound $V_2'$, we note that each $\nabla\cL(\thetas;Z_{i1})_l\nabla\cL(\thetas;Z_{i1})_{l'}=g'(y_{i1},x_{i1}^\top\thetas)^2 x_{i1,l} x_{i1,l'}$ is bounded under Assumptions~\ref{as:smth_glm}~and~\ref{as:design_glm}.  Applying Hoeffding's inequality, we obtain that for any $\delta\in(0,1)$,
	$$P\left(\left|\frac1n\sum_{i=1}^n\nabla\cL(\thetas;Z_{i1})_l\nabla\cL(\thetas;Z_{i1})_{l'}-\Ee\left[\nabla\cL(\thetas;Z)_l\nabla\cL(\thetas;Z)_{l'}\right]\right| > \sqrt{\frac{c\log\frac{2d^2}\delta}n}\right)\leq\frac\delta{d^2},$$
	which implies by the union bound that
	$$V_2'=\frac n{n+k-1} O_P\left(\sqrt{\frac{\log d}n}\right)=O_P\left(\sqrt{\frac{n\log d}{n+k}}\right).$$
	
	Lastly, we bound $V_1'(\btheta)$.  As in the proof of Lemma~\ref{lem:vcov_reg}, we have that
	\begin{align*}
        \frac{n+k-1}n V_1'(\btheta) &\leq \matrixnorm{\frac1n\sum_{i=1}^n \left(\nabla\cL(\theta;Z_{i1})-\nabla\cLs(\theta)-\nabla\cL(\thetas;Z_{i1})\right)\left(\nabla\cL(\theta;Z_{i1})-\nabla\cLs(\theta)-\nabla\cL(\thetas;Z_{i1})\right)^\top}_{\max} \\
        &\quad+ 2\matrixnorm{\frac1n\sum_{i=1}^n \nabla\cL(\thetas;Z_{i1})\left(\nabla\cL(\theta;Z_{i1})-\nabla\cLs(\theta)-\nabla\cL(\thetas;Z_{i1})\right)^\top}_{\max} \\
        &\defn V_{11}'(\btheta) + 2 V_{12}'(\btheta),
        \end{align*}
        and
        \begin{align*}
        V_{12}'(\btheta) &\leq \matrixnorm{\frac1n\sum_{i=1}^n \nabla\cL(\thetas;Z_{i1})\nabla\cL(\thetas;Z_{i1})^\top}_{\max}^{1/2}V_{11}'(\btheta)^{1/2}.
        \end{align*}
        Note that $\matrixnorm{\Ee\left[\nabla\cL(\thetas;Z)\nabla\cL(\thetas;Z)^\top\right]}_{\max} = O(1)$ under Assumption~\ref{as:hes_glm}.  Then, by the triangle inequality, we have that
        \begin{align*}
        &\matrixnorm{\frac1n\sum_{i=1}^n \nabla\cL(\thetas;Z_{i1})\nabla\cL(\thetas;Z_{i1})^\top}_{\max} \\
        &\leq \matrixnorm{\frac1n\sum_{i=1}^n \nabla\cL(\thetas;Z_{i1})\nabla\cL(\thetas;Z_{i1})^\top-\Ee\left[\nabla\cL(\thetas;Z)\nabla\cL(\thetas;Z)^\top\right]}_{\max} + \matrixnorm{\Ee\left[\nabla\cL(\thetas;Z)\nabla\cL(\thetas;Z)^\top\right]}_{\max} \\
        &= \frac{n+k-1}n V_2' + \matrixnorm{\Ee\left[\nabla\cL(\thetas;Z)\nabla\cL(\thetas;Z)^\top\right]}_{\max} = O_P\left(1+\sqrt{\frac{\log d}n}\right).
        \end{align*}
        It remains to bound $V_{11}'(\btheta)$.  Using the same argument for analyzing $\nabla\cL_j(\btheta)-\nabla\cLs(\btheta)-\nabla\cL_j(\thetas)$ in the proof of Lemma~\ref{lem:vcov0_reg_glm}, we obtain that
        \begin{align*}
        \nabla\cL(\theta;Z_{i1})-\nabla\cLs(\theta)-\nabla\cL(\thetas;Z_{i1})
        &=\left(g''(y_{i1},x_{i1}^\top\thetas)x_{i1}x_{i1}^\top-\Ee\left[g''(y,x^\top\thetas)xx^\top\right]\right) (\btheta-\thetas) \\
	&\quad+ \int_0^1\int_0^1 g'''(y_{i1},x_{i1}^\top(\thetas+st(\btheta-\thetas))) x_{i1}^\top t(\btheta-\thetas) x_{i1}x_{i1}^\top \\
	&\quad - \Ee_{x,y}\left[g'''(y,x^\top(\thetas+st(\btheta-\thetas))) x^\top t(\btheta-\thetas) xx^\top\right] dtds (\btheta-\thetas) \\
	&\defn V_{111,i}'+V_{112,i}'(\btheta),
        \end{align*}
        and
        	\begin{align*}
	V_{11}'(\btheta)&=\matrixnorm{\frac1n\sum_{i=1}^n \left(V_{111,i}'+V_{112,i}'(\btheta)\right)\left(V_{111,i}'+V_{112,i}'(\btheta)\right)^\top}_{\max} \\
	&\leq \frac1n\sum_{i=1}^n \matrixnorm{\left(V_{111,i}'+V_{112,i}'(\btheta)\right)\left(V_{111,i}'+V_{112,i}'(\btheta)\right)^\top}_{\max} \\
	&= \frac1n\sum_{i=1}^n \left\|V_{111,i}'+V_{112,i}'(\btheta)\right\|_\infty^2 \leq \frac2n\sum_{i=1}^n \left(\left\|V_{111,i}'\right\|_\infty^2 + \left\|V_{112,i}'(\btheta)\right\|_\infty^2\right).
	\end{align*}
	Moreover, under Assumptions~\ref{as:smth_glm}--\ref{as:hes_glm}, we have that
	\begin{align*}
	\left\|V_{111,i}'\right\|_\infty&=\left\|\left(\nabla^2\cL(\thetas;Z_{i1})-\nabla^2\cLs(\thetas)\right)(\btheta-\thetas)\right\|_\infty \leq\matrixnorm{\nabla^2\cL(\thetas;Z_{i1})-\nabla^2\cLs(\thetas)}_{\max}\left\|\btheta-\thetas\right\|_1 \\
	&\leq\left(\left|g''(y_{i1},x_{i1}^\top\thetas)\right|\left\|x_{i1}\right\|_\infty^2+\matrixnorm{\nabla^2\cLs(\thetas)}_{\max}\right)\left\|\btheta-\thetas\right\|_1 =O_P\left(r_{\btheta}\right),
	\end{align*}
	and
	\begin{align*}
	\left\|V_{112,i}'(\btheta)\right\|_\infty&\leq\int_0^1\int_0^1  \left|g'''(y_{i1},x_{i1}^\top(\thetas+st(\btheta-\thetas)))\right| \left\|x_{i1}\right\|_\infty t\left\|\btheta-\thetas\right\|_1 \left\|x_{i1}\right\|_\infty^2 \\
	&\quad + \Ee_{x,y}\left[\left|g'''(y,x^\top(\thetas+st(\btheta-\thetas)))\right| \left\|x\right\|_\infty t\left\|\btheta-\thetas\right\|_1 \left\|x\right\|_\infty^2\right] dtds \left\|\btheta-\thetas\right\|_1 \\
	&\lesssim \left\|\btheta-\thetas\right\|_1^2 =O_P\left(r_{\btheta}^2\right),
	\end{align*}
	and hence,
	$$V_{11}'(\btheta) = O_P\left(r_{\btheta}^2 + r_{\btheta}^4\right).$$
	
	Putting all the preceding bounds together, we obtain that
        $$V_{12}'(\btheta) = O_P\left(\left(1+\left(\frac{\log d}n\right)^{1/4}\right)\left(r_{\btheta} + r_{\btheta}^2\right)\right),$$
        \begin{align*}
        V_1'(\btheta) &= \frac n{n+k-1} O_P\left(\left(1+\left(\frac{\log d}n\right)^{1/4}\right)\left(r_{\btheta} + r_{\btheta}^2\right) + r_{\btheta}^2 + r_{\btheta}^4\right) \\
        &= O_P\left(\left(1+\left(\frac{\log d}n\right)^{1/4}\right)\frac{n}{n+k} \left(r_{\btheta} + r_{\btheta}^2\right) + \frac{n}{n+k} r_{\btheta}^4\right),
        \end{align*}
        and finally the bound in the lemma.
\end{proof}

\begin{lemma}\label{lem:hes_ld}
	In linear model, under Assumption~\ref{as:design}, if $n\gtrsim d$, we have that
	$$\matrixnorm{\widetilde\Theta}_\infty=O_P\left(\sqrt d\right)\quad\text{and}\quad\max_l\left\|\widetilde\Theta_l-\Theta_l\right\|_2=O_P\left(\sqrt{\frac dn}\right).$$
\end{lemma}

\begin{proof}[Lemma \ref{lem:hes_ld}]
	$\widetilde\Theta$ is simply the inverse of $X_1^\top X_1/n$.  We use the fact that for any matrix $A,B\in\R^{d\times d}$, $\matrixnorm{A^{-1}-B^{-1}}_2\leq\matrixnorm{B^{-1}}_2^2\matrixnorm{A-B}_2$, and obtain that
	$$\matrixnorm{\widetilde\Theta-\Theta}_2=\matrixnorm{\left(\frac{X_1^\top X_1}n\right)^{-1}-\Sigma^{-1}}_2\leq\matrixnorm{\Sigma^{-1}}_2^2\matrixnorm{\frac{X_1^\top X_1}n-\Sigma}_2.$$
	Since the design matrix is sub-Gaussian and $\matrixnorm{\Sigma}_2=O(1)$, by Proposition~2.1 of \cite{vershynin2012close}, we have that if $n\gtrsim d$,
	$$\matrixnorm{\frac{X_1^\top X_1}n-\Sigma}_2=O_P\left(\sqrt{\frac dn}\right).$$
	Also note that $\matrixnorm{\Sigma^{-1}}_2=O(1)$, and then, we have that
	\begin{align}
	\max_l\left\|\widetilde\Theta_l-\Theta_l\right\|_2\leq\matrixnorm{\widetilde\Theta-\Theta}_2=O_P\left(\sqrt{\frac dn}\right), \quad\text{and} \label{eqn:inv_conc}
	\end{align}
	$$\matrixnorm{\widetilde\Theta-\Theta}_\infty\leq\sqrt d\matrixnorm{\widetilde\Theta-\Theta}_2=O_P\left(\frac d{\sqrt{n}}\right).$$
	Note that $\matrixnorm{\Theta}_\infty\leq\sqrt d\matrixnorm{\Theta}_2=\sqrt d\matrixnorm{\Sigma^{-1}}_2=O\left(\sqrt d\right)$.  By the triangle inequality, we have that
	$$\matrixnorm{\widetilde\Theta}_\infty\leq\matrixnorm{\widetilde\Theta-\Theta}_\infty+\matrixnorm{\Theta}_\infty=O_P\left(\frac d{\sqrt{n}}\right)+O\left(\sqrt d\right)=O_P\left(\sqrt d\right).$$
\end{proof}

\begin{lemma}\label{lem:hes_ld_glm}
	In GLM, under Assumptions~\ref{as:smth_glm}--\ref{as:hes_glm}, if $n\gtrsim d\log d$ and $r_{\btheta}\lesssim1$, we have that
	$$\matrixnorm{\widetilde\Theta(\btheta)}_\infty=O_P\left(\sqrt d\right)\quad\text{and}\quad\max_l\left\|\widetilde\Theta(\btheta)_l-\Theta_l\right\|_2=O_P\left(\sqrt{\frac{d\log d}n}+r_{\btheta}\right).$$
\end{lemma}

\begin{proof}[Lemma \ref{lem:hes_ld_glm}]
    $\widetilde\Theta(\btheta)$ is simply the inverse of $\nabla^2\cL_1(\btheta)$.  Then, we have that
	$$\matrixnorm{\widetilde\Theta(\btheta)-\Theta}_2=\matrixnorm{\nabla^2\cL_1(\btheta)^{-1}-\nabla^2\cLs(\thetas)^{-1}}_2\leq\matrixnorm{\nabla^2\cLs(\thetas)^{-1}}_2^2\matrixnorm{\nabla^2\cL_1(\btheta)-\nabla^2\cLs(\thetas)}_2.$$
	Note that
	\begin{align*}
	\matrixnorm{\nabla^2\cL_1(\thetas)-\nabla^2\cLs(\thetas)}_2 = \matrixnorm{\frac1n\sum_{i=1}^n g''(y_{ij},x_{ij}^\top\thetas) x_{ij} x_{ij}^\top-\Ee[g''(y,x^\top\thetas)xx^\top]}_2,
	\end{align*}
	$$\left\|\sqrt{g''(y_{ij},x_{ij}^\top\thetas)}x_{ij}\right\|_2=\sqrt d\left\|\sqrt{g''(y_{ij},x_{ij}^\top\thetas)}x_{ij}\right\|_\infty=O(\sqrt d),$$
	and $\matrixnorm{\nabla^2\cLs(\thetas)}_2=O(1)$.  By Section~1.6.3 of \cite{tropp2015introduction}, we have that if $n\gtrsim d\log d$,
	$$\Ee\left[\matrixnorm{\nabla^2\cL_1(\thetas)-\nabla^2\cLs(\thetas)}_2\right]\lesssim\sqrt{\frac{d\log d}n},$$
	which implies that
	$$\matrixnorm{\nabla^2\cL_1(\thetas)-\nabla^2\cLs(\thetas)}_2=O_P\left(\sqrt{\frac{d\log d}n}\right).$$
	Also note that
	$$\matrixnorm{\nabla^2\cL_1(\btheta)-\nabla^2\cL_1(\thetas)}_2=\matrixnorm{\frac1n\sum_{i=1}^n \left(g''(y_{ij},x_{ij}^\top\btheta)-g''(y_{ij},x_{ij}^\top\thetas)\right) x_{ij} x_{ij}^\top}_2\lesssim\left\|\btheta-\thetas\right\|_1.$$
	By the triangle inequality, assuming that $\left\|\btheta-\thetas\right\|_1=O_P\left(r_{\btheta}\right)$, we have that
	$$\matrixnorm{\nabla^2\cL_1(\btheta)-\nabla^2\cLs(\thetas)}_2\leq\matrixnorm{\nabla^2\cL_1(\btheta)-\nabla^2\cL_1(\thetas)}_2+\matrixnorm{\nabla^2\cL_1(\thetas)-\nabla^2\cLs(\thetas)}_2=O_P\left(\sqrt{\frac{d\log d}n}+r_{\btheta}\right).$$
	Since $\matrixnorm{\nabla^2\cLs(\thetas)^{-1}}_2=O(1)$, we have that
	$$\max_l\left\|\widetilde\Theta(\btheta)_l-\Theta_l\right\|_2\leq\matrixnorm{\widetilde\Theta(\btheta)-\Theta}_2=O_P\left(\sqrt{\frac{d\log d}n}+r_{\btheta}\right),\quad\text{and}$$
	$$\matrixnorm{\widetilde\Theta(\btheta)-\Theta}_\infty\leq\sqrt d\matrixnorm{\widetilde\Theta(\btheta)-\Theta}_2=O_P\left(d\sqrt{\frac{\log d}n}+\sqrt d r_{\btheta}\right).$$
	Note that $\matrixnorm{\Theta}_\infty\leq\sqrt d\matrixnorm{\Theta}_2=\sqrt d\matrixnorm{\nabla^2\cLs(\thetas)^{-1}}_2=O\left(\sqrt d\right)$.  By the triangle inequality, if $r_{\btheta}\lesssim1$, we have that
	$$\matrixnorm{\widetilde\Theta(\btheta)}_\infty\leq\matrixnorm{\widetilde\Theta(\btheta)-\Theta}_\infty+\matrixnorm{\Theta}_\infty=O_P\left(d\sqrt{\frac{\log d}n}+\sqrt d r_{\btheta}\right)+O\left(\sqrt d\right)=O_P\left(\sqrt d\right).$$
	
\end{proof}

\begin{lemma}\label{lem:m}
	In linear model, under Assumptions~\ref{as:design}~and~\ref{as:noise}, if $N\gtrsim d$, then we have that $$\left\|\htheta-\thetas\right\|_2\lesssim\sqrt{\frac{d\log\frac{d}\delta}{N}}+\frac{d\log\frac{d}\delta}{N},$$
	with probability at least $1-\delta$, for any $\delta$ such that  $e^{-N}\lesssim\delta<1$.
\end{lemma}

\begin{proof}[Lemma \ref{lem:m}]
	Note that
	$$\left\|\htheta-\thetas\right\|_2=\left\|\left(X_N^\top X_N\right)^{-1}X_N^\top y_N-\thetas\right\|_2=\left\|\left(X_N^\top X_N\right)^{-1}X_N^\top e_N\right\|_2\leq\matrixnorm{\left(\frac{X_N^\top X_N}N\right)^{-1}}_2\left\|\frac{X_N^\top e_N}N\right\|_2.$$
	By \eqref{eqn:bern2}, we have with probability at least $1-\delta$ that
	$$\left\|\frac{X_N^\top e_N}N\right\|_2\leq\sqrt d\left\|\frac{X_N^\top e_N}N\right\|_\infty\lesssim \sqrt{\frac{d\log\frac{d}\delta}{N}}+\frac{d\log\frac{d}\delta}{N}.$$
	By Proposition~2.1 of \cite{vershynin2012close}, if $n\gtrsim d$, we have with probability at least $1-\delta$ that
	$$\matrixnorm{\frac{X_N^\top X_N}N-\Sigma}_2\lesssim\sqrt{\frac{d+\log\frac1\delta}N}+\frac{d+\log\frac1\delta}N,$$
	and then, by the triangle inequality,
	\begin{align}
	\matrixnorm{\left(\frac{X_N^\top X_N}N\right)^{-1}}_2 &\leq \matrixnorm{\left(\frac{X_N^\top X_N}N\right)^{-1}-\Theta}_2+\matrixnorm{\Theta}_2 \leq \matrixnorm{\Theta}_2^2 \matrixnorm{\frac{X_N^\top X_N}N-\Sigma}_2+\matrixnorm{\Theta}_2 \notag \\
	&\lesssim\sqrt{\frac{d+\log\frac1\delta}N}+\frac{d+\log\frac1\delta}N+1 \lesssim 1, \label{eqn:covinv_conc}
	\end{align}
	provided that $N\gtrsim d+\log(1/\delta)$.  Finally, by the union bound, we have with probability at least $1-2\delta$ that
	$$\left\|\htheta-\thetas\right\|_2\lesssim\sqrt{\frac{d\log\frac{d}\delta}{N}}+\frac{d\log\frac{d}\delta}{N}.$$
\end{proof}

\begin{lemma}\label{lem:csl}
	In linear model, under Assumptions~\ref{as:design}~and~\ref{as:noise}, if $n\gtrsim d$, then we have that for any $t\geq1$,
	$$\left\|\ttheta^{(t)}-\htheta\right\|_2\lesssim\sqrt{\frac{d+\log\frac1\delta}n}\left\|\ttheta^{(t-1)}-\htheta\right\|_2,$$
	with probability at least $1-\delta$, for any $\delta$ such that  $e^{-n}\lesssim\delta<1$, where $\ttheta^{(t)}$ is the $t$-step CSL estimator defined in Algorithm \ref{alg:kgrad+csl}.
\end{lemma}

\begin{proof}[Lemma \ref{lem:csl}]
	Note that
	\begin{align*}
	\left\|\ttheta^{(t)}-\htheta\right\|_2&=\left\|\ttheta^{(t-1)}-\nabla^2\cL_1(\ttheta^{(t-1)})^{-1}\nabla\cL_N(\ttheta^{(t-1)})-\htheta\right\|_2 \\
	&= \left\|\ttheta^{(t-1)}-\left(\frac{X_1^\top X_1}n\right)^{-1}\frac{X_N^\top\left(X_N\ttheta^{(t-1)})-y_N\right)}N-\left(\frac{X_N^\top X_N}N\right)^{-1}\frac{X_N^\top y_N}N\right\|_2 \\
	&= \left\|\ttheta^{(t-1)}-\left(\frac{X_1^\top X_1}n\right)^{-1}\frac{X_N^\top\left(X_N\ttheta^{(t-1)})-y_N\right)}N-\ttheta^{(t-1)}+\left(\frac{X_N^\top X_N}N\right)^{-1}\frac{X_N^\top\left(X_N\ttheta^{(t-1)})-y_N\right)}N\right\|_2 \\
	&\leq \matrixnorm{\left(\frac{X_1^\top X_1}n\right)^{-1}-\left(\frac{X_N^\top X_N}N\right)^{-1}}_2 \left\|\frac{X_N^\top X_N\left(\ttheta^{(t-1)}-\htheta\right)}N\right\|_2 \\
	&\leq \matrixnorm{\left(\frac{X_1^\top X_1}n\right)^{-1}-\left(\frac{X_N^\top X_N}N\right)^{-1}}_2 \matrixnorm{\frac{X_N^\top X_N}N}_2 \left\|\ttheta^{(t-1)}-\htheta\right\|_2.
	\end{align*}
	By \eqref{eqn:covinv_conc} with triangle inequality and the union bound, we have with probability at least $1-\delta$ that
	\begin{align*}
	\matrixnorm{\left(\frac{X_1^\top X_1}n\right)^{-1}-\left(\frac{X_N^\top X_N}N\right)^{-1}}_2 &\leq\matrixnorm{\left(\frac{X_1^\top X_1}n\right)^{-1}-\Theta}_2+\matrixnorm{\left(\frac{X_N^\top X_N}N\right)^{-1}-\Theta}_2 \\
	&\lesssim \sqrt{\frac{d+\log\frac1\delta}n}+\frac{d+\log\frac1\delta}n+\sqrt{\frac{d+\log\frac1\delta}N}+\frac{d+\log\frac1\delta}N \\
	&\lesssim \sqrt{\frac{d+\log\frac1\delta}n}+\frac{d+\log\frac1\delta}n, \quad\text{and}
	\end{align*}
	$$\matrixnorm{\frac{X_N^\top X_N}N}_2\leq\matrixnorm{\frac{X_N^\top X_N}N-\Sigma}_2+\matrixnorm{\Sigma}_2\leq\sqrt{\frac{d+\log\frac1\delta}N}+\frac{d+\log\frac1\delta}N+1.$$
	Provided that $d+\log\frac1\delta\lesssim n$, we obtain the bound in the lemma.
\end{proof}

\begin{lemma}\label{lem:m_glm}
	In GLM, under Assumptions~\ref{as:smth_glm}--\ref{as:hes_glm}, if $N\gtrsim d^4\log d$, then we have that
	$$\left\|\htheta-\thetas\right\|_2\lesssim \sqrt{\frac{d\log\frac{d}\delta}N},$$
	with probability at least $1-\delta$, for any $\delta$ such that  $e^{-N/d^4}\lesssim\delta<1$.
\end{lemma}

\begin{proof}[Lemma \ref{lem:m_glm}]
	We use the argument in the proof of Lemma~6 of \cite{zhang2012communication}.  By Theorem~1.6.2 of \cite{tropp2015introduction}, we have with probability at least $1-\delta$ that
	$$\matrixnorm{\nabla^2\cL_N(\thetas)-\nabla^2\cLs(\thetas)}_2\leq C\sqrt{\frac{d\log\frac d\delta}N}+C\frac{d\log\frac d\delta}N,$$
	for some constant $C>0$. By \eqref{eqn:lip}, for any $\theta$, we have that
	$$\matrixnorm{\nabla^2\cL_N(\theta)-\nabla^2\cL_N(\thetas)}_2 \leq d\matrixnorm{\nabla^2\cL_N(\theta)-\nabla^2\cL_N(\thetas)}_{\max} \leq Cd\left\|\theta-\thetas\right\|_1 \leq Cd^{3/2}\left\|\theta-\thetas\right\|_2.$$
	Let $\rho=(4C\mu d^{3/2})^{-1}$ and assume $4C\mu\sqrt{d\log(d/\delta)/N}\leq1$ and $4C\mu d\log(d/\delta)/N\leq1$.  Then, for any $\theta\in U\defn\{\theta:\left\|\theta-\thetas\right\|_2\leq\rho\}$, we have by the triangle inequality that
	$$\matrixnorm{\nabla^2\cL_N(\theta)-\nabla^2\cLs(\thetas)}_2 \leq \matrixnorm{\nabla^2\cL_N(\theta)-\nabla^2\cL_N(\thetas)}_2 + \matrixnorm{\nabla^2\cL_N(\thetas)-\nabla^2\cLs(\thetas)}_2 \leq(2\mu)^{-1}.$$
	Since $\lambdamin(\nabla^2\cLs(\thetas))\geq\mu^{-1}$, we have $\lambdamin(\nabla^2\cL_N(\theta))\geq(2\mu)^{-1}$ for any $\theta\in U$.  Then, for any $\theta'\in\R^d$, we have that
	$$\cL_N(\theta')\ge\cL_N(\thetas)+\nabla\cL_N(\thetas)^\top(\theta'-\thetas)+(4\mu)^{-1}\min\left\{\left\|\theta'-\thetas\right\|_2^2,\rho^2\right\},$$
	and then,
	\begin{align*}
	\min\left\{\left\|\theta'-\thetas\right\|_2^2,\rho^2\right\} &\leq 4\mu(\cL_N(\theta')-\cL_N(\thetas)-\nabla\cL_N(\thetas)^\top(\theta'-\thetas)) \\
	&\leq 4\mu(\cL_N(\theta')-\cL_N(\thetas)+\|\nabla\cL_N(\thetas)\|_2\left\|\theta'-\thetas\right\|_2).
	\end{align*}
	Dividing both sides by $\left\|\theta'-\thetas\right\|_2$ and then setting $\theta'=\kappa\htheta+(1-\kappa)\thetas$ for any $\kappa\in[0,1]$, we have
	\begin{align*}
	\min\left\{\kappa\left\|\htheta-\thetas\right\|_2,\frac{\rho^2}{\kappa\left\|\htheta-\thetas\right\|_2}\right\} &\leq \frac{4\mu\left(\cL_N(\kappa\htheta+(1-\kappa)\thetas)-\cL_N(\thetas)\right)}{\kappa\left\|\htheta-\thetas\right\|_2}+4\mu\|\nabla\cL_N(\thetas)\|_2 < 4\mu\|\nabla\cL_N(\thetas)\|_2,
	\end{align*}
	where we use that $\cL_N(\kappa\htheta+(1-\kappa)\thetas)<\cL_N(\thetas)$ for any $\kappa\in(0,1)$ since $\cL_N$ is strongly convex at $\thetas$ and $\htheta$ minimizes $\cL_N$.  Note that $\nabla\cL_N(\thetas)=\sum_{i=1}^n\sum_{j=1}^k g'(y_{ij},x_{ij}^\top\thetas) x_{ij}/N$ and $g'(y_{ij},x_{ij}^\top\thetas) x_{ij,l}=O(1)$ for each $l=1,\dots,d$ under Assumptions~\ref{as:smth_glm}~and~\ref{as:design_glm}.  Then, by Hoeffding's inequality, we have that
	$$P\left(\left|\nabla\cL_N(\thetas)_l\right|>\sqrt{\frac{c\log\frac{2d}\delta}N}\right)\leq\frac\delta d,$$
	for any $\delta\in(0,1)$.  By the union bound, we have with probability at least $1-\delta$ that
	\begin{align}
	\left\|\nabla\cL_N(\thetas)\right\|_\infty\leq\sqrt{\frac{c\log\frac{2d}\delta}N}. \label{eqn:hoef2}
	\end{align}
	Then, we have with probability at least $1-\delta$ that
	$$\left\|\nabla\cL_N(\thetas)\right\|_2\leq\sqrt d\left\|\nabla\cL_N(\thetas)\right\|_\infty\leq C\sqrt{\frac{d\log\frac{d}\delta}N},$$
	and by the union bound, with probability at least $1-2\delta$,
	$$\min\left\{\kappa\left\|\htheta-\thetas\right\|_2,\frac{\rho^2}{\kappa\left\|\htheta-\thetas\right\|_2}\right\} < 4C\mu\sqrt{\frac{d\log\frac{d}\delta}N}\leq\rho,$$
	provided that $4C\mu\sqrt{d\log(d/\delta)/N}\leq\rho$.  Since this holds for any $\kappa\in(0,1)$, if $\left\|\htheta-\thetas\right\|_2>\rho$, we may set $\kappa=\rho/\left\|\htheta-\thetas\right\|_2<1$, and find that
	$$\min\left\{\kappa\left\|\htheta-\thetas\right\|_2,\frac{\rho^2}{\kappa\left\|\htheta-\thetas\right\|_2}\right\}=\rho,$$
	which would yield a contradiction.  Thus, we have $\left\|\htheta-\thetas\right\|_2\leq\rho$, that is, $\htheta\in U$.  Furthermore, we have that
	$$\left\|\htheta-\thetas\right\|_2^2 \leq 4\mu\left(\cL_N(\htheta)-\cL_N(\thetas)+\|\nabla\cL_N(\thetas)\|_2\left\|\htheta-\thetas\right\|_2\right) \leq 4\mu\|\nabla\cL_N(\thetas)\|_2\left\|\htheta-\thetas\right\|_2,$$
	and thus,
	$$\left\|\htheta-\thetas\right\|_2\leq 4\mu\|\nabla\cL_N(\thetas)\|_2\leq4C\mu\sqrt{\frac{d\log\frac{d}\delta}N},$$
	with probability at least $1-2\delta$, provided that $4C\mu\sqrt{d\log(d/\delta)/N}\leq1$, $4C\mu d\log(d/\delta)/N\leq1$, and $4C\mu\sqrt{d\log(d/\delta)/N}\leq\rho$, which hold if $\delta\gtrsim e^{-N/d^4}$ and $N\gtrsim d^4\log d$.
\end{proof}

\begin{lemma}\label{lem:csl_glm}
	In GLM, under Assumptions~\ref{as:smth_glm}--\ref{as:hes_glm}, if $n\gtrsim d^4\log d$, then we have that for any $t\geq1$,
	$$\left\|\ttheta^{(t)}-\htheta\right\|_2 \lesssim \left(\sqrt{\frac{d\log\frac d\delta}n}+d^{3/2}\left\|\ttheta^{(t-1)}-\htheta\right\|_2\right) \left\|\ttheta^{(t-1)}-\htheta\right\|_2,$$
	with probability at least $1-\delta$, for any $\delta$ such that  $e^{-n/d^4}\lesssim\delta<1$, where $\ttheta^{(t)}$ is the $t$-step CSL estimator defined in Algorithm \ref{alg:kgrad+csl}.
\end{lemma}

\begin{proof}[Lemma \ref{lem:csl_glm}]
	We use the argument in the proof of Theorem~3 of \cite{jordan2019communication}.  Note by the triangle inequality that
	\begin{align*}
	&\left\|\ttheta^{(t)}-\htheta\right\|_2 \\
	&=\left\|\ttheta^{(t-1)}-\nabla^2\cL_1(\ttheta^{(t-1)})^{-1}\nabla\cL_N(\ttheta^{(t-1)})-\htheta\right\|_2 \\
	&\leq \left\|\ttheta^{(t-1)}-\nabla^2\cL_N(\ttheta^{(t-1)})^{-1}\nabla\cL_N(\ttheta^{(t-1)})-\htheta\right\|_2 + \left\|\left(\nabla^2\cL_N(\ttheta^{(t-1)})^{-1}-\nabla^2\cL_1(\ttheta^{(t-1)})^{-1}\right)\nabla\cL_N(\ttheta^{(t-1)})\right\|_2.
	\end{align*}
	To bound the first term on the right hand side, we have that
	\begin{align*}
	&\left\|\ttheta^{(t-1)}-\nabla^2\cL_N(\ttheta^{(t-1)})^{-1}\nabla\cL_N(\ttheta^{(t-1)})-\htheta\right\|_2 \\
	&= \left\|\ttheta^{(t-1)}-\htheta-\nabla^2\cL_N(\ttheta^{(t-1)})^{-1}\left(\nabla\cL_N(\ttheta^{(t-1)})-\nabla\cL_N(\htheta)\right)\right\|_2 \\
	&= \left\|\ttheta^{(t-1)}-\htheta-\nabla^2\cL_N(\ttheta^{(t-1)})^{-1}\int_0^1\nabla^2\cL_N(\htheta+s(\ttheta^{(t-1)}-\htheta))ds \left(\ttheta^{(t-1)}-\htheta\right)\right\|_2 \\
	&= \left\|\nabla^2\cL_N(\ttheta^{(t-1)})^{-1}\int_0^1\nabla^2\cL_N(\ttheta^{(t-1)})-\nabla^2\cL_N(\htheta+s(\ttheta^{(t-1)}-\htheta))ds \left(\ttheta^{(t-1)}-\htheta\right)\right\|_2 \\
	&\leq \matrixnorm{\nabla^2\cL_N(\ttheta^{(t-1)})^{-1}}_2\int_0^1\matrixnorm{\nabla^2\cL_N(\ttheta^{(t-1)})-\nabla^2\cL_N(\htheta+s(\ttheta^{(t-1)}-\htheta))}_2 ds \left\|\ttheta^{(t-1)}-\htheta\right\|_2.
	\end{align*}
	By the proof of Lemma~\ref{lem:m_glm}, we have that
	\begin{align*}
	\matrixnorm{\nabla^2\cL_N(\ttheta^{(t-1)})^{-1}}_2
	&\leq \matrixnorm{\nabla^2\cL_N(\ttheta^{(t-1)})^{-1}-\nabla^2\cLs(\thetas)^{-1}}_2 + \matrixnorm{\nabla^2\cLs(\thetas)^{-1}}_2 \\
	&\leq \matrixnorm{\nabla^2\cLs(\thetas)^{-1}}_2^2 \matrixnorm{\nabla^2\cL_N(\ttheta^{(t-1)})-\nabla^2\cLs(\thetas)}_2 + \matrixnorm{\nabla^2\cLs(\thetas)^{-1}}_2 \\
	&\lesssim \sqrt{\frac{d\log\frac d\delta}N}+\frac{d\log\frac d\delta}N+d^{3/2}\left\|\ttheta^{(t-1)}-\thetas\right\|_2+1,
	\end{align*}
	with probability at least $1-\delta$, and
	$$\matrixnorm{\nabla^2\cL_N(\ttheta^{(t-1)})-\nabla^2\cL_N(\htheta+s(\ttheta^{(t-1)}-\htheta))}_2\lesssim d^{3/2}\left\|\ttheta^{(t-1)}-\htheta\right\|_2,$$
	and thus,
	$$\left\|\ttheta^{(t-1)}-\nabla^2\cL_N(\ttheta^{(t-1)})^{-1}\nabla\cL_N(\ttheta^{(t-1)})-\htheta\right\|_2\lesssim\left(\sqrt{\frac{d\log\frac d\delta}N}+\frac{d\log\frac d\delta}N+d^{3/2}\left\|\ttheta^{(t-1)}-\thetas\right\|_2+1\right) d^{3/2}\left\|\ttheta^{(t-1)}-\htheta\right\|_2^2.$$
	To bound the second term, we have that
	\begin{align*}
	&\left\|\left(\nabla^2\cL_N(\ttheta^{(t-1)})^{-1}-\nabla^2\cL_1(\ttheta^{(t-1)})^{-1}\right)\nabla\cL_N(\ttheta^{(t-1)})\right\|_2 \\
	&\leq \matrixnorm{\nabla^2\cL_N(\ttheta^{(t-1)})^{-1}-\nabla^2\cL_1(\ttheta^{(t-1)})^{-1}}_2 \left\|\nabla\cL_N(\ttheta^{(t-1)})\right\|_2 \\
	&= \matrixnorm{\nabla^2\cL_N(\ttheta^{(t-1)})^{-1}-\nabla^2\cL_1(\ttheta^{(t-1)})^{-1}}_2 \left\|\nabla\cL_N(\ttheta^{(t-1)})-\nabla\cL_N(\htheta)\right\|_2 \\
	&\leq \matrixnorm{\nabla^2\cL_N(\ttheta^{(t-1)})^{-1}-\nabla^2\cL_1(\ttheta^{(t-1)})^{-1}}_2 \int_0^1\matrixnorm{\nabla^2\cL_N(\htheta+s(\ttheta^{(t-1)}-\htheta))}_2 ds \left\|\ttheta^{(t-1)}-\htheta\right\|_2.
	\end{align*}
	By the proof of Lemma~\ref{lem:m_glm}, we have that
	\begin{align*}
	\matrixnorm{\nabla^2\cL_N(\ttheta^{(t-1)})^{-1}-\nabla^2\cL_1(\ttheta^{(t-1)})^{-1}}_2
	&\leq \matrixnorm{\nabla^2\cL_N(\ttheta^{(t-1)})^{-1}-\nabla^2\cLs(\thetas)^{-1}}_2 + \matrixnorm{\nabla^2\cL_1(\ttheta^{(t-1)})^{-1}-\nabla^2\cLs(\thetas)^{-1}}_2 \\
	&\lesssim \sqrt{\frac{d\log\frac d\delta}n}+\frac{d\log\frac d\delta}n+d^{3/2}\left\|\ttheta^{(t-1)}-\thetas\right\|_2,
	\end{align*}
	with probability at least $1-\delta$, and
	\begin{align*}
	\matrixnorm{\nabla^2\cL_N(\htheta+s(\ttheta^{(t-1)}-\htheta))}_2
	&\leq \matrixnorm{\nabla^2\cL_N(\htheta+s(\ttheta^{(t-1)}-\htheta))-\nabla^2\cLs(\thetas)}_2 + \matrixnorm{\nabla^2\cLs(\thetas)}_2 \\
	&\lesssim \sqrt{\frac{d\log\frac d\delta}N}+\frac{d\log\frac d\delta}N+d^{3/2}\left(\left\|\ttheta^{(t-1)}-\thetas\right\|_2+\left\|\htheta-\thetas\right\|_2\right)+1 \\
	&\lesssim d^{3/2} \left\|\ttheta^{(t-1)}-\thetas\right\|_2 + 1,
	\end{align*}
	for $\delta\gtrsim e^{-N/d^4}$, provided that $N\gtrsim d^4\log d$, and thus,
	\begin{align*}
	&\left\|\left(\nabla^2\cL_N(\ttheta^{(t-1)})^{-1}-\nabla^2\cL_1(\ttheta^{(t-1)})^{-1}\right)\nabla\cL_N(\ttheta^{(t-1)})\right\|_2 \\
	&\lesssim\left(\sqrt{\frac{d\log\frac d\delta}n}+\frac{d\log\frac d\delta}n+d^{3/2}\left\|\ttheta^{(t-1)}-\thetas\right\|_2\right) \left(d^{3/2} \left\|\ttheta^{(t-1)}-\thetas\right\|_2 + 1\right) \left\|\ttheta^{(t-1)}-\htheta\right\|_2.
	\end{align*}
	Provided that $n\gtrsim d^4\log d$ and $\delta\gtrsim e^{-n/d^4}$, we have $d^{3/2} \left\|\ttheta^{(t-1)}-\thetas\right\|_2\lesssim 1$ for any $t\geq1$, and then,
	$$\left\|\ttheta^{(t)}-\htheta\right\|_2 \lesssim d^{3/2}\left\|\ttheta^{(t-1)}-\htheta\right\|_2^2 +  \left(\sqrt{\frac{d\log\frac d\delta}n}+d^{3/2}\left\|\ttheta^{(t-1)}-\thetas\right\|_2\right) \left\|\ttheta^{(t-1)}-\htheta\right\|_2.$$
	Since
	$$\left\|\ttheta^{(t-1)}-\thetas\right\|_2\leq\left\|\ttheta^{(t-1)}-\htheta\right\|_2+\left\|\htheta-\thetas\right\|_2\leq\left\|\ttheta^{(t-1)}-\htheta\right\|_2+\sqrt{\frac{d\log\frac{d}\delta}N},$$
	we obtain the bound in the lemma.
\end{proof}


\end{document}